\documentclass{article}
\PassOptionsToPackage{numbers,compress}{natbib}
\usepackage[final]{neurips_2022}
\usepackage[utf8]{inputenc}
\usepackage[T1]{fontenc}
\usepackage{hyperref}
\usepackage{url}
\usepackage{booktabs}
\usepackage{amsfonts}
\usepackage{nicefrac}
\usepackage{microtype}
\usepackage{xcolor}

\usepackage{pifont}
\newcommand{\cmark}{\ding{51}}
\definecolor{light-gray}{gray}{0.8}
\newcommand{\xmark}{\color{light-gray}\ding{55}}

\usepackage{wrapfig}
\usepackage{graphicx}
\usepackage{amsmath}
\usepackage{amssymb}
\usepackage{mathtools}
\usepackage{amsthm}
\usepackage{amsfonts}
\usepackage{bm}

\theoremstyle{plain}
\newtheorem{theorem}{Theorem}
\newtheorem{proposition}{Proposition}
\newtheorem{lemma}{Lemma}

\theoremstyle{definition}
\newtheorem{definition}{Definition}

\theoremstyle{remark}

\usepackage[inline]{enumitem}
\usepackage{tikz}
\usepackage{pgfplots}
\pgfplotsset{compat=1.17}
\usepackage{adjustbox}
\usepackage{afterpage}
\usepackage{multirow}
\newcommand*{\citeincell}[1]{{\citep{#1}}}

\usetikzlibrary{positioning, backgrounds, shapes, calc}
\tikzset{
  partial ellipse/.style args={#1:#2:#3}{
    insert path={+ (#1:#3) arc (#1:#2:#3)}
  }
}

\usetikzlibrary{arrows}
\usetikzlibrary{shapes,arrows.meta}
\tikzset{
  block/.style   ={draw, thick, rectangle, minimum height=1.5em,
    minimum width=5em},
  sum/.style     ={draw, circle, node distance=1cm},
  input/.style   ={coordinate},
  output/.style  ={coordinate},
  el/.style={inner sep=2pt, align=left},
}

\DeclareMathOperator*{\argmin}{argmin}
\DeclareMathOperator*{\minimize}{minimize}
\DeclareMathOperator*{\maximize}{maximize}

\DeclareMathOperator*{\diag}{diag}
\DeclareMathOperator{\tr}{tr}

\newcommand*{\barpiEt}{\bar{\pi}_{\!{\scriptscriptstyle E},t}}
\newcommand*{\barpiEn}{\bar{\pi}_{\!{\scriptscriptstyle E},n}}
\newcommand*{\barpiEone}{\bar{\pi}_{\!{\scriptscriptstyle E},\hspace*{-.2pt}\scriptscriptstyle 1}}

\newcommand*{\barpiET}{\bar{\pi}_{\!{\scriptscriptstyle E},{\scriptscriptstyle T}}}
\newcommand*{\barpiETT}{\bar{\pi}_{\!{\scriptscriptstyle E},{\scriptscriptstyle T}+1}}
\newcommand*{\barpiEnn}{\bar{\pi}_{\!{\scriptscriptstyle E},n+1}}
\newcommand*{\barpiEst}{\bar{\pi}^s_{\!{\scriptscriptstyle E},t}}
\newcommand*{\barpiEsn}{\bar{\pi}^s_{\!{\scriptscriptstyle E},n}}
\newcommand*{\barpiEstt}{\bar{\pi}^s_{\!{\scriptscriptstyle E},t+1}}
\newcommand*{\barpiEsTT}{\bar{\pi}^s_{\!{\scriptscriptstyle E},{\scriptscriptstyle T}+1}}
\newcommand*{\barpiEsnn}{\bar{\pi}^s_{\!{\scriptscriptstyle E},n+1}}

\newcommand*{\rVertast}{\rVert_{\!\ast}}
\newcommand*{\bigrVertast}{\bigr\rVert_{\!\ast}}

\newcommand*{\bigranglecA}{\bigr\rangle_{\!\!\mathcal{A}}}

\newcommand*{\tauonet}{\tau_{\hspace*{-.3pt}1\hspace*{-.5pt}:\hspace*{-.2pt}t}}
\newcommand*{\tauonett}{\tau_{\hspace*{-.3pt}1\hspace*{-.5pt}:\hspace*{-.2pt}t+1}}
\newcommand*{\tauoneT}{\tau_{\hspace*{-.3pt}1\hspace*{-.5pt}:\hspace*{-.2pt}T}}
\newcommand*{\tauoneTT}{\tau_{\hspace*{-.3pt}1\hspace*{-.5pt}:\hspace*{-.2pt}T+1}}
\newcommand*{\tauonen}{\tau_{\hspace*{-.3pt}1\hspace*{-.5pt}:\hspace*{-.2pt}n}}
\newcommand*{\tauonenn}{\tau_{\hspace*{-.3pt}1\hspace*{-.5pt}:\hspace*{-.2pt}n+1}}

\newcommand*{\piE}{\pi_{\!\scriptscriptstyle E}}
\newcommand*{\hatqE}{\hat{q}_{\scriptscriptstyle E}}
\newcommand*{\qE}{q_{\scriptscriptstyle E}}
\newcommand*{\rE}{r_{\!\scriptscriptstyle E}}
\newcommand*{\hatrE}{\hat{r}_{\!\scriptscriptstyle E}}
\newcommand*{\piT}{\pi_{\hspace{-0.05em}\scriptscriptstyle T}}
\newcommand*{\tauT}{\tau_{\hspace{-0.05em}\scriptscriptstyle T}}

\newcommand*{\psiE}{\psi_{\!\scriptscriptstyle E}}

\newcommand*{\alphaT}{\alpha_{\scriptscriptstyle T}}
\newcommand*{\etaT}{\eta_{\scriptscriptstyle T}}

\newcommand*{\PsiOmega}{\Psi_{\!\scriptscriptstyle\Omega}}
\newcommand*{\JOmega}{J_{\scriptscriptstyle\Omega}}
\newcommand*{\cPOmega}{\mathcal{P}_{\!\scriptscriptstyle\Omega}}
\newcommand*{\Pis}{\Pi^{\hspace*{-.5pt}\textrm{\raisebox{.5pt}{$s$}}}}
\newcommand*{\piinPi}{{\pi\hspace*{-.5pt}{\textrm{\raisebox{.3pt}{$\scriptscriptstyle\in$}}}\Pi}}
\newcommand*{\pisinPis}{{\pi^{\hspace*{-.5pt}s}\hspace*{-.5pt}{\textrm{\raisebox{1pt}{$\scriptscriptstyle\in$}}}\Pi^{\hspace*{-.5pt}\textrm{\raisebox{.5pt}{$s$}}}}}

\newcommand*{\pisinDeltacA}{{\pi^{\hspace*{-.5pt}s}\hspace*{-.5pt}{\textrm{\raisebox{.7pt}{$\scriptscriptstyle\in$}}}\Delta_{\mathcal{A}}}}
\newcommand*{\tildepisinDeltacA}{{\tilde{\pi}^{\hspace*{-.5pt}s}\hspace*{-.5pt}{\textrm{\raisebox{.7pt}{$\scriptscriptstyle\in$}}}\Delta_{\mathcal{A}}}}

\newcommand*{\sincS}{{s\hspace*{-.1pt}{\textrm{\raisebox{.45pt}{$\scriptscriptstyle\in$}}}\hspace*{-.1pt}\mathcal{S}}}

\newcommand*{\rp}{\mathrm{p}}
\renewcommand*{\rq}{\mathrm{q}}
\newcommand*{\dprime}{\prime\prime}

\newcommand*{\rI}{\mathrm{I}}
\newcommand*{\cA}{\mathcal{A}}
\newcommand*{\cT}{\mathcal{T}}

\newcommand*{\cW}{\mathcal{W}}

\newcommand*{\cO}{\mathcal{O}}
\newcommand*{\cP}{\mathcal{P}}
\newcommand*{\cS}{\mathcal{S}}

\newcommand*{\cF}{\mathcal{F}}

\newcommand*{\bE}{\mathbb{E}}
\newcommand*{\bN}{\mathbb{N}}
\newcommand*{\bR}{\mathbb{R}}
\newcommand*{\bI}{\mathbb{I}}
\newcommand*{\cN}{\mathcal{N}}

\newcommand*{\cL}{\mathcal{L}}
\newcommand*{\cJ}{\mathcal{J}}
\newcommand*{\sT}{\mathsf{T}}
\newcommand*{\rd}{\mathrm{d}}
\newcommand*{\lnq}{\ln q}

\newcommand*{\Rm}{Riemannian}
\newcommand*{\Sh}{Shannon}
\newcommand*{\Bz}{Boltzmann}
\newcommand*{\Bn}{Banach}

\newcommand*{\Ts}{Tsallis}

\newcommand*{\Bg}{Bregman}
\newcommand*{\Gs}{Gaussian}
\newcommand*{\Ec}{Euclidean}

\newcommand*{\Mkv}{Markovian}
\newcommand*{\KuLi}{Kullback-Leibler}
\newcommand*{\CaSc}{Cauchy-Schwarz}
\newcommand*{\LeFe}{Legendre-Fenchel}

\newcommand*{\Ch}{Cholesky}
\newcommand*{\Lip}{Lipschitz}

\newcommand*{\KL}{\mathrm{KL}}
\newcommand*{\MJ}{MuJoCo}

\newcommand*{\Lp}{$L^\mathrm{p}$}
\newcommand*{\Lq}{$L^\mathrm{q}$}

\newcommand*{\suma}{$+$}
\newcommand*{\rew}{$r_{\phi}(s,\cdot)$}
\newcommand*{\rewtwo}{$r_{\phi}(s,a)$}
\newcommand*{\dervone}{$\psi_{\phi}(s,\cdot)$}
\newcommand*{\dervtwo}{$\psi_{\phi}(s,a)$}
\newcommand*{\bfn}{$d_\xi(s)$}

\definecolor{superlightgray}{RGB}{235,235,235}
\definecolor{verylightgray}{RGB}{165,165,165}

\newcommand*{\alg}[1]{Algorithm~\ref{#1}}
\newcommand*{\fig}[1]{Fig.~\ref{#1}}
\newcommand*{\tab}[1]{Tab.~\ref{#1}}
\newcommand*{\eq}[1]{Eq.~(\ref{#1})}
\newcommand*{\eqbrief}[1]{Eq.~(\ref{#1})}
\newcommand*{\thm}[1]{Theorem~\ref{#1}}
\newcommand*{\defn}[1]{Definition~\ref{#1}}
\newcommand*{\prop}[1]{Proposition~\ref{#1}}
\newcommand*{\lem}[1]{Lemma~\ref{#1}}
\newcommand*{\sect}[1]{Section~\ref{#1}}
\newcommand*{\appsect}[1]{Appendix~\ref{#1}}

\newcommand*{\topicquad}{\ }

\usepackage{algorithm}
\usepackage[noend]{algpseudocode}
\usepackage{caption}
\newcommand*{\enumone}{{\raisebox{.53pt}{\small\textcircled{\raisebox{-.23pt} {\hskip.06pt\scriptsize1}}}}}
\newcommand*{\enumtwo}{{\raisebox{.53pt}{\small\textcircled{\raisebox{-.2pt} {\hskip.1pt\scriptsize2}}}}}
\newcommand*{\enumthree}{{\raisebox{.53pt}{\small\textcircled{\raisebox{-.25pt} {\hskip.15pt\scriptsize3}}}}}
\newcommand*{\enumfour}{{\raisebox{.53pt}{\small\textcircled{\raisebox{-.25pt} {\scriptsize4}}}}}
\newcommand*{\enumfive}{{\raisebox{.53pt}{\small\textcircled{\raisebox{-.25pt} {\hskip.1pt\scriptsize5}}}}}
\newcommand*{\subfigrefA}{(a)}
\newcommand*{\subfigrefB}{(b)}
\newcommand*{\subfigrefC}{(c)}

\title{Robust Imitation via\\Mirror Descent Inverse Reinforcement Learning}

\author{
  Dong-Sig Han,\hskip5ptHyunseo Kim,\hskip5ptHyundo Lee,\hskip4.5ptJe-Hwan Ryu,\hskip4ptByoung-Tak Zhang\\
  Artificial Intelligence Institute, Seoul National University\\
  \texttt{\{dshan\hskip-1.5pt,\,hskim\hskip-1.5pt,\,hdlee\hskip-1.5pt,\,jhryu\hskip-1.5pt,\,btzhang\}@bi.snu.ac.kr}
}

\begin{document}

\maketitle

\begin{abstract}
Recently, adversarial imitation learning has shown a scalable reward acquisition method for inverse reinforcement learning (IRL) problems. However, estimated reward signals often become uncertain and fail to train a reliable statistical model since the existing methods tend to solve hard optimization problems directly. Inspired by a first-order optimization method called mirror descent, this paper proposes to predict a sequence of reward functions, which are iterative solutions for a constrained convex problem. IRL solutions derived by mirror descent are tolerant to the uncertainty incurred by target density estimation since the amount of reward learning is regulated with respect to local geometric constraints. We prove that the proposed mirror descent update rule ensures robust minimization of a Bregman divergence in terms of a rigorous regret bound of $\mathcal{O}(1/T)$ for step sizes $\{\eta_t\}_{t=1}^{T}$. Our IRL method was applied on top of an adversarial framework, and it outperformed existing adversarial methods in an extensive suite of benchmarks.
\end{abstract}

\section{Introduction}\label{sect:intro}
One crucial requirement of practical imitation learning methods is \textit{robustness}, often described as learning expert behavior for a finite number of demonstrations, overcoming various realistic challenges \citep{robust_il}. In real-world problems such as motor control tasks, the demonstration size can be insufficient to create a precise model of an expert \citep{deepmindgrasp}, and even in some cases, demonstrations can be noisy or suboptimal to solve the problem \citep{NIPS2013_fd5c905b}. For such challenging scenarios, imitation learning algorithms inevitably struggle with unreliable statistical models; thus, the way of handling the uncertainty of estimated cost functions dramatically affects imitation performance. Therefore, a thorough analysis of addressing these issues is required to construct a robust algorithm.

Inverse reinforcement learning (IRL) is an algorithm for learning ground-truth rewards from expert demonstrations where the expert acts optimally with respect to an unknown reward function \citep{irl, pwil}. Traditional IRL studies solve the imitation problem based on iterative algorithms \citep{maxmgirl1, maxentirl}, alternating between the reward estimation process and a reinforcement learning (RL)  \citep{rl} algorithm. In contrast, newer studies of adversarial imitation learning (AIL)  \citep{gail, airl} rather suggest learning reward functions of a certain form ``directly,'' by using adversarial learning objectives \citep{gan} and nonlinear discriminative neural networks \citep{airl}. Compared to classical approaches, the AIL methods have shown great success on control benchmarks in terms of scalability for challenging control tasks \citep{dac}.

Technically, it is well known that  AIL formulates a divergence minimization problem with its discriminative signals, which incorporates fine-tuned estimations of the target densities \citep{fairl}. Through the lens of differential geometries, the limitation of AIL naturally comes from the implication that minimizing the divergence does not guarantee unbiased progression due to constraints of the underlying space \citep{bregdist}. In order to ensure further stability, we argue that an IRL algorithm's progress needs to be regulated, yielding gradual updates with respect to local geometries of policy distributions.

We claim that there are two issues leading to unconstrained policy updates: \enumone{} a statistical divergence often cannot be accurately obtained for challenging problems, and \enumtwo{} an immediate divergence between agent and expert densities does not guarantee unbiased learning directions. Our approach is connected to a collection of optimization processes called mirror descent (MD) \citep{md}. For a sequence of parameters $\{w_t\}_{t=1}^T$ and a convex function $\Omega$, an MD update for a cost function $F_t$ is derived as
\begin{equation}\label{eq:md}
  \nabla\Omega\bigl(\hspace*{.2pt}w_{t+1}\hspace*{-.5pt}\bigr)=\nabla\Omega\bigl(\hspace*{.2pt}w_t\bigr)-\eta_t\nabla\hspace*{-1pt} F_t\bigl(\hspace*{.2pt}w_t\hspace*{-.1pt}\bigr).
\end{equation}
In the equation, the gradient $\nabla\Omega(\hspace*{1pt}\cdot\hspace*{1pt})$ creates a transformation that links a parametric space to its dual space. Theoretically, MD is a first-order method for solving constrained problems, which enjoys rigorous regret bounds for various geometries \citep{omd_univ, omd_converge} including probability spaces. Thus, applying MD to the reward estimation process can be efficient in terms of the number of learning phases.

In this paper, we derive an MD update rule in IRL upon a postulate of nonstationary estimations of the expert density, resulting in convergent reward acquisition even for challenging problems. Compared to MD algorithms in optimization studies, our methodology draws a sequence of functions on an alternative space induced by a reward operator $\PsiOmega$ (\defn{defn:rop}). To this end, we propose an AIL algorithm called mirror descent adversarial inverse reinforcement learning (MD-AIRL). Our empirical evidence showed that MD-AIRL outperforms the regularized adversarial IRL (RAIRL) \citep{rairl} methods. For example, MD-AIRL showed higher performance in 30 distinct cases among 32 different configurations in challenging MuJoCo \citep{mujoco} benchmarks, and it also clearly showed higher tolerance to suboptimal data. All of these results are strongly aligned with our theoretical analyses.

\vskip-7pt
\begin{table}[h]
\centering
  \caption{A technical overview. Traditional IRL methods lack scalability, and RAIRL does not guarantee convergence of its solution for realistic cases. MD-AIRL combines desirable properties.}\label{tab:overview}
  \begin{adjustbox}{width=0.99\textwidth, center}
\begin{tabular}{lcccccc}
\toprule
   Method&
   Reference&
   Scalability&
   Rewards&
   \Bg{} divergence &
   Iterative solutions&
   Convergence analyses\\
  \midrule
  BC {\small(\citeyear{bc})}&\small\citeincell{bc}&\cmark&\xmark&\xmark&\xmark&\cmark\\
  MM-IRL {\small(\citeyear{maxmgirl1})}&\small\citeincell{maxmgirl1}&\xmark&\cmark&\xmark&\cmark&\cmark\\
  GAIL {\small(\citeyear{gail})}&\small\citeincell{gail}&\cmark&\cmark&\xmark&\xmark&\xmark\\
  RAIRL {\small(\citeyear{rairl})}&\small\citeincell{rairl}&\cmark&\cmark&\cmark&\xmark&\xmark\\
\midrule
   \textbf{MD-AIRL (ours)}&\textbf{--}&\cmark&\cmark&\cmark&\cmark&\cmark\\
\bottomrule
\end{tabular}
\end{adjustbox}
\end{table}

\textbf{Our contributions.\topicquad}
Our work is complementary to previous IRL studies; the theoretical and technical contributions are built upon a novel perspective of considering iterative RL and IRL algorithms as a combined optimization process with dual aspects. Comparing MD-AIRL and RAIRL, both are highly generalized algorithms in terms of a variety of choices of divergence functions. \tab{tab:overview} shows that MD-AIRL brings beneficial results in realistic situations of limited time and data, since our approach is more aligned with earlier theoretical IRL studies providing formalized reward learning schemes and convergence guarantees. In summary, we list our main contributions below:
\begin{itemize}[leftmargin=*, noitemsep, partopsep=0pt, topsep=0pt, parsep=0pt]
  \item Instead of a monolithic estimation process of a global solution in AIL, we derive a sequence of reward functions that provides iterative local objectives (\sect{sect:reward}).
  \item We formally prove that rewards derived by an MD update rule guarantee the robust performance of divergence minimization along with a rigorous regret bound (\sect{sect:theory}).
  \item We propose a novel adversarial algorithm that is motivated by mirror descent, which is tolerant of unreliable discriminative signals of the AIL framework (Sections \ref{sect:mdirl} and \ref{sect:expr}).
\end{itemize}

\section{Related Works}
\textbf{Mirror descent.\topicquad}
We are interested in a family of statistical divergences called the \Bg{} divergence \citep{bregman1967relaxation}. The divergence generalizes constrained optimization problems such as least squares \citep{cvx, fundamentals_cvx}, and it also has been applied in various subfields of machine learning \citep{cluster_breg, bregmn}. In differential geometries, the \Bg{} divergence is a first-order approximation for a metric tensor and satisfies metric-like properties \citep{bregdist, bregman_traingle}. MD is also closely related to optimization methods regarding non-\Ec{} geometries with a discretization of steps such as natural gradients \citep{ng, md_info}. In the primal space, training with the infinitesimal limit of MD steps corresponds to a \Rm{} gradient flow \citep{gf, mirrorless}. In the RL domain, MD has been recently studied for policy optimization \citep{policy_md1, policy_md2, mdpo}. In this paper, we focus on learning with suboptimal representations of policy, and our distinct goal is to draw a robust reward learning scheme based on MD for the IRL problem.

\textbf{Imitation learning.\topicquad}
As a statistical model for the information geometry \citep{infogeo}, energy-based policies (i.e., \Bz{} distributions) appeared in early IRL studies, such as Bayesian IRL, natural gradient IRL, and maximum likelihood IRL \citep{bayesirl, ngirl, mlirl} for modeling expert distribution to parameterized functions. Notably, MaxEnt IRL \citep{maxentirl, maxcausalirl} is one of the representative classical IRL algorithms based on an information-theoretic perspective toward IRL solutions. Also, discriminators of AIL are trained by logistic regression; thus, the logit score of the discriminator defines an energy function that approximates the truth data density for the expert distribution \citep{ebgan}. Other statistical entropies have also been applied to AIL, such as the \Ts{} entropy \citep{tsallis_il}. On the one hand, our approach is closely related to RAIRL \citep{rairl}, which defined its AIL objective using the \Bg{} divergence. On the other hand, this work further employs the \Bg{} divergence to derive iterative MD updates for reward functions, resulting in theoretically pleasing properties while retaining the scalability of AIL.

\textbf{Learning theory.\topicquad}
There have been considerable achievements in dealing with temporal costs $\{F_t\}_{t=1}^\infty$, often referred to as \textit{online learning} \citep{online}. The most ordinary approach is stochastic gradient descent (SGD): $w_{t+1}\!=w_t-\eta_t \nabla F_t(w_t)$. In particular, SGD is a desirable algorithm when the parameter $w_t$ resides in the \Ec{} space since it ensures unbiased minimization of the expected cost. Apparently, policies appear in geometries of probabilities; thus, an incurred gradient may not be the direction of the steepest descent due to geometric constraints \citep{ng, infogeo}. An online form of MD in \eq{eq:md} is analogous to SGD for non-\Ec{} spaces, where each local metric is specified by a \Bg{} divergence \citep{mirrorless}. Our theoretical findings and proofs follow the results of online mirror descent (OMD) that appeared in previous literature for general aspects  \citep{md,omd_univ,md_nonlin,md_info,omd_converge,mirrorless}. Our analyses extend existing theoretical results to IRL; at the same time, they are also highly general to cover various online imitation learning problems which require making decisions sequentially.

\section{Background}\label{sect:prem}
For sets $X$ and $Y$, let $Y^X$ be a set of functions from $X$ to $Y$ and $\Delta_X$ ($\Delta^Y_X$) be a set of (conditional) probabilities over $X$ (conditioned on $Y$). We consider an MDP defined as a tuple $(\cS,\cA,P,r,\gamma)$ with the state space $\cS$, the action space $\cA$, the \Mkv{} transition kernel $P\in\Delta^{\cS\times\cA}_\cS$, the reward function $r\in\bR^{S\times A}$ and the discount factor $\gamma\in[0,1)$. Let a function $\Omega\!: \Delta_{\hspace*{-1.5pt}\cA}\!\to\!\bR$ be strongly convex. Using $\Omega$,
the \Bg{} divergence is defined as
\begin{equation*}
  D_\Omega(\pi^s\Vert\,\hat{\pi}^s)\coloneqq\Omega(\pi^s)-\Omega(\hat{\pi}^s)-\!\bigl\langle\nabla\Omega(\hat{\pi}^s),\;\pi^s-\hat{\pi}^s\bigr\rangle_{\!\!\scriptscriptstyle\cA},
\end{equation*}
where $\pi^s$ and $\hat{\pi}^s$ denote arbitrary policies for a given state $s$. For a representative divergence, one can consider the popular \KuLi{} (KL) divergence. The KL divergence is a \Bg{} divergence when $\Omega$ is specified as the negative \Sh{}  entropy: $\Omega(\pi^s) = \sum_{a} \pi^s(a)\ln \pi^s(a)$.

Regularization of the policy distribution with respect to convex $\Omega$ brings distinct properties to the learning agent \citep{reg_mdp, rac}. The objective of regularized RL is to find $\pi\,\textrm{\raisebox{.5pt}{$\textstyle\in$}}\,\Pi$ that maximizes the expected value of discounted cumulative returns along with a causal convex regularizer $\Omega$, i.e.,
\begin{equation}\label{eq:reg_obj}
  \JOmega(\pi, r)\coloneqq\bE_\pi\!\Bigl[\,\sum\nolimits_{i=0}^\infty\gamma^i\hspace*{-.2pt}\Bigl\{r(s_i, a_i)\!-\!\Omega\bigl(\pi(\,\cdot\,|s_i)\bigr)\!\Bigr\}\hspace*{-.5pt}\Bigr],
\end{equation}
where the subscript $\pi$ on the expectation indicates that each action is sampled by $\pi(\hspace*{.5pt}\cdot\hspace*{.5pt}|s_i)$ for the given MDP.
In this setup, a regularized RL algorithm finds a unique solution in a subset of the conditional probability space denoted as $\Pi\! \coloneqq\! [\Pis\hspace*{-.5pt}]_{\hspace*{-.8pt}\sincS}\! \subset\! \Delta^\cS_\cA$ constrained by the parameterization of a policy.

The objective of IRL is to find a function $\rE$ that rationalizes the behavior of an expert policy $\piE$. For an inner product $\langle\hspace*{.2pt}\cdot,\cdot\rangle_{\!\scriptscriptstyle\cA}$, consider $\Omega^\ast$\!, the \LeFe{} transform (convex conjugate) of $\Omega$:
\begin{equation}\label{eq:conj}
 \forall\,q^s{\textrm{\raisebox{.5pt}{$\textstyle\in$}}}\,\bR^{\cA}\!,\hspace*{10pt}  \Omega^\ast\bigl(q^s\bigr)\, = \!\max_{\pisinDeltacA}\!\bigl\langle\hspace*{.2pt}\pi^s\!,\,\,q^s\hspace*{.5pt}\bigr\rangle_{\!\!\scriptscriptstyle\cA} \!\!- \Omega\bigl(\hspace*{1pt}\pi^s\hspace*{-.5pt}\bigr),
\end{equation}
where $q^s$ and $\pi^s$ denote the shorthand notation of $q(s,\cdot\,)$ and $\pi(\hspace*{.5pt}\cdot\hspace*{.5pt}|s)$. Differentiating both sides with respect to $q^s$, the gradient of conjugate $\nabla\Omega^\ast$ maps $q^s$ to a policy distribution.
One fundamental property in \textit{regularized} IRL \citep{reg_mdp} is that $\piE$ is the maximizing argument of $\Omega^\ast$ for $\qE$, where $\qE$ is the regularized state-action value function $\qE(s,a)\!=\!\bE_{\piE}\![\sum_{i=0}^\infty \gamma^i \{ \rE(s,a)\! -\! \Omega(\piE^s)\}\!|s_0\!\!=\!\!s,\hspace*{-.1pt}a_0\!\!=\!\!a]$. Note that the problem is ill-posed, and every $\hatrE$ that makes its value function $\hatqE$ satisfy $\piE^s\!=\!\!\nabla\Omega^\ast(\hatqE^s)\ \forall s\in\cS$ is a valid solution. Addressing this issue, \citet{rairl} proposed a reward operator $\PsiOmega\!:\!\Delta^\cS_\cA\hspace*{-3pt}\to\!\bR^{\cS\times\cA}$, providing a unique IRL solution by $\PsiOmega(\piE)$.
\begin{definition}[Regularized reward operators]\label{defn:rop}
Define the regularized reward operator $\PsiOmega$ as $\psi_\pi(s,a)\coloneqq \Omega^\prime(s,a;\pi)-\bigl\langle \pi^s\!,\, \nabla\Omega(\pi^s)\hspace*{-1pt}\bigr\rangle_{\!\!\scriptscriptstyle\cA}\!\!+ \Omega(\pi^s)$, for $\Omega^\prime(s,\cdot\,;\pi)\coloneqq\nabla\Omega(\pi^s)=\bigl[\nabla_p \Omega(p)\bigr]_{\! p=\pi(\cdot|s)}$.
\end{definition}
The reward function $\psiE\coloneqq\PsiOmega(\piE)$ replaces its state-action value function, since the sum of composite \Bg{} divergences derived from \eq{eq:reg_obj} allows reward learning in a greedy manner \citep{rairl}.

\begin{wrapfigure}{r}{0.47\textwidth}
\vspace*{-36pt}
\newcommand*{\PrimalSpace}{Policy Space}
\newcommand*{\DualSpace}{Reward Space}
\newcommand*{\PrimalGeometry}{$\Pi$}
\newcommand*{\PrimalGeometryU}{$\Delta^{\!\cS}_{\!\cA}$}
\newcommand*{\DualGeometry}{$\bR^{\cS\!\times\!\cA}$}
\newcommand*{\Convert}{$\PsiOmega$}
\newcommand*{\Reverse}{$\nabla\Omega^\ast$}
\newcommand*{\yslant}{0.09}
\newcommand*{\xslant}{0}
\newcommand*{\don}{2pt}
\newcommand*{\doff}{1.8pt}
\centering
\begin{adjustbox}{width=.405\textwidth,center}
\begin{tikzpicture}[tight background, every node/.style={inner sep=0,outer sep=0,minimum size=1cm},on grid]
\clip (-3.23, -1.59) rectangle+(6.46, 3.53);
\begin{scope}[xshift=-1.85cm, every node/.append style={yslant=\yslant}, yslant=\yslant, xslant=\xslant]
  \begin{scope}[thin]
    \begin{scope}
      \clip (-1.3,-1.45) rectangle (1.3,1.45);
      \draw[black] (-1.3,-1.45) rectangle (1.3,1.45);
    \end{scope}
    \begin{scope}
      \clip (-1.3,-1.45) -- (-1.35,-1.4) -- (-1.35,1.5) -- (-1.3,1.45) -- cycle;
      \draw[black,fill=gray] (-1.3,-1.45) -- (-1.35,-1.4) -- (-1.35,1.5) -- (-1.3,1.45) -- cycle;
    \end{scope}
    \begin{scope}
      \clip (-1.3,1.45) -- (1.3,1.45) -- (1.25,1.5) -- (-1.35,1.5) -- cycle;
      \draw[black,fill=gray] (-1.3,1.45) -- (1.3,1.45) -- (1.25,1.5) -- (-1.35,1.5) -- cycle;
    \end{scope}
  \end{scope}
  \draw[semithick,fill=superlightgray, rotate around={-50:(0.3,0.3)}] (0.3,0.3) ellipse (0.7 and 0.8);
  \draw[thin, rotate around={145:(1.58,2.1)}] (1.58,2.1) [partial ellipse=48:129:2.5cm and 3.5cm];
  \node at (-0.05, 1.68) {\PrimalSpace};
  \node at (-1.07, 1.17) {\PrimalGeometry};
  \node[yslant=\yslant,line width=0.2mm,draw,fill=black,circle,minimum size=3pt] (PA) at (0.6,0.6) {};
  \node[xslant=\yslant,line width=0.2mm,draw,fill=black,circle,minimum size=3pt] (PB) at (-0,0) {};
  \node[draw,fill=white,circle,minimum size=3pt] (PC) at (-0.8,-0.8) {};
  \fill[black] (PA) node [left=-0.2, scale=.9, xslant=-\yslant] {$\pi_{\hspace*{-.5pt}t}$} (PB) node [right, scale=.9, xslant=-\yslant] {$\pi_{\hspace*{-.3pt}t+1}$};
  \node[scale=.9] at (-.92,-1) {$\tilde{\pi}_{\hspace*{-.3pt}t+1}$};
  \node[scale=.8] at (-.85,-.05) {$\cP_\Omega$};
  \draw[-latex,  dash pattern={on \don off \doff}, verylightgray, very thick, shorten >=1pt] (PC) to[out=90,in=180] (PB);
\end{scope}
\begin{scope}[xshift=1.85cm, yslant=-\yslant,xslant=\xslant,every node/.append style={yslant=-\yslant}]
  \begin{scope}[thin]
    \begin{scope}
      \clip (-1.3,-1.45) rectangle (1.3,1.45);
      \draw[black] (-1.3,-1.45) rectangle (1.3,1.45);
    \end{scope}
    \begin{scope}
      \clip (1.3,-1.45) -- (1.35,-1.4) -- (1.35,1.5) -- (1.3,1.45) -- cycle;
      \draw[black,fill=gray] (1.3,-1.45) -- (1.35,-1.4) -- (1.35,1.5) -- (1.3,1.45) -- cycle;
    \end{scope}
    \begin{scope}
      \clip (-1.3,1.45) -- (1.3,1.45) -- (1.35,1.5) -- (-1.25,1.5) -- cycle;
      \draw[black,fill=gray] (-1.3,1.45) -- (1.3,1.45) -- (1.35,1.5) -- (-1.25,1.5) -- cycle;
    \end{scope}
  \end{scope}
  \node at (0.05,1.68) {\DualSpace};
  \draw[semithick,-{latex}] (-1.3,0.95) .. controls (0.1,0.8) and (0.7,-0.45) .. (1.,-1.25);
  \node[draw,xslant=-\yslant,fill=black,circle,minimum size=3.1pt] (DA) at (-0.55,0.73) {};
  \node[draw,xslant=-\yslant,fill=black,circle,minimum size=3.1pt] (DB) at (0.35,0.02) {};
  \node[draw,xslant=-\yslant,fill=white,circle,minimum size=3pt] (DC) at (0.69,0.5) {};
  \fill[black] (DA)  node [above=-0.2, scale=.9, xslant=\yslant] {$\psi_{\hspace*{-.3pt}t}$};
  \fill[black] (DB)  node [right, scale=.9, xslant=\yslant] {$\psi_{\hspace*{-.3pt}t+1}$};
  \node[scale=0.8] at (.92, 1.25) {\DualGeometry};
  \draw[-latex, dash pattern={on \don off \doff}, verylightgray, very thick,  shorten >=1pt] (DA) -- (DC);
\end{scope}
\draw[thin,-latex,shorten >=2pt](PA) to[out=18,in=168] (DA);
\draw[thin,-latex,shorten >=2pt](DB) to[out=-150,in=-30] (PB);
\draw[-latex,  dash pattern={on \don off \doff}, verylightgray, very thick, shorten >=1pt](DC) to[out=-45,in=-33] (PC);
\node at (0.03,1.16) {\Convert};
\node[scale=.85] at (0,-0.835) {\Reverse};
\end{tikzpicture}
\end{adjustbox}
\vskip-3pt
\caption{A schematic illustration. MD is locally constrained by a divergence (gray area), i.e., $D_\Omega(\,\cdot\, \Vert\hspace*{-.2pt} \pi_t\hspace*{-1pt})$. An MD update is performed for the reward function $\psi_t$ in an associated reward space of defined by $\PsiOmega$, and $\pi_{t+1}$ is achieved in the desired space of $\Pi$ by applying $\nabla\Omega^\ast$ for the function $\psi_{t+1}$. The gray dashed lines provide another interpretation of MD with $\tilde{\pi}_{t+1}$ and the  projection operator $\cPOmega$.} \label{fig:mdschema}
\vskip-23pt
\end{wrapfigure}
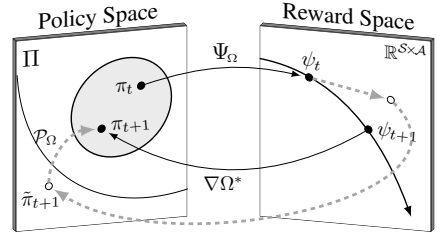
\section{RL-IRL as a Proximal Method}\label{sect:reward}

\textbf{Associated reward functions.\topicquad}
We consider the RL-IRL processes as a sequential algorithm with local constraints and define sequences $\{\pi_t\}_{t=1}^\infty$ and $\{\psi_t\}_{t=1}^\infty$ that denote policies and associated reward functions, respectively. The associated reward functions are in a space $\PsiOmega(\Pi)$, which is an alternative space of the dual space, defined by the regularized reward operator $\PsiOmega$. Formally, we provide \lem{lem:psi_unique}, which shows a bijective relation between the operators $\nabla\Omega^\ast$ and $\PsiOmega$ in the set $\Pi$. The proof is in \appsect{appsect:proof}.
\begin{lemma}[Natural isomorphism]\label{lem:psi_unique}
Let $\psi\in\PsiOmega(\Pi)$ for $\PsiOmega(X)\!\coloneqq\!\{\,\psi\mid\psi(s,a)=\psi_\pi(s,a),\ \forall s\!\in\!\cS, a\!\in\!\cA,\pi\!\in\!X\}$. Then, $\nabla\Omega^\ast(\psi)$ is unique and for every $\pi\! =\!\nabla\Omega^{\hspace*{-1pt}\ast}\!(\psi)$, $\pi\in\Pi$.
\end{lemma}
\fig{fig:mdschema} illustrates that there is a unique $\psi_t$ for $\pi_t$ in every time step. Note that $\PsiOmega(\pi_t)$ is different from $\nabla\Omega(\pi_t)$; it is shifted by a vector $\mathbf{1}c$ with a constant $c=\Omega(\pi^s_t)- \langle\pi^s_t, \nabla\Omega(\pi^s_t)\rangle_{\!\scriptscriptstyle\cA}$. Since the underlying space is a probability simplex, the operator $\nabla\Omega^\ast$ reconstructs the original point for both $\PsiOmega$ and $\nabla\Omega$, as the distributivity \citep{reg_mdp} $\Omega^\ast(y+\mathbf{1}c)=\Omega^\ast(y)+c$ holds (so $\nabla\Omega^\ast(y+\mathbf{1}c)=\nabla\Omega^\ast(y)$). An alternative interpretation is of considering a projection (gray dashed line in \fig{fig:mdschema}).  Suppose that a policy $\pi_t$ is updated to $\tilde{\pi}_{t+1}\,\textrm{\raisebox{.5pt}{$\textstyle\in$}}\,\bR^{\cS\times\cA}$. The \Bg{} projection operator $\cPOmega$ is applied that locates the subsequent update $\pi_{t+1}$ to the ``feasible'' region, i.e., $\cPOmega(\tilde{\pi}_{t+1}) \coloneqq \argmin_\piinPi[D_\Omega(\pi^s\Vert \tilde{\pi}^s_{t+1})]_\sincS$.

Consequently, one can consider an updated reward function $\psi_{t\hspace*{-.5pt}+\hspace*{-.5pt}1}$ as a projected target of MD associated with an alternative parameterization of $\Pi$. For instance, the parameters of $\psi_{t}$ can construct a softmax policy for a discrete space, or a \Gs{} policy for a continuous space. Using the reward function $\psi_{t\hspace*{-.5pt}+\hspace*{-.5pt}1}$, an arbitrary regularized RL process maximizing \eq{eq:reg_obj} at the $t$-th step \citep{rairl}
\begin{equation}\label{eq:mdlearn}
  \JOmega\hspace*{-.5pt}(\pi, \psi_{t+1})\!=\!-\hspace*{-1pt}\bE_{\pi}\!\Bigl[\sum\nolimits_{i=0}^\infty\hspace*{-1pt}\gamma^iD_\Omega\hspace*{-1pt}\Bigl(\hspace*{-1.1pt}\pi(\,\cdot\,|{s_i}\hspace*{-1pt})\Big\Vert\pi_{t+1}\hspace*{-1pt}(\,\cdot\,|s_i\hspace*{-1pt})\!\Bigr)\!\Bigr]\!
\end{equation}
becomes finding the next iteration $\pi_{t+1}=\nabla\Omega^\ast(\psi_{t+1})$ by maximizing the expected cumulative return. The equation shows that a regularized RL algorithm with the regularizer $\Omega$ forms a cumulative sum of \Bg{} divergences; thus, the policy $\pi_{t+1}$ is uniquely achieved by the property of divergence.

\begin{figure}[b]
  \centering
  \vskip-7pt
  \begin{adjustbox}{width=\textwidth,center}
  \begin{tikzpicture}[tight background, every node/.style={inner sep=0,outer sep=0},on grid]
  \clip (-1.81,-1.64) rectangle+(15.35, 3.286);

  \begin{scope}
    \node (vis) at (0, 0) {\includegraphics[height=91pt]{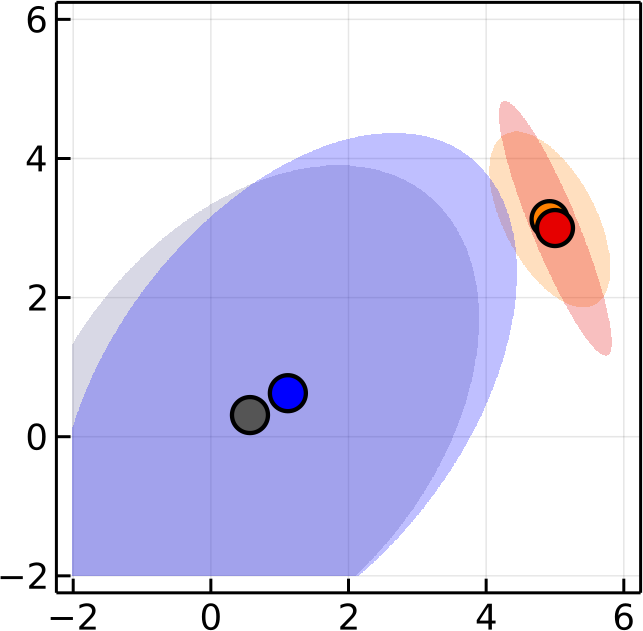}};
    \draw[fill={rgb:red,0.4;green,0.4;blue,0.6}] (-1.15,1.37) circle (1.4pt) node[xshift=0.27cm,black] {$\pi_t$};
    \draw[fill={rgb:red,0;green,0;blue,1}] (-.58,1.37) circle (1.4pt) node[xshift=0.42cm, yshift=-0.02cm, black] {$\pi_{t+1}$};
    \draw[fill=orange] (0.89,1.37) circle (1.4pt) node[xshift=0.37cm, yshift=-0.02cm, black] {$\barpiEt$};
    \draw[fill={rgb:red,0.9;green,0;blue,0}] (0.27,1.37) circle (1.4pt) node[xshift=0.27cm, yshift=-0.01cm, black] {$\piE$};
  \end{scope}
  \begin{scope}[xshift=3.65cm]
    \begin{scope}[xshift=0.4cm, yshift=0.14cm]
      \clip (-1.71, 0.11) rectangle+(0.99, 1.02);
      \node (demo1) {\includegraphics[height=65pt]{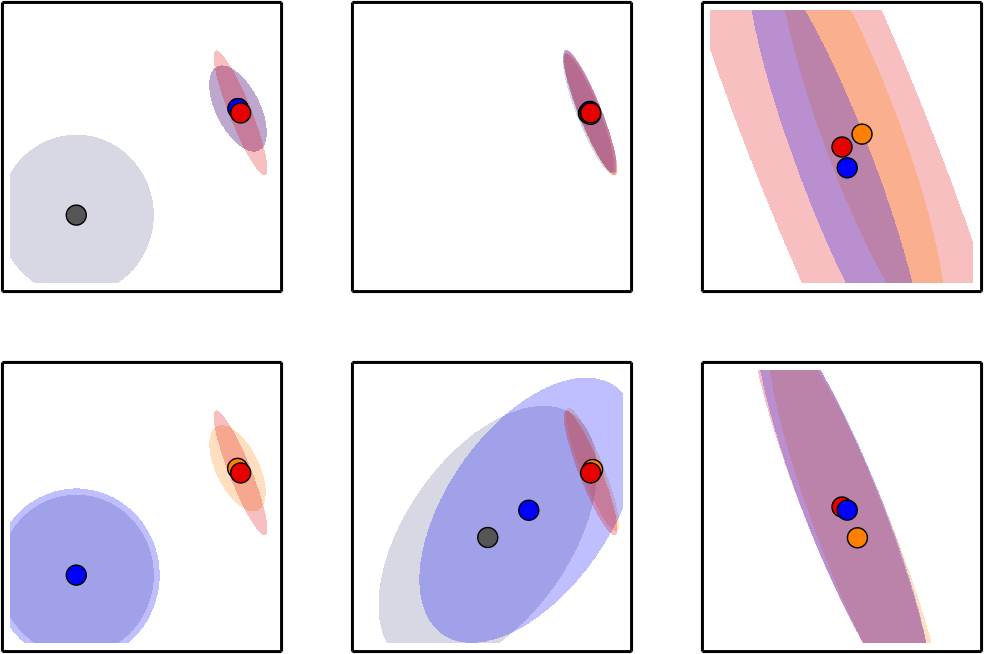}};
    \end{scope}
    \begin{scope}[xshift=0.28cm, yshift=0.14cm]
      \clip (-0.495, 0.11) rectangle+(0.99, 1.02);
      \node (demo2) at (0,0) {\includegraphics[height=65pt]{fig/toy_demo.png}};
    \end{scope}
    \begin{scope}[xshift=-0.07cm, yshift=0.11cm]
      \clip (0.95, 0.15) rectangle+(1.29, 1.34);
      \node (demo5) at (0,0) {\includegraphics[height=85pt]{fig/toy_demo.png}};
    \end{scope}
    \begin{scope}[xshift=0.4cm, yshift=-0.25cm]
      \clip (-1.71, -1.145) rectangle+(0.99, 1.02);
      \node (demo1) {\includegraphics[height=65pt]{fig/toy_demo.png}};
    \end{scope}
    \begin{scope}[xshift=0.28cm, yshift=-0.26cm]
      \clip (-0.495, -1.145) rectangle+(0.99, 1.02);
      \node (demo5) at (0,0) {\includegraphics[height=65pt]{fig/toy_demo.png}};
    \end{scope}
    \begin{scope}[xshift=-0.07cm, yshift=0.1cm]
      \clip (0.95, -1.49) rectangle+(1.29, 1.34);
      \node (demo5) at (0,0) {\includegraphics[height=85pt]{fig/toy_demo.png}};
    \end{scope}
    \draw[fill=black] (-1.2,1.43) circle (1.2pt) node[xshift=0.7cm,black] {\small\textsf{$\mathsf{\eta}=1.0$}};
    \draw[fill=black] (-1.2,-0.2) circle (1.2pt) node[xshift=0.7cm,black] {\small\textsf{$\mathsf{\eta}=0.2$}};
    \node at (-.76,0.12){\scriptsize\textsf{t=1}};
    \node at (0.3,0.12){\scriptsize\textsf{t=10}};
    \node at (1.6,0.12){\tiny\textsf{t=100 (4X)}};
    \node at (-.76,-1.53){\scriptsize\textsf{t=1}};
    \node at (.3,-1.535){\scriptsize\textsf{t=10}};
    \node at (1.6,-1.53){\tiny\textsf{t=100 (4X)}};
  \end{scope}
  \begin{scope}[xshift=8.1cm,yshift=0cm]
    \node at (0.06,-0.14) {\includegraphics[height=81pt]{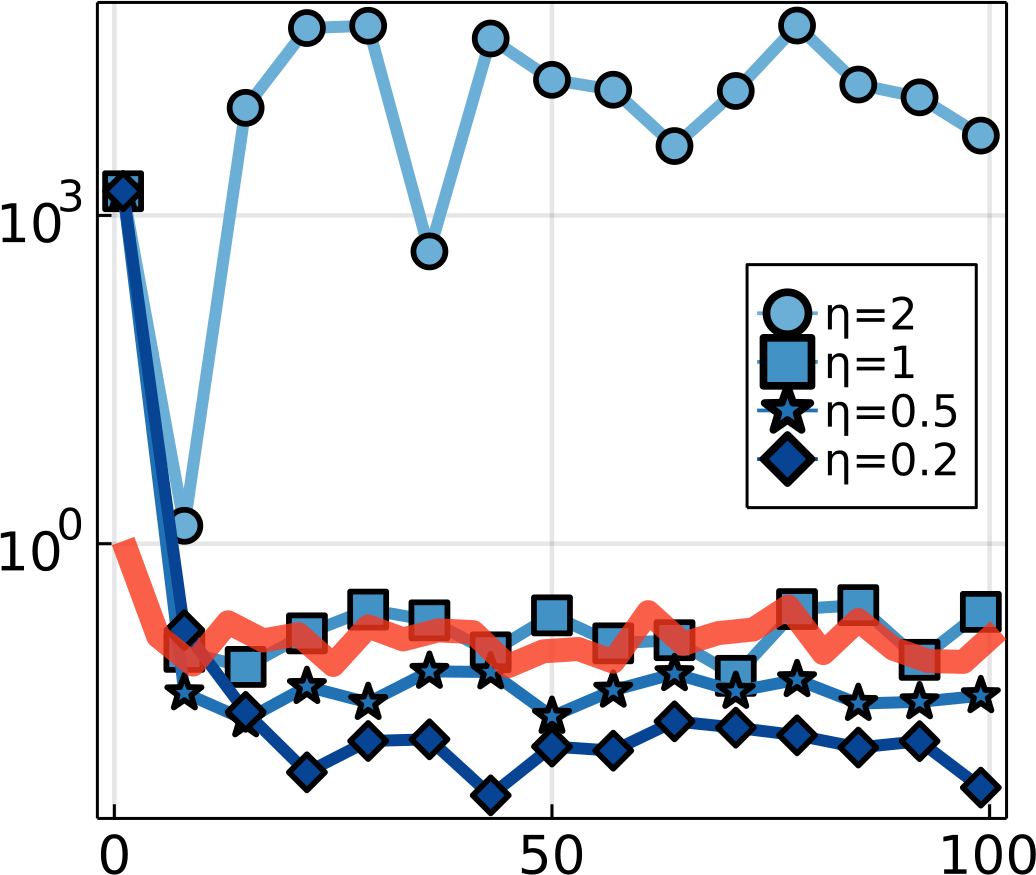}};
    \node[scale=0.9,rotate=90] at (-1.8,0){\small\textsf{\Bg{} div.}};
    \node[scale=0.9] at (0.15,1.455){{\textsf{\Sh{} regularizer}}};
    \node[scale=0.9,fill=white] at (0.19,-1.53){\scriptsize{\textsf{Time Step}}};
  \end{scope}
  \begin{scope}[xshift=11.6cm, yshift=0cm]
    \node at (0,-0.14) {\includegraphics[height=81pt]{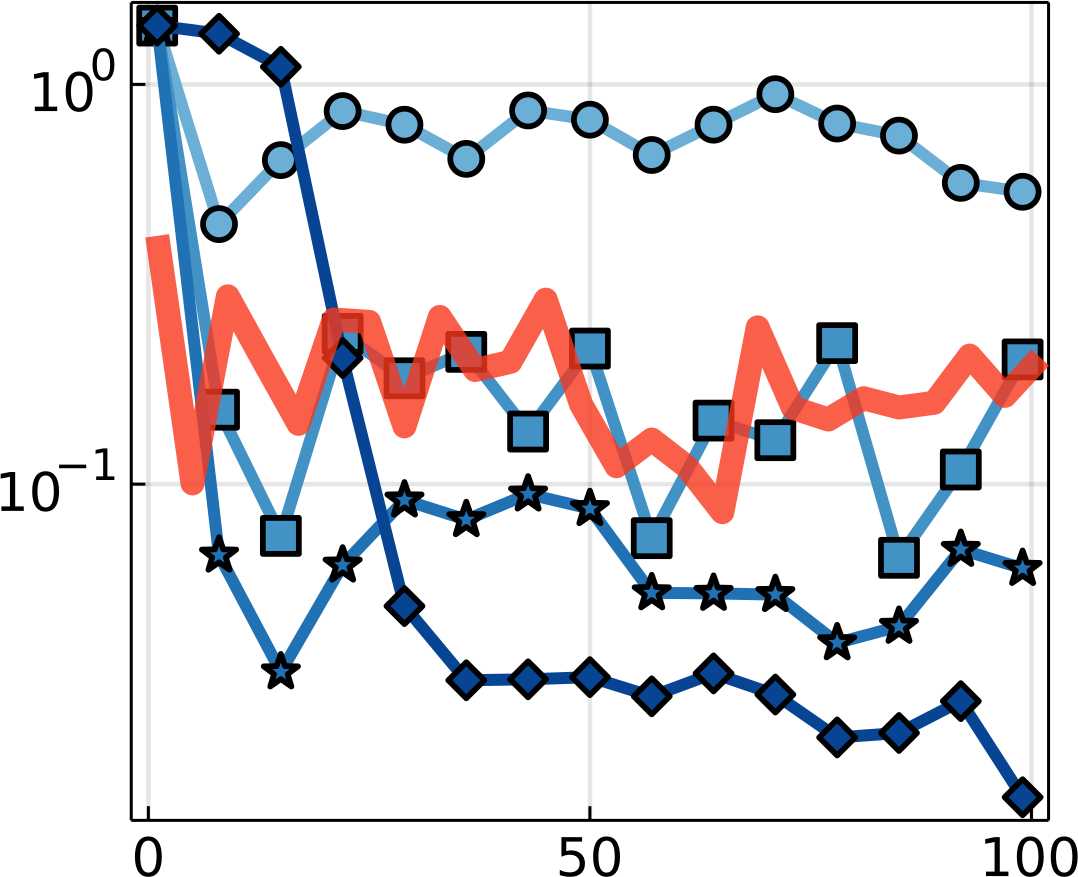}};
    \node[scale=0.9] at (0.2,1.455){{\textsf{\Ts{} regularizer}}};
    \node[scale=0.9,fill=white] at (0.19,-1.53){\scriptsize{\textsf{Time Step}}};
  \end{scope}

  \node[scale=1.28] at (-1.69,1.43) {\textsf{a}};
  \node[scale=1.28] at (2.12,1.43) {\textsf{b}};
  \node[scale=1.28] at (6.3,1.43) {\textsf{c}};
  \end{tikzpicture}
  \end{adjustbox}
  \caption{\subfigrefA{} A policy $\pi_t$ learns from MD updates for temporal costs
  $D_\Omega(\,\cdot\,\Vert \barpiEt)$. \subfigrefB{} The updates of $\pi_t$ vary by $\eta$, and the distance between $\pi_t$ and $\piE$ can be closer than the distance between $\barpiEt$ and $\piE$ when $t$ is sufficiently large and the $\eta$ is effectively low. \subfigrefC{} Two plots show $D_\Omega(\pi_t, \Vert \piE)$ associated with entropic regularizers for four different $\eta$ (10 trials), with the red baselines $D_\Omega(\barpiEt \Vert \piE)$.}\label{fig:md_toy}
\end{figure}

\textbf{Online imitation learning.\topicquad}
Our setup starts from the apparent yet vital premise that an imitation learning algorithm does not retain the global target $\piE$ during training. That is, it is fundamentally uncertain to model global objectives (such as $\JOmega\hspace*{-.5pt}(\pi\hspace*{-.5pt},\psiE\hspace*{-1pt})$), which are not attainable for both RL and IRL. Instead, we hypothesize on the existence of a random process $\{\barpiEt\}_{t=1}^\infty$ where each estimation $\barpiEt$ resides in a closed, convex neighborhood of $\piE$, generated by an arbitrary estimation algorithm. Substituting $\psiE$ to $\psi_{\barpiEt}\!=\!\PsiOmega(\barpiEt)$, the nonstationary objective $\JOmega\hspace*{-.5pt}(\pi\hspace*{-.5pt},\psi_{\barpiEt}\hspace*{-1pt})$ forms a temporal cost:
\begin{equation}\label{eq:F_def}
  F_t(\pi) = \,\bE_{\pi}\!\Bigl[\,\sum\nolimits_{i=0}^\infty \gamma^iD_\Omega\hspace*{-1pt}\bigl(\hspace*{-1pt}\pi(\,\cdot\,|s_i)\big\Vert\hspace*{1pt}\barpiEt\hspace*{-.5pt}(\,\cdot\,|s_i)\hspace*{-1pt}\bigr)\!\Bigr]\!.
\end{equation}
For the sake of better understanding, we considered an actual experiment depicted in \fig{fig:md_toy}. Suppose that the policies of the learning agent and the expert follow multivariate \Gs{} distributions at $\cN({\hspace*{-1pt}\textrm{\scriptsize \raisebox{.8pt}{$[0,0]$}}^{\hspace*{-1pt}\scriptscriptstyle\sT}}\hspace*{-2.5pt},\rI\hspace*{-.1pt})$ and $\cN({\hspace*{-1pt}\textrm{\scriptsize \raisebox{.8pt}{$[5,3]$}}^{\hspace*{-1pt}\scriptscriptstyle\sT}}\hspace*{-2.5pt},\Sigma_{\!\scriptscriptstyle E}\hspace*{-1pt})$ with $\lvert\Sigma_{\!\scriptscriptstyle E}\hspace*{-1pt}\rvert \!<\!1$. Let a (suboptimal) reference policy $\barpiEt$ be independently fitted with a maximum likelihood estimator with a relatively high learning rate, starting from $\barpiEone\!=\!\pi_1$. The policy $\pi_t$ was trained by a cost function $D_\Omega(\cdot \Vert \barpiEt)$ using the MD update rule in \eq{eq:md}. In \fig{fig:md_toy}, we first observed that choosing a high step size constant $\eta$ accelerated the training speed mainly in the early phase. The results also showed that the performance of MD ($D_\Omega(\pi_t\Vert\piE)$) outperformed that of referenced maximum likelihood estimation ($D_\Omega(\barpiEt\Vert\piE)$) by choosing an effectively low step size. This empirical evidence suggests that there are clear advantages in formalization of the training steps and scheduling the step sizes, especially for unreliable statistical model $\barpiEt$.

\textbf{MD update rules.\topicquad}
As a result of these findings, we formulate subsequent MD steps with a regularized reward function. Let $w_t$ be a parameter on a set $\cW$ and $F_t:\cW\to\bR$ be a convex cost function from a class of functions $\cF$ at the $t$-th step. Replacing the L2 proximity term of proximal gradient descent with a \Bg{} divergence, the proximal form of the MD update for \eq{eq:md} is written as \citep{smd}
\begin{equation}\label{eq:prox_bregman}
  \minimize_{w\in\cW}\,\bigl\langle\hspace*{-.5pt}\nabla F_t(w_t),\,w\!-\!w_t\hspace*{-1.pt}\bigr\rangle_{\hspace*{-1.5pt}\scriptscriptstyle\cW}\hspace*{-4pt}+\alpha_t D_\Omega\hspace*{-1.pt}\bigl(\hspace*{1.pt}w\hspace*{.5pt}\big\Vert\hspace*{1.pt} w_t\hspace*{-.5pt}\bigr),
\end{equation}
where $\alpha_t\coloneqq\nicefrac{1}{\eta_t}$ denotes an inverse of the current step size $\eta_t$ \citep{bregman_proximal}. Plugging each divergence of the cumulative cost $F_t$ to \eq{eq:prox_bregman}, the MD-IRL update for the subsequent reward function $\psi_{t+1} = \PsiOmega(\pi_{t+1})$ is derived by solving a problem
\begin{align}
  &\minimize_\pisinPis\hspace*{1.pt} \bigl\langle \hspace*{-2.2pt}\underbrace{\nabla\hspace*{-.5pt} D_\Omega\hspace*{-1pt}\bigl(\hspace*{1pt}\pi^s_{t}\hspace*{.5pt}\big\Vert\hspace*{1pt}\barpiEst\hspace*{-.5pt}\bigr)}_{\nabla \Omega(\pi^s_t)-\nabla\Omega(\barpiEst)}\hspace*{-2pt},\,\pi^s\hspace*{-2pt}-\pi^s_t\bigr\rangle_{\!\!\scriptscriptstyle\cA}\hspace*{-2pt}+\alpha_t\hspace*{1.pt}  D_\Omega\hspace*{-1.pt}\bigl(\hspace*{1pt}\pi^s\hspace*{.5pt}\big\Vert\hspace*{1.pt}\hspace*{1pt}\pi^s_t\hspace*{.5pt}\bigr)\nonumber\\
  &\qquad\iff \quad \minimize_\pisinPis D_\Omega\hspace*{-1pt}\bigl(\pi^s\hspace*{-.5pt}\big\Vert\hspace*{.2pt}\barpiEst\hspace*{-.7pt}\bigr)\hspace*{-.5pt}\!-\!D_\Omega\hspace*{-1pt}\bigl(\pi^s\hspace*{-.7pt}\big\Vert\hspace*{.1pt}\pi^s_t\hspace*{-.7pt}\bigr)\hspace*{-.5pt}\!+\!\alpha_tD_\Omega\hspace*{-1pt}\bigl(\pi^s\hspace*{-.8pt}\big\Vert\pi^s_t\hspace*{-.6pt}\bigr)\nonumber\\
  &\qquad\iff\quad\minimize_\pisinPis\hspace*{1pt}\eta_t\!\underbrace{D_\Omega\bigl(\pi^s\big\Vert\barpiEst\bigr)}_{\textrm{estimated expert}}\!+\,(\hspace*{-.3pt}1\!-\!\eta_t\hspace*{-.3pt})\underbrace{D_\Omega\bigl(\pi^s\big\Vert\pi^s_t\bigr)}_{\textrm{learning agent}} \qquad \forall s \in \cS, \label{eq:mdobj}
\end{align}
where the gradient of $D_\Omega$ is taken with respect to its first argument $\pi^s_t$. Note that
solving the optimization \eq{eq:mdobj} requires interaction between $\pi_t$ and the dynamics of the given environment in order to minimize $F_t$; thus, the corresponding RL process plays an essential role in sequential learning by the induced the value measures. At a glance, the objective is analogous to finding an interpolation at each iteration where the point is controlled $\eta_t$. \fig{fig:omd_demo} shows that the uncertainty of $\pi_t$ (blue region) gets minimal regardless of persisting uncertainty of $\barpiEt$ (red region).

\begin{figure}[H]
  \vskip-5.2pt
  \newcommand*{\secondxshift}{8cm}
  \centering
  \begin{adjustbox}{width=0.905\textwidth,center}
  \begin{tikzpicture}[tight background, every node/.style={inner sep=0,outer sep=0},on grid]
  \clip (-3.55, -1.11) rectangle+(15.7, 2.45);
  \begin{scope}
    \clip (-2.8,-0.61) rectangle+(6.06, 1.905);
    \node at (0,0) {\includegraphics[height=82pt]{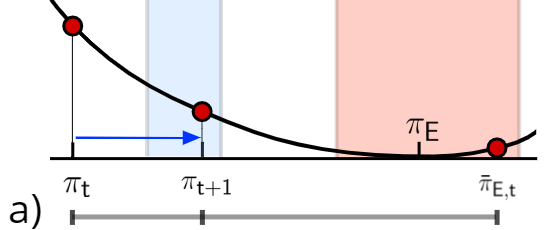}};
  \end{scope}
  \begin{scope}[yshift=0.167cm]
    \clip (-2.8,-1.04) rectangle+(6.06, 0.3);
    \node at (0,0) {\includegraphics[height=82pt]{fig/stepsize_1.png}};
  \end{scope}
  \begin{scope}[yshift=0.32cm]
    \clip (-2.8,-1.36) rectangle+(6.06, 0.18);
    \node at (0,0) {\includegraphics[height=82pt]{fig/stepsize_1.png}};
  \end{scope}
  \begin{scope}[xshift=\secondxshift]
  \begin{scope}
    \clip (-2.8,-0.61) rectangle+(6.06, 1.905);
    \node at (0,0) {\includegraphics[height=82pt]{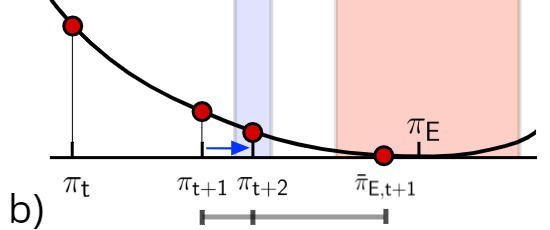}};
  \end{scope}
  \begin{scope}[yshift=0.167cm]
    \clip (-2.8,-1.04) rectangle+(6.06, 0.3);
    \node at (0,0) {\includegraphics[height=82pt]{fig/stepsize_2.png}};
  \end{scope}
  \begin{scope}[yshift=0.32cm]
    \clip (-2.8,-1.36) rectangle+(6.06, 0.18);
    \node at (0,0) {\includegraphics[height=82pt]{fig/stepsize_2.png}};
  \end{scope}
  \end{scope}
  \filldraw[fill=teal, very thick] (-2.53,1.11) circle (3.1pt);
  \filldraw[fill=teal, very thick] (-.918,0.04) circle (3.1pt);
  \filldraw[fill={rgb:red,202;green,8;blue,14}, very thick] (2.782,-.42) circle (3.1pt);
  \node[scale=1.38] at (-3.4,1.1) {\textsf{a}};
  \node[scale=0.65, rotate=90] at (-2.95,0.39) {\textsf{Ground-truth cost}};
  \node[scale=0.65] at (3.64,-1) {\color{darkgray}\textsf{\textbf{MD\,with}\,\textrm{\large\raisebox{.8pt}{$\bm{\eta}_{\hspace*{.4pt}\bm{t}}$}}}};
  \begin{scope}[xshift=\secondxshift, yshift=2.63cm]
    \filldraw[fill=teal, very thick] (-2.53,-1.514) circle (3.1pt);
    \filldraw[fill=teal, very thick] (-.918,-2.59) circle (3.1pt);
    \filldraw[fill=teal, very thick] (-.28,-2.84) circle (3.1pt);
    \filldraw[fill={rgb:red,202;green,8;blue,14}, very thick] (1.36,-3.14) circle (3.1pt);
    \node[scale=1.38] at (-3.4,-1.53) {\textsf{b}};
    \node[scale=0.65, rotate=90] at (-2.95,-2.233) {\textsf{Ground-truth cost}};
    \node[scale=0.65] at (2.5,-3.59) {\color{darkgray}\textsf{\textbf{MD\,with}\,\textrm{\large\raisebox{.88pt}{$\bm{\eta}_{\, \bm{t\textrm{\raisebox{1.5pt}{$\scriptscriptstyle+$}}1}}$}}}};
  \end{scope}
  \end{tikzpicture}
  \end{adjustbox}
  \vskip-1pt
  \caption{Illustrations of MD at the \subfigrefA{} $t$-th iteration and \subfigrefB{} $(t+1)$-th iteration where  $\eta_t\!>\!\eta_{t+1}$. $\{\hspace*{-1pt}\barpiEt\!\}_{t=1}^\infty$ is an arbitrary estimation process attained from a neighborhood of $\piE$ with respect to a norm. The MD update is taken inside the interval of $\pi_t$ and $\barpiEt$ using \eq{eq:mdobj}.}\label{fig:omd_demo}
  \vskip-7.8pt
\end{figure}

\section{Convergence Analyses}\label{sect:theory}
In this section, we present our theoretical results. The main goals of the following arguments are to address \enumone{} the convergence of MD updates for various cases and \enumtwo{} the necessity of scheduling the amount of learning. Suppose that state instances of $s^{(t)}_{i} \textrm{\scriptsize \raisebox{.3pt}{$\textstyle\in$}}\, \tau_t$ cover the entire $\cS$ by executing the policy $\pi_t$ in an infinite horizon. From this assumption, we define a temporal cost function at the time step $t$:
\begin{equation}\label{eq:f_def}
  \textstyle f(\pi_t,\tau_t) \hspace*{-1pt}\coloneqq\! \sum_{i=0}^\infty\gamma^i D_\Omega\hspace*{-1pt}\bigl(\hspace*{-.5pt}\pi_t(\;\cdot\;|\,s^{(t)}_i)\big\Vert\,\barpiEt(\;\cdot\;|\,s^{(t)}_i)\bigr),
\end{equation}
that involves $\pi_t$, and additionally a trajectory $\tau_t$ as inputs. We refer to the global objective as finding a unique fixed point $\pi_{\hspace*{-1pt}\ast}\,\textrm{\scriptsize \raisebox{.8pt}{$\textstyle\in$}}\,\Pi$ that minimizes a total cost $F(\pi) \coloneqq \bE[f(\pi,\tau_t)]$, where the expectation is taken over trajectories of entire steps, i.e., $\lim_{t\to\infty}\!\bE_{\tauonet}\hspace*{-1pt}[f(\pi, \tau_t)]$.
 Taking the (stepwise) gradient for each $\pi(\cdot|s)$, an optimal policy $\pi_{\hspace*{-1pt}\ast}$ is found by $ \bE[\nabla\Omega(\pi_\ast(\cdot|s))-\nabla\Omega(\barpiEt(\cdot|s))]=0$ when $t\to\infty$; hence, $\nabla\Omega(\pi_{\hspace*{-1pt}\ast}(\cdot|s))=\lim_{t\to\infty}\!\bE[\nabla\Omega(\barpiEt(\cdot|s))\hspace*{-.5pt}]$. Introducing the optimal policy $\pi_{\hspace*{-1pt}\ast}$ allows not only the specific situation when \enumone{} $\piE=\pi_{\hspace*{-1pt}\ast}\,\textrm{\scriptsize \raisebox{.8pt}{$\textstyle\in$}}\,\Pi$ and the estimation algorithm of $\barpiEt$ is actually convergent with $t\to\infty$, but also more general situations where \enumtwo{} $\piE\,\textrm{\scriptsize \raisebox{.8pt}{$\textstyle\notin$}}\,\Pi$ or the estimated expert policy $\barpiEt$ does not converge; the algorithm finds convergence to a fixed point by scheduling  updates.

We state two conditions of $\{\eta_t\}_{\!t=1}^{\!\infty}$ to guarantee convergence justified in Theorems~\ref{thm:step_size}~and~\ref{thm:conv_optimal}.\\
$\bullet$\hspace*{5pt}Convergent sequence \& divergent series:
\vskip-13pt
\begin{equation}\label{eq:step_cond}
  \textstyle\lim_{t\to \infty }\eta_t=0\qquad\text{and}\qquad\sum\nolimits_{t=1}^\infty \eta_t=\infty.
\end{equation}
$\bullet$\hspace*{5pt}Divergent series \& convergent series of squared terms:
\vskip-13pt
\begin{equation}\label{eq:step_cond2}
  \hskip3pt\textstyle\sum\nolimits_{t=1}^\infty \eta_t=\infty\hskip22.5pt\textrm{and}\qquad\sum\nolimits_{t=1}^\infty \eta^2_t < \infty.
\end{equation}
Let us assume \Lip{} continuity of $\nabla\Omega$ and boundedness of $D_\Omega$ in a \Bn{} space. In some $\Omega$, these two assumptions do not necessarily hold for extreme cases in $\Delta^\cS_\cA$, e.g., a distribution that $\pi(a|s) = 0$ for some entries. Nevertheless, these outliers can be left out if  the parametrization is constrained to satisfy the assumptions. For example, one can either \enumone{} prevent a policy from having non-zero entries of probabilities for a discrete policy or \enumtwo{} prevent a policy from having too low entropy for a continuous policy, by enforcing certain constraints on its parametric representations.

\thm{thm:step_size} argues that the sequence $\{\eta_t\}_{\! t=1}^{\!\infty}$ shall diverge for its series; therefore, \eq{eq:step_cond} is satisfied.
\begin{theorem}[Stepsize considerations]\label{thm:step_size}
Let $\Omega$ be strongly convex, $\nabla\Omega$ be \Lip{} continuous, and the associated \Bg{} divergence $D_\Omega$ is bounded. Assume a general condition of the problem that $\inf_{\!\piinPi}\!\bE\bigl[f(\pi,\tau_t)\bigr]\! >\!0$. Then we get $\lim_{T\to\infty}\!\bE_{\tauoneT}\!\bigl[\hspace*{1pt}\sum_{\scriptscriptstyle i=0}^\infty\hspace*{-1pt} D_\Omega\hspace*{-1pt}\bigl(\hspace*{1pt}\pi_{\hspace*{-1pt}\ast}(\hspace*{1pt}\cdot\hspace*{1pt}| {s_i})\big \Vert  \pi_{\!\scriptscriptstyle T}(\hspace*{1pt}\cdot\hspace*{1pt}|s_i))\hspace*{-1pt}\bigr]=0$ if and only if
\eqbrief{eq:step_cond} is satisfied.
\vspace*{-4pt}
\begin{enumerate}[leftmargin=*,label=(\alph*), noitemsep, partopsep=0pt, topsep=0pt, parsep=0pt]
  \item If $\lim_{t\to\infty}\eta_t=0$, then $T\,\textrm{\raisebox{.6pt}{$\textstyle\in$}}\,\bN$, $n\!<\!T$, and $c\!>\!0$ exist such that $\textstyle{\bE_{\tauoneT} \!\bigl[\hspace*{.5pt}f_T(\piT,\tauT)\bigr] \ge \frac{c}{T-n}}$.
  \item If the step size is in the form of  $\eta_t=\frac{4}{t+1}$, then $\textstyle{\bE_{\tauoneT}\!\bigl[\hspace*{1pt}\sum_{\scriptscriptstyle i=0}^\infty\hspace*{-1pt} D_\Omega\hspace*{-1pt}\bigl(\hspace*{1pt}\pi_{\hspace*{-1pt}\ast}(\hspace*{1pt}\cdot\hspace*{1pt}| {s_i})\big \Vert  \pi_{\!\scriptscriptstyle T}(\hspace*{1pt}\cdot\hspace*{1pt}|s_i))\hspace*{-1pt}\bigr]\!=\!\cO(1/T\hspace*{-.5pt})}$.
\end{enumerate}
\end{theorem}
Next, we present \thm{thm:conv_optimal}, which  addresses the convergence in a specific case when $\piE$ resides in $\Pi$. Additionally, the theorem addresses the bounds of the performance for fixed size update $\eta_t \equiv \eta_1$.
\begin{theorem}[Optimal cases]\label{thm:conv_optimal}
Let $\Omega$ be strongly convex, $\nabla\Omega$ be \Lip{} continuous, and the associated \Bg{} divergences be bounded. Assume $\pi_1 \ne \piE$ and $\inf_{\!\piinPi}\bE[f(\pi,\tau_t)]=0$. Then, $\bE\bigl[f(\pi_t,\tau_t)\bigr]=0$ if and only if $\sum_{t=1}^\infty \eta_t=\infty$. If $\eta_t \equiv \eta_1$, then there exist $c_1,c_2\in(0,1)$ such that $c_1^{T-1}\!\cdot\!A_1 \le A_T \le c_2^{T-1}\!\cdot\! A_1$, for $A_t= \sup_{s\in\cS}\bE_{\tauonet}\bigl[D_\Omega(\piE^s\Vert\pi^s_t)\bigr]$.
\end{theorem}
Lastly, \prop{prop:conv_general} provides the sufficient condition for the almost certain convergence of the algorithm by imposing the stronger condition of step size in \eq{eq:step_cond2}. The proofs are in \appsect{appsect:proof}.
\begin{proposition}[General cases]\label{prop:conv_general}
Assume that $\piE\notin\Pi$, hence $\inf_{\piinPi}\bE[f(\pi,\tau_t)]>0$. If the step sizes satisfies \eqbrief{eq:step_cond2}, then $\lim_{t\to\infty}\!\sum_{\scriptscriptstyle i=0}^\infty\!\gamma^iD_\Omega\hspace*{-1pt}\bigl(\pi_{\hspace*{-1pt}\ast}(\,\cdot\,|{s_i})\big\Vert\pi_{t}(\,\cdot\,|{s_i})\hspace*{-1pt}\bigr)$ converges to 0 almost surely.
\end{proposition}

\textbf{Regrets.\topicquad}For a sequence of state trajectories $\{\tau_t\}_{\hspace*{-1pt}t\textrm{\raisebox{.7pt}{$\scriptstyle\in$}}\bN}$, let us define a regret at the $t$-th iteration as
\begin{equation}
  \textstyle\frac{1}{t}\,{\textstyle\!\sum_{i=1}^t}f\hspace*{-.5pt}(\pi_i\hspace*{-1pt}, \tau_i\hspace*{-.5pt}) - \inf_{\!\piinPi}\bigl\{\hspace*{-1pt}\frac{1}{t}\,{\textstyle\!\sum_{j=1}^t}f\hspace*{-.5pt}(\pi\hspace*{-1pt}, \tau_j\hspace*{-.6pt} )\hspace*{-1.2pt}\bigr\}.
\end{equation}
In the optimal case of $\inf_{\!\piinPi}\bE[f(\pi,\tau_t)]=0$, the cost $f$ inherits the property of \Bg{} divergence so that the infimum is achieved by $0$ at $\piE$. In this case, the regret is bounded to $\cO(1/T)$ by the theorems. By \prop{prop:conv_general}, the MD updates converge for the case of $\inf_{\!\piinPi} \bE[f(\pi,\tau_t)]>0$ when the step sizes abide by \eq{eq:step_cond2}. Thus, the regret is bounded to $\cO(1/T)$  even for the general case.

\section{Algorithm: MD-IRL on an Adversarial Framework}\label{sect:mdirl}
In this section, we propose MD-AIRL, a novel AIL algorithm which trains a parameterized reward function with adversarial learning and the MD update rule. Neural network parameters $\theta$, $\phi$, and $\nu$ are newly presented representing agent policy, reward, and expert policy functions respectively.

\textbf{Dual discriminators.\ }
In order to bridge the gap between theory and practice, we propose a novel discriminative architecture, motivated by GAN studies regarding multiple discriminators \citep{dual, triple}. Basically, the proposed discriminators separate two concepts in AIL: matching overall state densities and imitating specific behavior. Given a learning agent policy $\pi_\theta$, an estimation policy $\pi_\nu$, and a discriminative neural network for states $d_\xi:\cS\to\bR$, the two discriminators are defined as
\begin{equation*}
  \,D_{\hspace*{-.5pt}\nu}\hspace*{-.5pt}(\hspace*{-.2pt}s\hspace*{-.3pt},a;\hspace*{1pt}{\scriptstyle\theta}\hspace*{-.5pt},{\scriptstyle\xi})=\sigma\hspace*{-.5pt}\bigl(\hspace*{.5pt}\log\bigl\{\hspace*{-.5pt}\pi_\nu(a|s)\big/\pi_\theta(a|s)\hspace*{-1pt}\bigr\}+d_\xi(s)\bigr) \textrm{\quad and\quad}
  D_{\hspace*{-.5pt}\xi}\hspace*{-.5pt}(s)=\sigma\hspace*{-.5pt}\bigl(\hspace*{.5pt}d_\xi(s)\bigr),\quad \forall\ s\in\cS\hspace*{-.5pt},\ a\in\cA,
\end{equation*}
where $\sigma(\cdot)$ denotes the sigmoid function.  The discriminators are trained using binary logistic regression losses with respect to mini-batch adversarial samples:
\begin{align}
  \maximize \cJ_{d_\xi} &=  \bE_{\pi_E}\bigl[\log D_\xi (s)\bigr] + \bE_{\pi_\theta}\bigl[\log (1 - D_\xi(s))\bigr],\\
  \maximize \cJ_{\pi_\nu} &= \bE_{\pi_E}\bigl[\log D_\nu (s, a)\bigr] + \bE_{\pi_\theta}\bigl[\log (1 - D_\nu(s,a)) \bigr],
\end{align}

\vskip-4pt
\begin{minipage}[b]{0.75\linewidth}
\centering
\begin{algorithm}[H]
\footnotesize
  \caption{\small\hspace*{3pt} Mirror Descent Adversarial Inverse Reinforcement Learning.}\label{alg:mdirl}
\begin{algorithmic}[1]
  \vspace*{-3.5pt}
  \State {\bfseries Input:} trajectories $\{\tau^\ast_t\hspace*{-1pt}\}_{\!\scriptscriptstyle t=1}^{\!\textrm{\raisebox{-2.5pt}{$\scriptscriptstyle T$}}}$, an agent $\pi_\theta$, a reference policy $\pi_\nu$, a neural network $d_\xi\hspace*{-3pt}:\!\cS\!\to\!\bR$, a regularized reward function $\psi_\phi\!\in\!\PsiOmega\hspace*{-1pt}(\Pi)$, $\alpha_1$,$\alphaT$, and $\lambda$.\vspace*{-.5pt}
  \For{$t \gets 1$ to $T$}
    \vspace*{-.5pt}
    \State $\alpha_t \gets \alpha_1+(t\!-\!1)(\nicefrac{\alpha_T-\alpha_1}{T-1})$\hspace*{7pt}and then\hspace*{7pt}$\eta_t \gets 1/\alpha_t$.
    \State Optimize $d_\xi$ and $\pi_\nu$ via binary logistic regression for $D_\xi$ and $D_\nu$.
    \vspace*{-.5pt}
    \State Optimize $\psi_\phi$ with the objective in \eq{eq:mdairlobj} using both $\tau^\ast_t$ and $\tau_t$.
    \vspace*{-.5pt}
    \State Train $\pi_\theta$ via RL to maximize $\psi^\lambda_\phi(s,a)$ with regularizer $\lambda\Omega(\cdot)$.
    \vspace*{-.5pt}
  \EndFor
  \vspace*{-3pt}
  \State {\bfseries Output:} $\pi_\theta$, $\psi^\lambda_\phi$.
  \vspace*{-1pt}
\end{algorithmic}
\vskip-2pt
\end{algorithm}
\end{minipage}\hfill
\begin{minipage}[b]{0.23\linewidth}
\centering
\begin{tikzpicture}[scale=0.4, transform shape,tight background, every node/.style={inner sep=0,outer sep=0},on grid]
\clip (-0.8,-0.34) rectangle+(7.98,6.0);
\begin{axis}[grid=both, mark=none, xmin=0.85, ymin=0, xmax=3.4, ymax=2.4,
  axis line style=ultra thick,
  axis lines=middle,
  enlargelimits=upper,
  clip=false,
  xtick={1,1.6,...,3.4},
  ytick={0,0.6,...,2.4},
  yticklabel=\empty,
  xticklabel=\empty,
  every axis plot/.append style={very thick},
  legend style={at={(0.65,0.53)},anchor=west}]
  \addplot[line width=0.65mm,color={rgb:red,202;green,8;blue,14}, domain=1:3.6,restrict y to domain=0:2.6,samples=100]{x-0.5};
  \addplot[line width=0.65mm,domain=1:3.6,restrict y to domain=0:2.6, samples=100]{1/(x-0.5)};
  \addplot[line width=1mm,dashed, domain=1:3.6,restrict y to domain=0:2.6, samples=100]{1-1/(x-0.5))};
  \node[circle,fill,inner sep=3.5pt] at (axis cs:1,2) {};
  \node[fill={rgb:red,202;green,8;blue,14},circle,fill,inner sep=3.5pt] at (axis cs:1,0.5) {};
  \legend{\LARGE$\alpha_t$,\LARGE$\eta_t$,\LARGE$1\!-\!\eta_t$};
\end{axis}
  \node at (7.05,-0.1) {\huge $t$};
  \node at (-0.44,1) {\huge $\alpha_1$};
  \node at (-0.44,4.3) {\huge $\eta_1$};
\end{tikzpicture}
\vskip-4.5pt

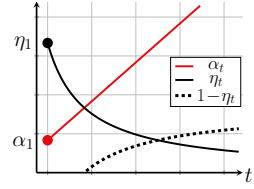
\captionof{figure}{A sequence example of $(\alpha_t, \eta_t)$.}\label{fig:step_eg}
\end{minipage}

where $d_\xi$ and $\pi_\theta$ are not trained for learning $D_\nu$.
Let $\rho_{\pi}\in\Delta_\cS$ denote the state visitation density of $\pi$, which is defined as $\rho_\pi(s)\coloneqq(1\!-\hspace*{-1pt}\gamma)\,\bE_\pi\!\bigl[\hspace*{1pt}\sum_{\scriptscriptstyle i=0}^\infty\gamma^i\,\textrm{\raisebox{-1.pt}{$\bI$}}\{s_i\!=\!s\}\!\hspace*{1pt}\bigr]$, where $\textrm{\raisebox{-1.pt}{$\bI$}}\{\cdot\}$ is an indicator function. The convergence of functions in an ideal case is found at $\pi_\nu\!=\!\piE$ and $d_\xi(s)\!=\!\log\{\hspace*{-.5pt}\rho_{\piE}(s)/\rho_{\pi_\theta}(s)\hspace*{-1pt}\}$.

\textbf{Learning with MD-based rewards.\topicquad}Based on the MD solution for a regularized reward function, we focus on developing an MD-based learning objective. Let $\psi_\phi\in\Psi(\Pi)$ denote a parameterized regularized reward function, and $\pi_{\phi}$ denotes a corresponding policy from $\phi$. Note that the transformation between $\psi_\phi$ and $\pi_\phi$ can be performed with shared $\phi$ without the additional computational costs under specific parameterizations \citep{rairl}. Using a step size $\eta_t$, the RL agent  $\pi_\theta$, and the estimated expert policy $\pi_\nu$, we define the objective of $\phi$ as a direct interpretation of the update rule of \eq{eq:mdobj}:
\begin{equation}\label{eq:mdairlobj}
  \minimize \cL_{\psi_\phi} = \bE_{s\sim\bar{\tau}_t}\hspace*{-1pt} \bigl[\eta_t\,D_\Omega\hspace*{-1pt}\bigl(\pi_\phi(\,\cdot\,|\hspace*{1pt}s\hspace*{1pt})\big\Vert\pi_\nu(\,\cdot\,|\hspace*{1pt}s\hspace*{1pt})\bigr) + (1\!-\!\eta_t) D_\Omega\hspace*{-1pt}\bigl(\pi_\phi(\,\cdot\,|\hspace*{1pt}s\hspace*{1pt})\big\Vert\pi_\theta(\,\cdot\,|\hspace*{1pt}s\hspace*{1pt})\bigr)\bigr],
\end{equation}
where the trajectory $\bar{\tau}_t$ denotes sample states using both agent and expert trajectories. As shown in \fig{fig:step_eg}, $\eta_t$ is adjusted by linearly increasing $\alpha_t$, which originated from the analyses in \sect{sect:theory}.

Another important consideration is the way of handling covariate shifts \citep{NEURIPS2021_07d59386} since it is likely that state densities between the expert and the agent are misaligned. Thus, we define the IRL reward function as linear combinations of $\psi_\phi$ and the state density discriminative signal:
\begin{equation}\label{eq:linear}
  \psi^\lambda_\phi(s,a)=\lambda\,\psi_\phi(s,a)+d_\xi(s),
\end{equation}
with a coefficient $\lambda\in\bR^+$. Utilizing an arbitrary regularized RL algorithm with a regularizer $\lambda\Omega(\cdot)$, the reward learning regarding agent policy $\pi_\theta$ is decomposed into the following:
\begin{equation*}
\begin{aligned}
  \bE_{\pi_\theta}\!\bigl[\hspace*{.5pt}\psi^\lambda_\phi(s,a)\!-\!\lambda\Omega\bigl(\pi_\theta(\,\cdot\,|\hspace*{.5pt}s\hspace*{.5pt})\bigr)\bigr]
  &=\lambda\bE_{\pi_\theta}\!\bigl[\hspace*{.5pt}\psi_\phi(s,a)\!-\!\Omega\bigl(\pi_\theta(\,\cdot\,|\hspace*{.5pt}s\hspace*{.5pt})\bigr)\bigr]\!-\!D_\KL\bigl(\hspace*{1pt}\rho_{\pi_\theta}\hspace*{-1pt}\big\Vert\hspace*{1pt}\rho_{\piE}\hspace*{-1pt}\bigr)\\
  &=-\lambda\bE_{\pi_\theta}\!\bigl[\hspace*{.5pt}D_\Omega\hspace*{-.5pt}\bigl(\pi_\theta(\,\cdot\,|s)\big\Vert\pi_\phi(\,\cdot\,|s)\bigr)\bigr]\!-\!D_\KL\bigl(\hspace*{1pt}\rho_{\pi_\theta}\hspace*{-1pt}\big\Vert\hspace*{1pt}\rho_{\piE}\hspace*{-1pt}\bigr).
\end{aligned}
\end{equation*}
Minimizing the first term of $\bE_{\pi_\theta}\!\bigl[D_\Omega(\pi^s_\theta\Vert\pi^s_\phi)\bigr]$ represents learning with the MD formulation. Minimizing the second term $D_\KL(\rho_{\pi_\theta}\Vert \rho_{\piE})$ plays an auxiliary role in facilitating the supports of state visitation densities to be correctly matched. With the hyperparameter $\lambda$, we report that learning the second term is helpful when the state densities are heavily misaligned in certain benchmarks. \alg{alg:mdirl} summarizes the entire procedure. We defer additional details to Appendices~\ref{appsect:full}~and~\ref{appsect:implement}.

\section{Experimental Results}\label{sect:expr}
The aim of our experiments was to identify whether MD-AIRL facilitates robustness for various $\Omega$ while retaining the scalability of AIL. The comparative method was RAIRL with density-based models (RAIRL-DBM) which contained comparable expressiveness as MD-AIRL. For RL, we used RAC \citep{rac}, which is a generalization of the SAC algorithm \citep{sac}. We considered a class of regularizers $\Omega(p)\!=\!-\bE_{x\sim p}[\varphi\hspace*{.1pt}(p(x)\hspace*{-.6pt})\hspace*{-.4pt}]$ with \enumone{} \Sh{} \ ($\varphi(x)\!=\!\log(x)$), \enumtwo{} \Ts{}   ($\varphi(x;q)=\frac{1}{q-1}\hspace*{-1pt}(x^{q-1}\hspace*{-1pt}-\hspace*{-1pt}1\hspace*{-1pt})$, $q=2$ by default), \enumthree{} $\exp$ ($\varphi(x)\!=\!e\!-\!e^x$), \enumfour{} $\cos$ ($\varphi(x)\!=\!\cos(\frac{\pi}{2}x)$), and \enumfive{} $\sin$ ($\varphi(x)\!=\!1\!-\sin(\frac{\pi}{2}x)$).

\subsection{Large scale multiarmed bandits}\label{subsect:bandit}
To measure the performance of IRL, we first considered multiarmed bandit problems, where the cardinality of action spaces varies largely. Learning the optimal distribution of $\piE$ becomes challenging as the cardinality of the space $\lvert\cA\vert$ increases, because the frequency of each sample becomes sparse due to the curse of dimensionality. The stateless expert distribution $\piE$ was generated by the parameters of softmax distribution $\piE(a)=\nicefrac{\exp(z_a)}{\sum_{i}\exp(z_i)}$, where the logits $z_i$ were randomly initialized. We set the action size to $\lvert\cA\rvert=\mathtt{10^2},\mathtt{10^3},\mathtt{10^4}$ and restricted each sample size of 16.

\begin{minipage}[b]{0.58\linewidth}
\centering
\captionof{table}{The training results of $\lvert\cA\rvert\cdot D_\Omega(\piT\Vert\piE)$
with five types of regularization (five runs with different seeds).}\label{tab:bandit}
\begin{adjustbox}{width=0.99\textwidth,center}
  \begin{tabular}{c c c c c c c}
    \toprule
    &\multicolumn{2}{c}{$\lvert\cA\rvert=\mathtt{10^2}$}&\multicolumn{2}{c}{$\lvert\cA\rvert=\mathtt{10^3}$}&\multicolumn{2}{c}{$\lvert\cA\rvert=\mathtt{10^4}$}\\\hline\rule{0pt}{\normalbaselineskip}
    {\footnotesize Method}&RAIRL&MD-AIRL&RAIRL&MD-AIRL&RAIRL&MD-AIRL\\\midrule
    {\scriptsize \Sh{}}&$2.55\pm1.59$&$\mathbf{2.28\pm1.20}$&$140.3\pm87.5$&$\mathbf{125.3\pm61}$&-&-\\
    {\footnotesize \Ts{}}&$0.21\pm0.13$&$\mathbf{0.11\pm0.04}$&$0.55\pm0.13$&$\mathbf{0.24\pm0.03}$&$4.95\pm2.3$&$\mathbf{4.21\pm0.2}$\\
    $\exp$&$0.27\pm0.17$&$\mathbf{0.13\pm0.06}$&$0.55\pm0.12$&$\mathbf{0.23\pm0.03}$&$5.06\pm2.4$&$\mathbf{4.97\pm0.7}$\\
    $\cos$&$0.05\pm0.04$&$\mathbf{0.02\pm0.01}$&$0.03\pm0.02$&$\mathbf{0.01\pm0.01}$&$0.21\pm0.6$&$\mathbf{0.05\pm0.1}$\\
    $\sin$&$0.34\pm0.25$&$\mathbf{0.12\pm0.04}$&$3.82\pm3.46$&$\mathbf{1.07\pm0.75}$&$8.12\pm3.8$&$\mathbf{7.59\pm1.0}$\\
    \bottomrule
  \end{tabular}
\end{adjustbox}
\end{minipage}\hfill
\begin{minipage}[b]{0.41\linewidth}
\newcommand*{\banditfigheight}{65pt}
\begin{adjustbox}{width=0.99\textwidth,center}
  \begin{tikzpicture}[tight background, every node/.style={inner sep=0,outer sep=0},on grid]
  \clip (-0.9,-1.25) rectangle+(8.18,2.4);
  \node at (0.1,0) {\includegraphics[height=\banditfigheight]{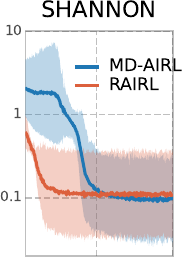}};
  \node at (1.64,0) {\includegraphics[height=\banditfigheight]{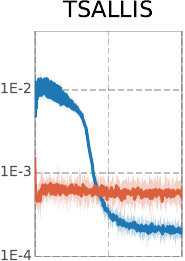}};
  \node at (3.25,0) {\includegraphics[height=\banditfigheight]{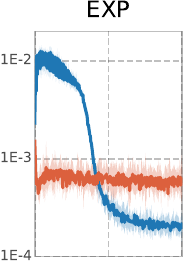}};
  \node at (4.85,0) {\includegraphics[height=\banditfigheight]{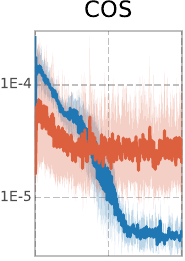}};
  \node at (6.47,0) {\includegraphics[height=\banditfigheight]{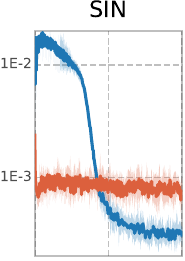}};
  \node[rotate=90] at (-0.8,0){\tiny\textsf{\Bg{} div.}};
  \node[scale=0.7] at (0.58,-1.19){\tiny\textsf{0.3M}};
  \node[scale=0.7] at (2.24,-1.19){\tiny\textsf{0.3M}};
  \node[scale=0.7] at (3.86,-1.19){\tiny\textsf{0.3M}};
  \node[scale=0.7] at (5.5,-1.19){\tiny\textsf{0.3M}};
  \node[scale=0.7] at (7.1,-1.19){\tiny\textsf{0.3M}};
  \end{tikzpicture}
\end{adjustbox}
\vskip-6pt
\captionof{figure}{The cost $D_\Omega(\piT\Vert\piE)$ on the log-scale at $\lvert\cA\vert=\mathtt{10^3}$. The shade represents $95\%$ confidence interval.}
\vskip-26pt
\label{fig:bandit}
\end{minipage}

\medskip
\tab{tab:bandit} shows that MD-AIRL achieved overall lower \Bg{} divergence on average when three different cardinalities and five regularizers were considered. \fig{fig:bandit} shows that the \Bg{} divergence was large for MD-AIRL at the early training phase, because we chose the initial step size $\eta_1$ to be greater than $1$ ($\alpha_1=0.5$). MD-AIRL exceeded the discriminative performance of RAIRL after certain steps, while the progression of RAIRL mostly stopped at local minima. MD-AIRL outperformed RAIRL in four cases by choosing an effectively low step size at $\etaT$ to be less than $1$ ($\alphaT=2$). These results match the properties of MD algorithms and our convergence analyses. Therefore, we argue that a constrained update rule with appropriate step sizes is necessary for robust reward acquisition and imitation for situations when the total number of data samples is limited.

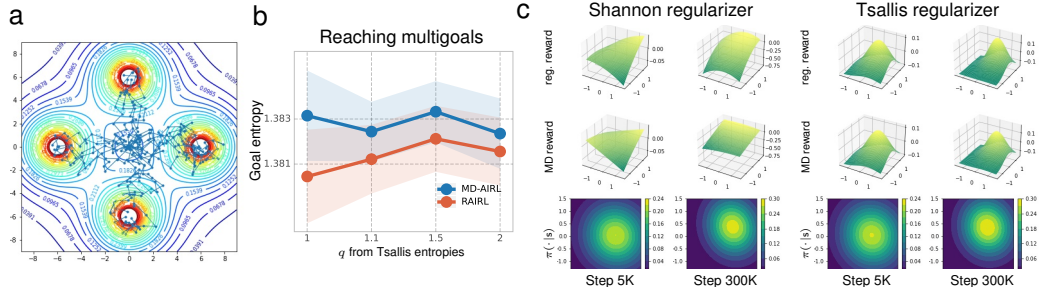
\begin{figure}[h]
  \vskip-4pt
  \centering
  \begin{adjustbox}{width=0.98\columnwidth,center}
  \begin{tikzpicture}[tight background, every node/.style={inner sep=0,outer sep=0},on grid]
  \clip (-4.6, -1.68) rectangle+(16.05, 4.5);
  \node[scale=1.2] at (-4.5,2.53) {\textsf{a}};
  \node[scale=1.2] at (-0.7,2.56) {\textsf{b}};
  \node at (-2.8,0.5 ) {\includegraphics[height=102pt]{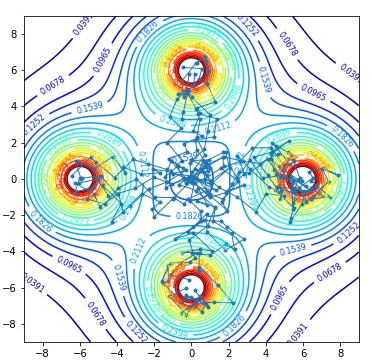}};
  \node at (1.3,0.52) {\includegraphics[height=85pt]{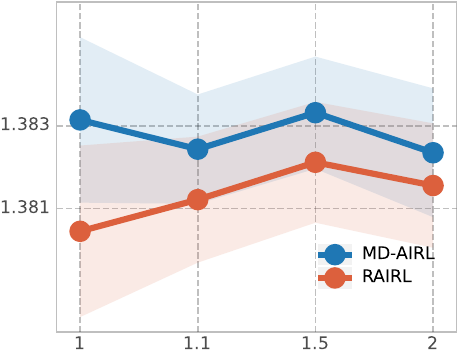}};
  \node[scale=0.9, rotate=90] at (-0.76,0.6){\scriptsize\textsf{Goal entropy}};
  \node[scale=0.9] at (1.5,2.2){\footnotesize\textsf{Reaching multigoals}};
  \node[scale=0.9] at (1.5,-1.15){\tiny\textsf{$q$ from \Ts{} entropies}};
  \begin{scope}[xshift=4.5cm, yshift=0.5cm]
    \node[scale=1.2] at (-1.1,2.13) {\textsf{c}};
    \begin{scope}
      \node at (-0.7,1.4) [rotate=90,scale=0.8] {\tiny\textsf{reg. reward}};
      \node at (-0.7,0) [rotate=90,scale=.8] {\tiny\textsf{MD reward}};
      \node at (-0.7,-1.5) [rotate=90,scale=.8] {\tiny$\mathsf{\pi(\cdot|s)}$};
      \node[scale=0.9] at (1.2, 2.12)  {\footnotesize\textsf{\Sh{} regularizer}};
      \node at (0.3, 0) {\includegraphics[height=108pt]{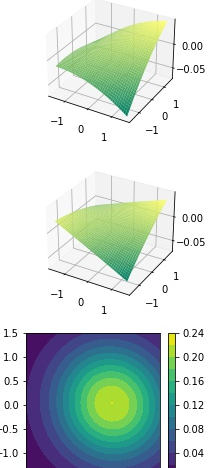}};
      \node at (2.1, 0) {\includegraphics[height=108pt]{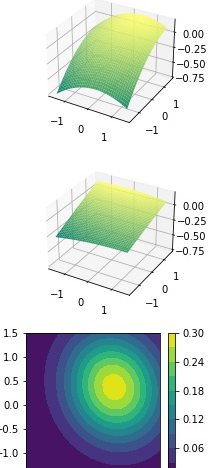}};
      \node at (0.25,-2.1) {\tiny\textsf{Step 5K}};
      \node at (2.05,-2.1) {\tiny\textsf{Step 300K}};
    \end{scope}
    \begin{scope}[xshift=4.cm]
      \node at (-0.7,1.4) [rotate=90,scale=0.8] {\tiny\textsf{reg. reward}};
      \node at (-0.7,0) [rotate=90,scale=.8] {\tiny\textsf{MD reward}};
      \node at (-0.7,-1.5) [rotate=90,scale=.8] {\tiny$\mathsf{\pi(\cdot|s)}$};
      \node[scale=0.9] at  (1.2, 2.12) {\footnotesize\textsf{\Ts{} regularizer}};
      \node at (0.3, 0) {\includegraphics[height=108pt]{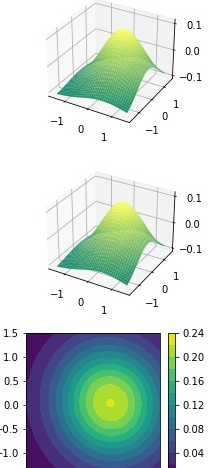}};
      \node at (2.1, 0) {\includegraphics[height=108pt]{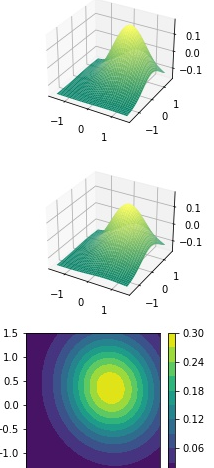}};
      \node at (0.25,-2.1) {\tiny\textsf{Step 5K}};
      \node at (2.05,-2.1) {\tiny\textsf{Step 300K}};
    \end{scope}
  \end{scope}
  \end{tikzpicture}
  \end{adjustbox}
  \vskip-4pt
  \caption{\subfigrefA{} Visualization of trajectories trained by MD-AIRL, and the ground-truth reward surface. \subfigrefB{} The entropies for the probabilities of achieving four goals. The x-axis indicates the $q$ value from the \Ts{} regularizers (the \Sh{} regularizer is considered by $q\!=\!1$ \citep{tsallis}). \subfigrefC{} The top and middle of each column show induced reward surfaces. The bottom shows the agent policy.}\label{fig:multigoal_result}
  \vskip-12pt
\end{figure}

\subsection{A continuous multigoal environment}\label{subsect:multigoal}
We then considered a multigoal environment. In this environment, an agent is a two-dimensional point mass initialized at the origin, and the four goals are located in the four cardinal directions. The objective of imitation learning is to go toward each direction evenly as possible where the expert model was trained by the SAC algorithm. To draw informative reward surfaces regarding stochastic actions, we considered the multivariate \Gs{} distribution policies parameterized with full covariance matrices instead of conventional diagonal \Gs{} policies (see Appendices~\ref{appsect:full}~and~\ref{appsect:implement}).

\fig{fig:multigoal_result}~(a) shows trajectories generated by the trained agent. \fig{fig:multigoal_result}~(b) shows that MD-AIRL achieved higher entropy for reaching the multiple goals. \fig{fig:multigoal_result}~(c) shows reward surfaces with regularizers, which were calculated by $\psi_\phi(s,a)+\varphi(\pi_\theta (a|s))$ for each point of $a\in\cA$ and $s=(5,-1)$. During the training, the MD reward was similar to the estimated ground truth using adversarial training. However, the surface of MD-AIRL became flatter than the ground-truth estimation when $\pi_t$ was sufficiently close to the expert behavior. As a result, we claim that a drastic change in the target distribution, which is one of the typical characteristics of adversarial frameworks, is prevented. We argue that these characteristics mitigate overfitting caused by unreliable discriminative signals.

\begin{figure}
  \newcommand*{\curveheight}{56pt}
  \newcommand*{\titlescale}{0.6}
  \newcommand*{\xaxistitlescale}{0.38}
  \newcommand*{\yaxistitlescale}{0.52}
  \newcommand*{\xaxistitle}{Time step}
  \newcommand*{\yaxistitle}{Score}
  \newcommand*{\legendlinewidth}{.6mm}
  \newcommand*{\legendmargin}{1.9pt}
  \newcommand*{\legendscale}{.52}
  \definecolor{mdcolor}{HTML}{1F77B4}
  \definecolor{regcolor}{HTML}{DC603D}
  \definecolor{randcolor}{HTML}{399A47}
  \definecolor{expcolor}{HTML}{4B4B4B}
  \definecolor{bccolor}{HTML}{BB67CA}

  \centering
  \begin{tikzpicture}[tight background, every node/.style={inner sep=0,outer sep=0},on grid]
  \clip (-1.88, -.15) rectangle+(13.85, 2.84);
  \begin{scope}[yshift=1.4cm]
    \node at (0,0.03) {\includegraphics[height=59pt]{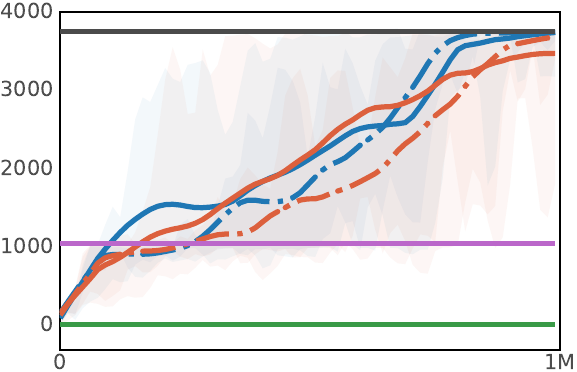}};
    \node at (3.4,0) {\includegraphics[height=\curveheight]{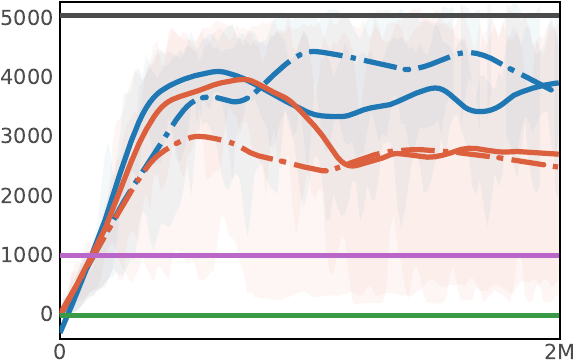}};
    \node at (6.8,0) {\includegraphics[height=\curveheight]{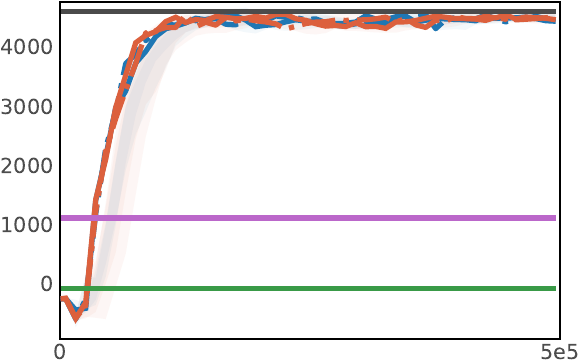}};
    \node at (10.2,0) {\includegraphics[height=\curveheight]{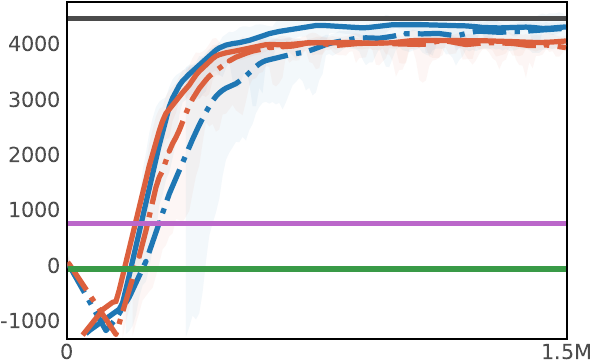}};
    \node[scale=\titlescale] at (0.12, 1.15) {\textsf{Hopper-v3}};
    \node[scale=\titlescale] at (3.52, 1.15) {\textsf{Walker2d-v3}};
    \node[scale=\titlescale] at (6.92, 1.15) {\textsf{HalfCheetah-v3}};
    \node[scale=\titlescale] at (10.32, 1.15) {\textsf{Ant-v3}};
    \node[scale=\xaxistitlescale] at (0.15,-1.05) {\textsf{\xaxistitle}};
    \node[scale=\xaxistitlescale] at (3.55,-1.05) {\textsf{\xaxistitle}};
    \node[scale=\xaxistitlescale] at (6.95,-1.05) {\textsf{\xaxistitle}};
    \node[scale=\xaxistitlescale] at (10.35,-1.05) {\textsf{\xaxistitle}};
    \node[rotate=90, scale=\yaxistitlescale] at (-1.72,0.04) {\textsf{\yaxistitle}};
  \end{scope}
  \begin{scope}[xshift=5.03cm, yshift=0.05cm]
    \draw[color=expcolor, line width=\legendlinewidth] (-6.1, 0) -- (-5.7, .0) node[scale=\legendscale, right=\legendmargin] {\textsf{expert}};
    \draw[color=mdcolor, line width=\legendlinewidth] (-4.7, .0) -- (-4.3, 0) node[scale=\legendscale, right=\legendmargin] {\textsf{MD-AIRL\hspace*{2pt}(\Sh{})}};
    \draw[color=mdcolor, dash pattern={on 3pt off 1.5pt on 1.5pt off 1.5pt}, line width=\legendlinewidth] (-2.1, .0) -- (-1.7, 0) node[scale=\legendscale, right=\legendmargin] {\textsf{MD-AIRL\hspace*{2pt}(\Ts{})}};
    \draw[color=regcolor, line width=\legendlinewidth] (0.15, .0) -- (0.55, .0) node[scale=\legendscale, right=\legendmargin] {\textsf{RAIRL (\Sh{})}};
    \draw[color=regcolor, dash pattern={on 3pt off 1.5pt on 1.5pt off 1.5pt}, line width=\legendlinewidth] (2.4, .0) -- (2.8, .0) node[scale=\legendscale, right=\legendmargin] {\textsf{RAIRL (\Ts{})}};
    \draw[color=bccolor, line width=\legendlinewidth] (4.4,0) -- (4.8,0) node[scale=\legendscale, right=\legendmargin] {\textsf{bc}};
    \draw[color=randcolor, line width=\legendlinewidth] (5.4,0) -- (5.8,0) node[scale=\legendscale, right=\legendmargin] {\textsf{random}};
    \draw[line width=0.1mm, black] (-6.3, -.17) rectangle+(13,.36);
  \end{scope}
  \end{tikzpicture}
  \vspace*{-4pt}
  \caption{Average scores for 5 runs with two different regularizers (\Sh{} and \Ts{} regularizer).  The agent and IRL reward functions were trained with 4 episodes of expert demonstrations.}\label{fig:curve}
  \vskip-6pt
\end{figure}
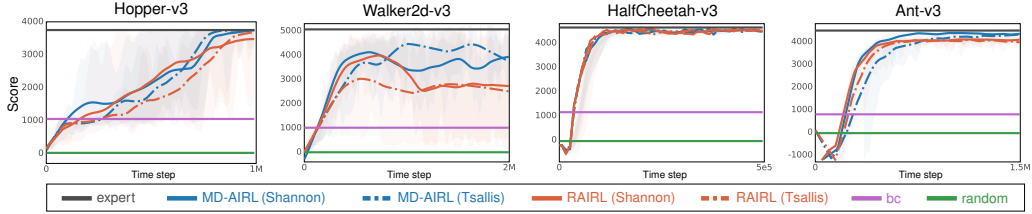

\begin{figure}
  \newcommand*{\summaryheight}{67.5pt}
  \newcommand*{\titlescale}{0.55}
  \newcommand*{\xaxistitlescale}{0.45}
  \newcommand*{\yaxistitlescale}{0.47}
  \newcommand*{\xaxistitle}{Time step}
  \newcommand*{\yaxistitle}{Score}
  \newcommand*{\legendlinewidth}{.6mm}
  \newcommand*{\legendmargin}{1.5pt}
  \newcommand*{\legendscale}{.47}
  \centering
  \begin{adjustbox}{width=.99\textwidth,center}
  \begin{tikzpicture}[tight background, every node/.style={inner sep=0,outer sep=0},on grid]

  \begin{scope}
    \node at (0,0) {\includegraphics[height=\summaryheight]{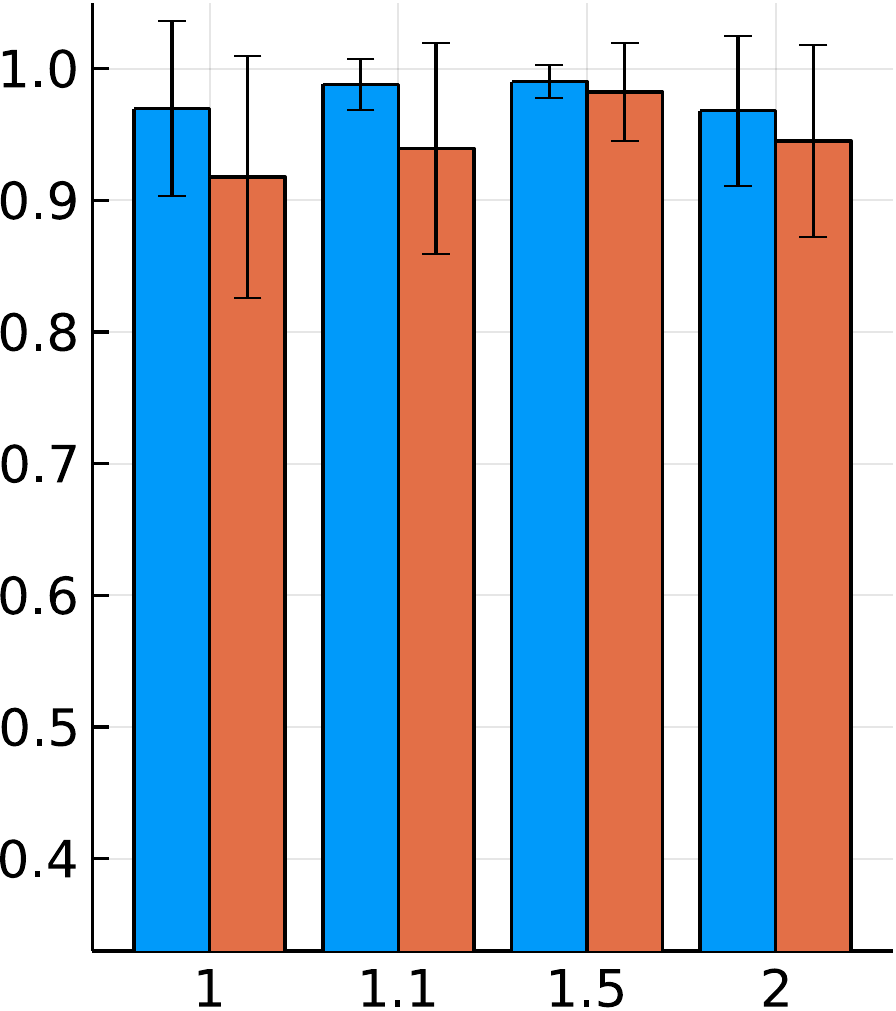}};
    \node at (2,0) {\includegraphics[height=\summaryheight]{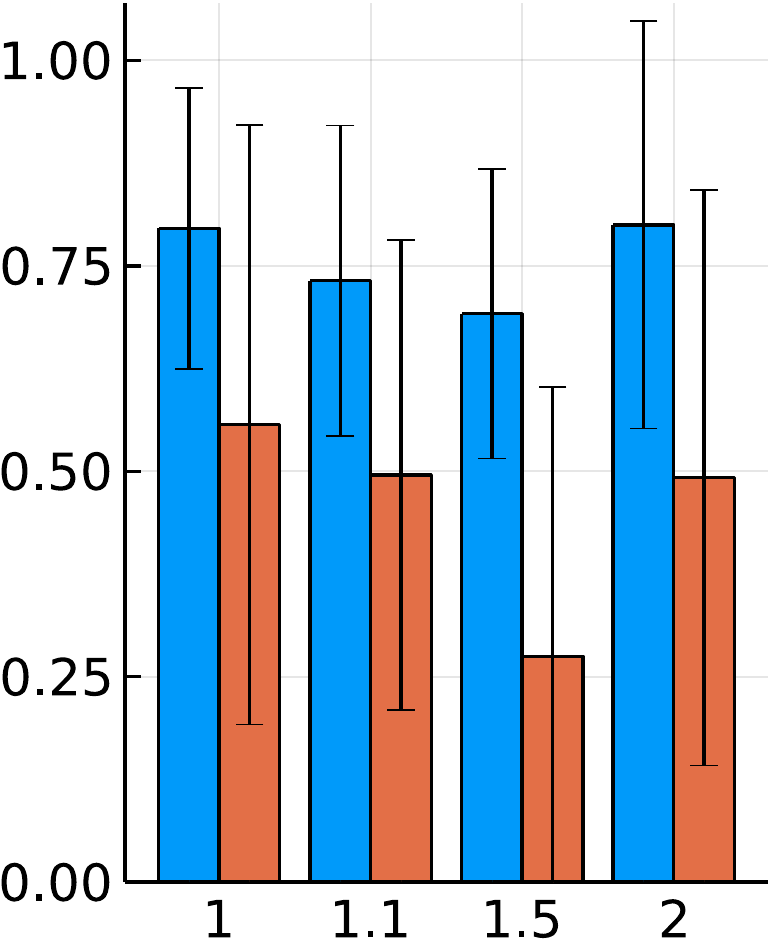}};
    \node at (4,0) {\includegraphics[height=\summaryheight]{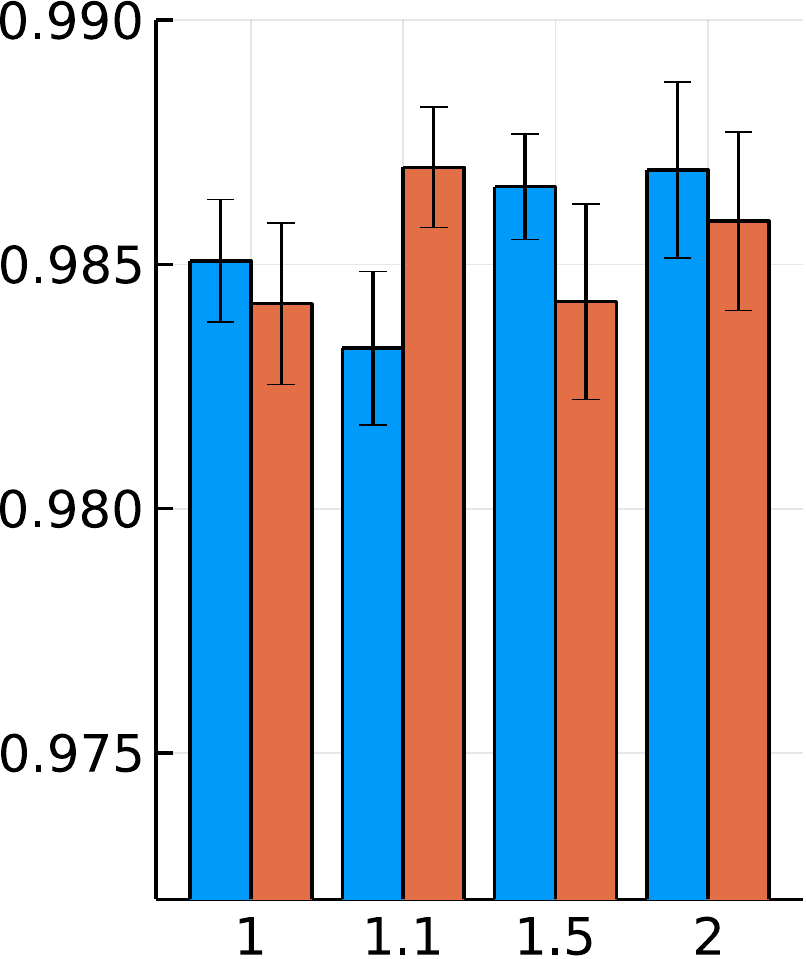}};
    \node at (6,0) {\includegraphics[height=\summaryheight]{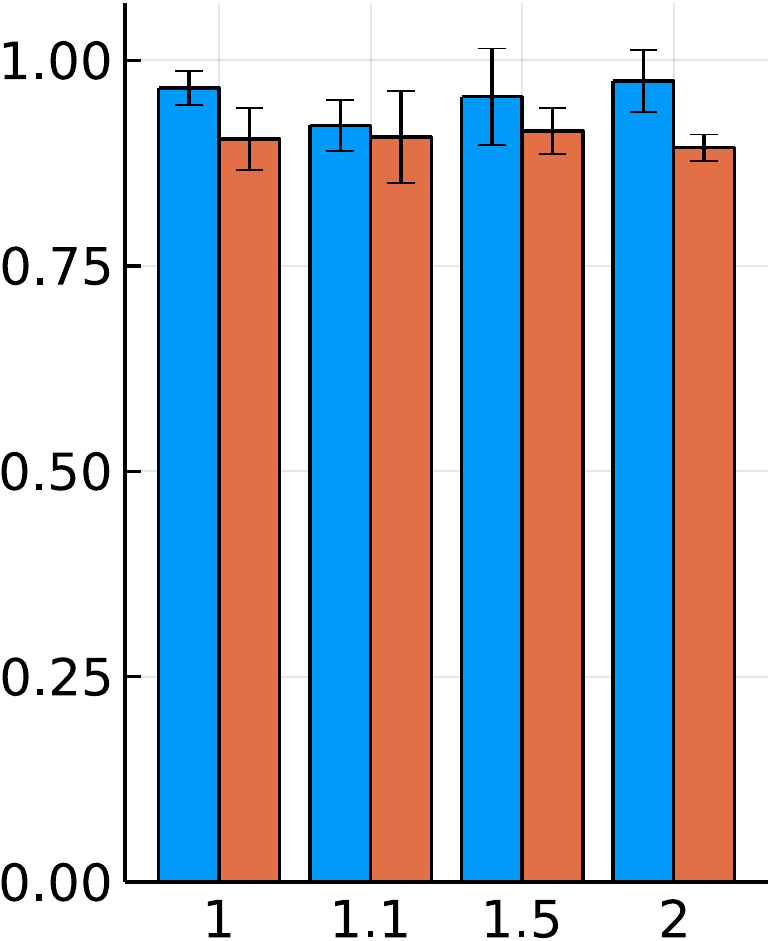}};
    \node[scale=\titlescale] at (0.1, 1.24) {\textsf{Hopper}};
    \node[scale=\titlescale] at (2.08, 1.24) {\textsf{Walker2d}};
    \node[scale=\titlescale] at (4.2, 1.24) {\textsf{HalfCheetah}};
    \node[scale=\titlescale] at (6.1, 1.24) {\textsf{Ant}};
  \end{scope}
  \begin{scope}[xshift=8.5cm]
    \node at (0,0) {\includegraphics[height=\summaryheight]{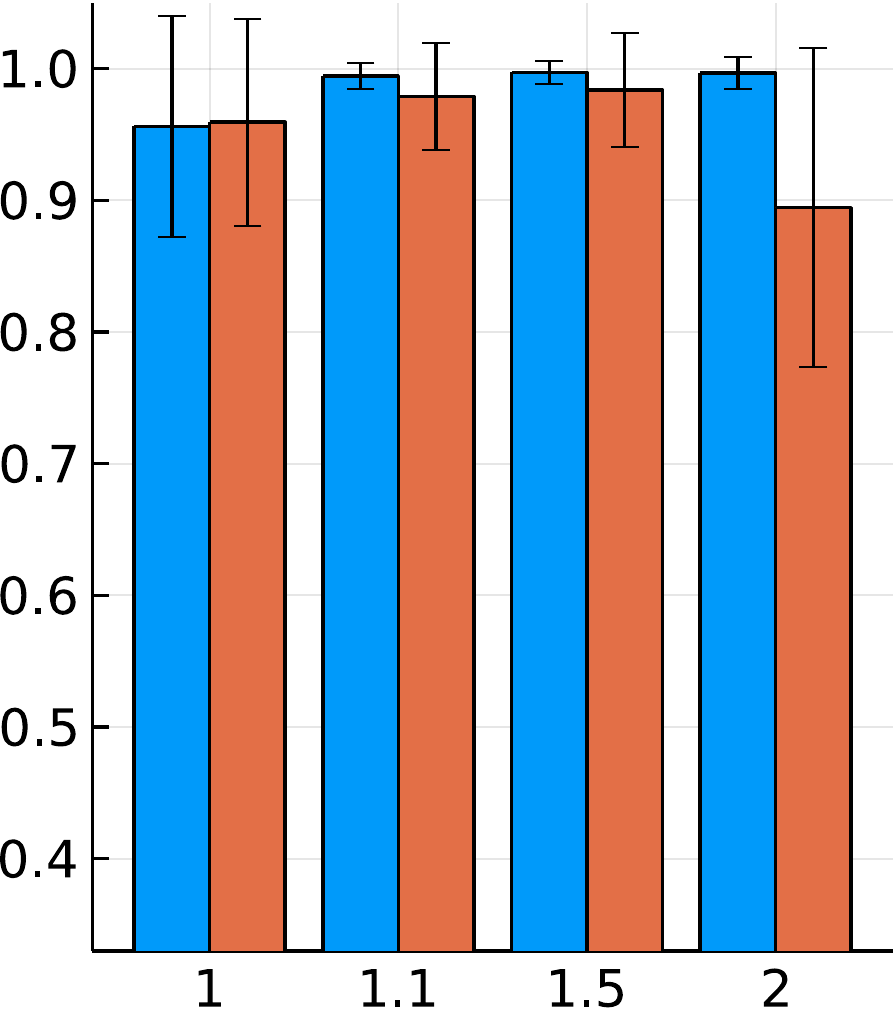}};
    \node at (2,0) {\includegraphics[height=\summaryheight]{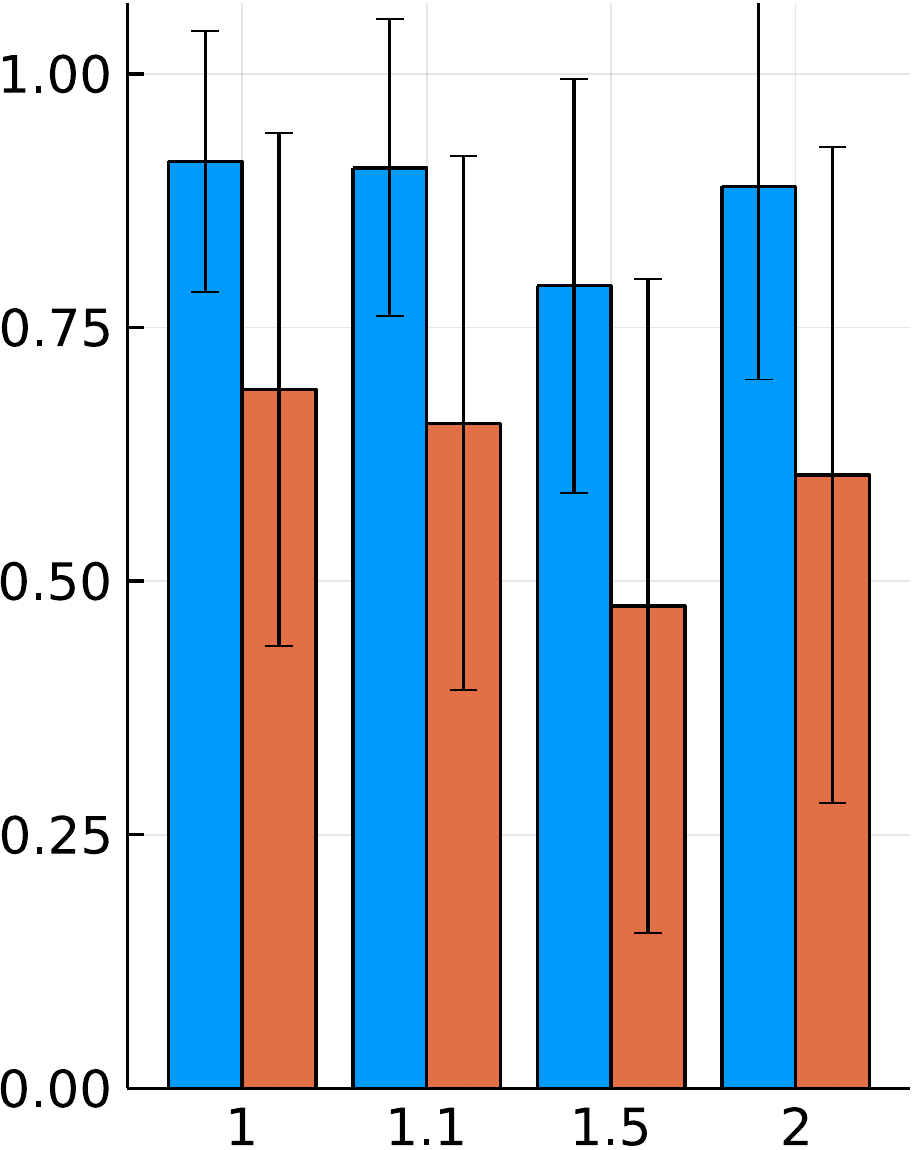}};
    \node at (4,0) {\includegraphics[height=\summaryheight]{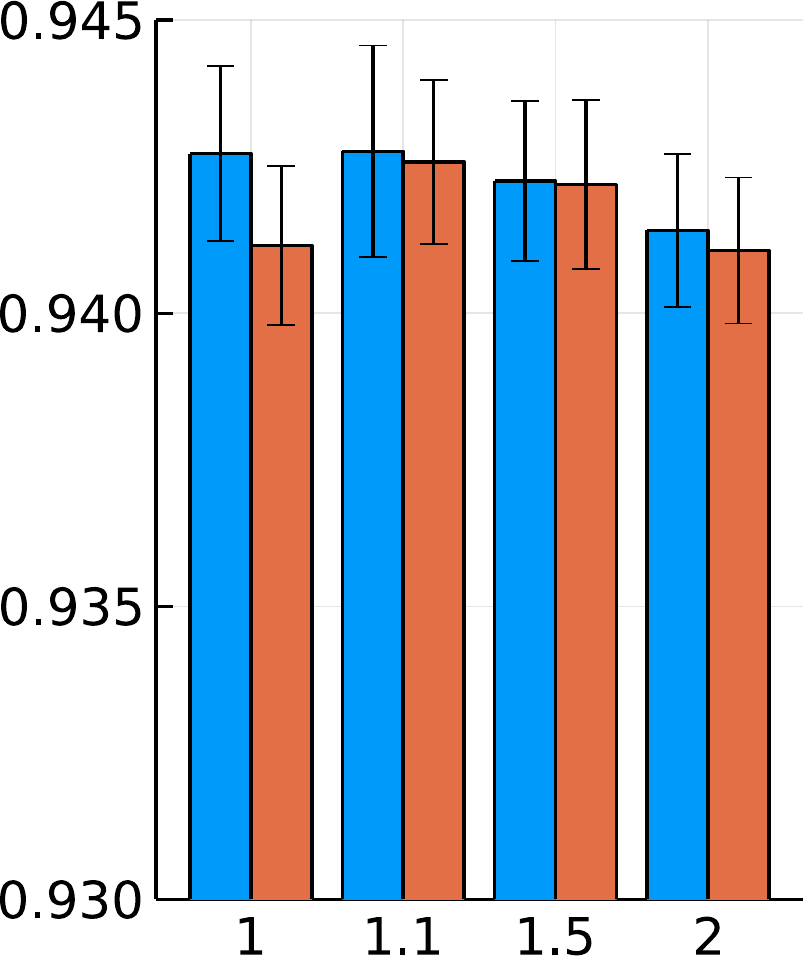}};
    \node at (6,0) {\includegraphics[height=\summaryheight]{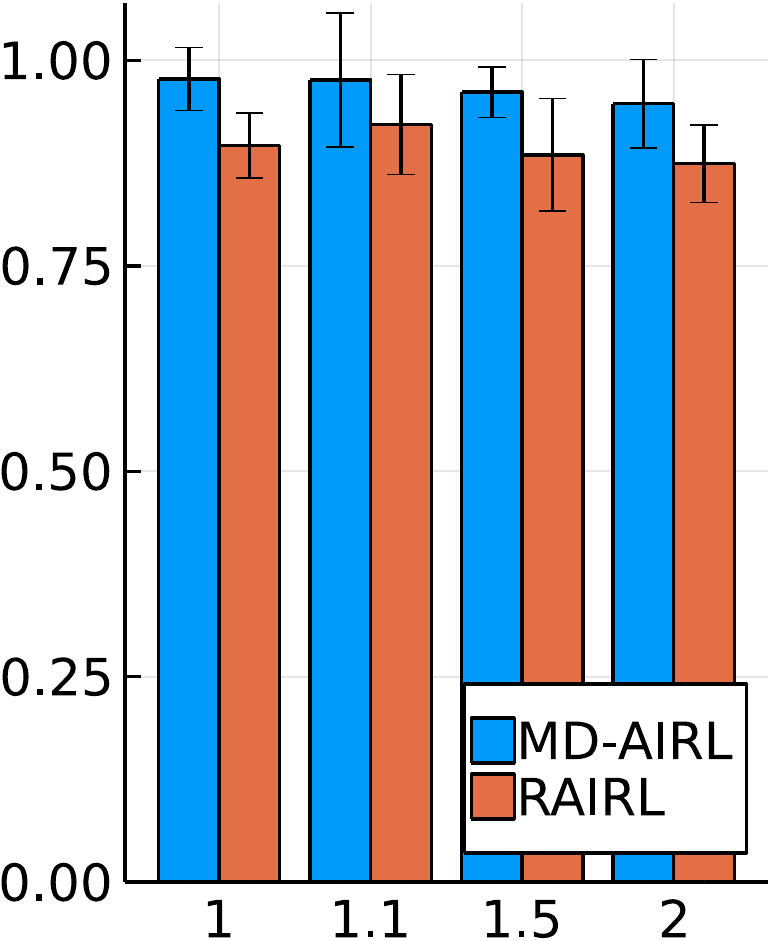}};
    \node[scale=\titlescale] at (0.1, 1.24) {\textsf{Hopper}};
    \node[scale=\titlescale] at (2.08, 1.24) {\textsf{Walker2d}};
    \node[scale=\titlescale] at (4.2, 1.24) {\textsf{HalfCheetah}};
    \node[scale=\titlescale] at (6.1, 1.24) {\textsf{Ant}};
  \end{scope}
  \end{tikzpicture}
  \end{adjustbox}
  \vskip-4pt
  \caption{Scores on the last $10^5$ steps in a total of $32$ different settings. The $x$-axis indicates the $q$ value of the \Ts{} regularizers. The scores are rescaled by considering the expert performance as $1$, and the error bars represent standard deviations. Left: 4 demonstrations. Right: 100 demonstrations.}\label{fig:summary}
  \vskip-14pt
\end{figure}
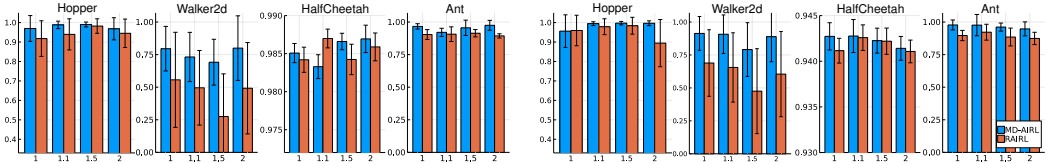

\subsection{A continuous control benchmark: \MJ{}}\label{subsect:mujoco}
Lastly, we validated MD-AIRL on the \MJ{} continuous control benchmark suite. We assumed full covariance \Gs{} policies for both learner's policy $\pi$ and expert policy $\piE$. We used the hyperbolized environment assumption \citep{rairl} where the action constraint is incorporated into the dynamics as a part of the environment using hyperbolic tangent activation.

\textbf{Sample efficiency.\topicquad}
For each task, we considered two different numbers of episodes collected by an expert policy. In \fig{fig:curve}, the performance of MD-AIRL, RAIRL, and behavior cloning (bc) algorithms \citep{bc} is shown with the expert and random agent performance. MD-AIRL was able to achieve consistent performance throughout the tasks and demonstration size. On the training curves, MD-AIRL showed high tolerance to the scarcity of data compared to RAIRL for 4 expert demonstrations. The plots in \fig{fig:summary} indicate that MD-AIRL showed higher average scores with lower variance compared to RAIRL, across 30 distinct cases among 32 configurations we have tested. MD-AIRL inherits the scalability of AIL, and it is highly stable with respect to limited sample sizes.

\begin{table}[h]
\vspace*{-14pt}
\caption{Scores on noisy demonstrations. The values of $\varepsilon$ represents scales of the Gaussian noises.}\label{tab:score}
\centering
\begin{adjustbox}{width=0.995\textwidth, center}
\small
\begin{tabular}{c|ccc}
\toprule
\multicolumn{2}{c}{\normalsize\textbf{Method}}&\normalsize$\varepsilon=0.01$&\normalsize$\varepsilon=0.5$\\
\midrule
\multirow{4}{*}{\rotatebox[origin=c]{90}{{\textbf{Hopper}}}}
  &RAIRL {\scriptsize(Shannon)}&$3636.03\pm391.09$&$3573.74\pm508.14$\\
  &MD-AIRL {\scriptsize(Shannon)}&$\mathbf{3669.25\pm177.78}$&$\mathbf{3653.31\pm267.87}$\\
\cmidrule{2-4}
  &RAIRL {\scriptsize(Tsallis)}&$3671.12\pm322.32$&$3576.17\pm515.75$\\
  &MD-AIRL {\scriptsize(Tsallis)}&$\mathbf{3730.14\pm63.09}$&$\mathbf{3701.24\pm205.68}$\\
\midrule
\multirow{4}{*}{\rotatebox[origin=c]{90}{\textbf{Walker2d}}}
  &RAIRL {\scriptsize(Shannon)}&$2856.56\pm939.9$&$2451.00\pm1392.6$\\
  &MD-AIRL {\scriptsize(Shannon)}&$\mathbf{3386.38\pm953.59}$&$\mathbf{3252.65\pm1395.7}$\\
\cmidrule{2-4}
  &RAIRL {\scriptsize(Tsallis)}&$2731.84\pm1058.7$&$2435.10\pm1555.2$\\
  &MD-AIRL {\scriptsize(Tsallis)}&$\mathbf{3624.00\pm992.63}$&$\mathbf{3093.54\pm963.96}$\\
\bottomrule
\end{tabular}
\hspace*{4pt}
\begin{tabular}{c|ccc}
\toprule
\multicolumn{2}{c}{\normalsize\textbf{Method}}&\normalsize$\varepsilon=0.01$&\normalsize$\varepsilon=0.5$\\
\midrule
\multirow{4}{*}{\rotatebox[origin=c]{90}{{\scriptsize\textbf{HalfCheetah}}}}
  &RAIRL {\scriptsize(Shannon)}&$4354.15\pm63.83$&$4216.99\pm661.17$\\
  &MD-AIRL {\scriptsize(Shannon)}&$\mathbf{4373.17\pm68.12}$&$\mathbf{4337.18\pm106.40}$\\
\cmidrule{2-4}
  &RAIRL {\scriptsize(Tsallis)}&$4364.13\pm68.09$&$4216.67\pm248.08$\\
  &MD-AIRL {\scriptsize(Tsallis)}&$\mathbf{4388.87\pm73.19}$&$\mathbf{4247.44\pm266.73}$\\
\midrule
\multirow{4}{*}{\rotatebox[origin=c]{90}{\textbf{A\hspace*{.5pt}n\hspace*{.5pt}t}}}
  &RAIRL {\scriptsize(Shannon)}&$4493.74\pm383.04$&$3777.78\pm505.78$\\
  &MD-AIRL {\scriptsize(Shannon)}&$\mathbf{4658.29\pm201.37}$&$\mathbf{4284.38\pm329.79}$\\
\cmidrule{2-4}
  &RAIRL {\scriptsize(Tsallis)}&$4359.62\pm168.46$&$3660.22\pm508.54$\\
  &MD-AIRL {\scriptsize(Tsallis)}&$\mathbf{4705.25\pm130.53}$&$\mathbf{4127.37\pm457.25}$\\
\bottomrule
\end{tabular}
\end{adjustbox}
\vskip-3pt
\end{table}

\textbf{Noisy demonstrations.\topicquad} \tab{tab:score} shows the results of imitation learning experiments for 100 expert demonstrations with two levels of \Gs{} additive noises, resulting in suboptimal demonstrations. MD-AIRL is highly tolerant to noisy data, consistently achieving higher performance. The experiment is closely related to the general case in the theory; the results suggest that the characteristics of MD-AIRL are in alignment with our analyses of the MD reward learning scheme.

\begin{wrapfigure}{r}{0.405 \textwidth}
  \newcommand*{\titlescale}{0.6}
  \newcommand*{\titleypos}{0.86}
  \newcommand*{\bgscale}{0.6}
  \newcommand*{\envnamescale}{0.44}
  \newcommand*{\envnamexshift}{-1.225cm}
  \newcommand*{\envnameyshift}{-0.44cm}
  \newcommand*{\subfigxshift}{2.47cm}
  \newcommand*{\bottomyshift}{-1.66cm}

  \newcommand{\xlabel}{
    \node[scale=0.5] at (0,-0.89) {\textsf{\Bg{} div.}}
  }
  \newcommand{\ylabel}{
    \begin{scope}[xshift=\envnamexshift, yshift=\envnameyshift]
      \node[scale=\envnamescale] at (-0.129, 1.02) {\textsf{Hopper}};
      \node[scale=0.37] at (-0.142, 0.68) {\textsf{Walker2d}};
      \node[scale=0.3] at (-0.17, 0.34) {\textsf{HalfCheetah}};
      \node[scale=\envnamescale] at (0, 0.0) {\textsf{Ant}};
    \end{scope}}

  \vskip-18pt
  \centering
  \begin{adjustbox}{width=.405\textwidth,center}

  \begin{tikzpicture}[tight background, every node/.style={inner sep=0,outer sep=0},on grid]
    \clip (-1.68, -2.43) rectangle+(5.30, 3.4);
    \begin{scope}
      \node[scale=\titlescale] at (-0.06, \titleypos){{\textsf{\Sh{} regularizer}}};
      \begin{scope}
        \node at (0, 0) {\includegraphics[height=43pt]{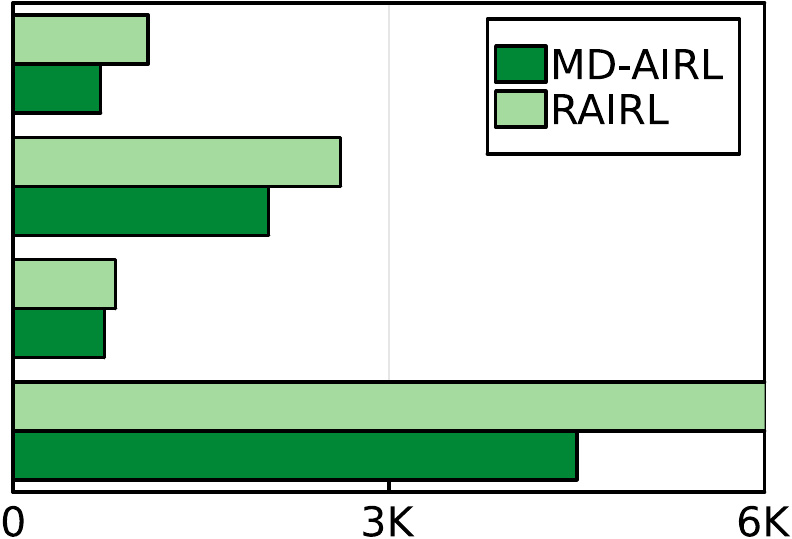}};
        \ylabel{};
      \end{scope}
      \begin{scope}[yshift=\bottomyshift]
        \node at (0.03, 0) {\includegraphics[height=43pt]{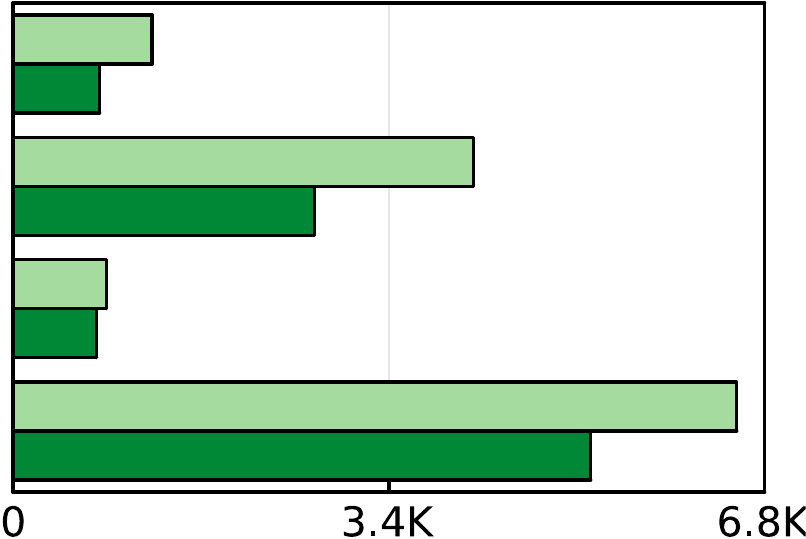}};
        \ylabel{};
      \end{scope}
    \end{scope}
    \begin{scope}[xshift=\subfigxshift]
      \node[scale=\titlescale] at (-0.06, \titleypos){{\textsf{\Ts{} regularizer}}};
      \begin{scope}
        \node at (0, 0) {\includegraphics[height=43pt]{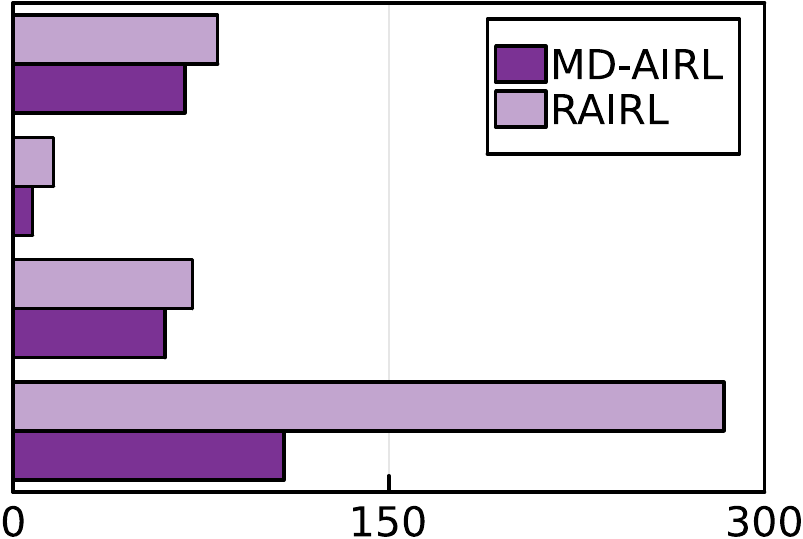}};
      \end{scope}
      \begin{scope}[yshift=\bottomyshift]
        \node at (0, 0) {\includegraphics[height=43pt]{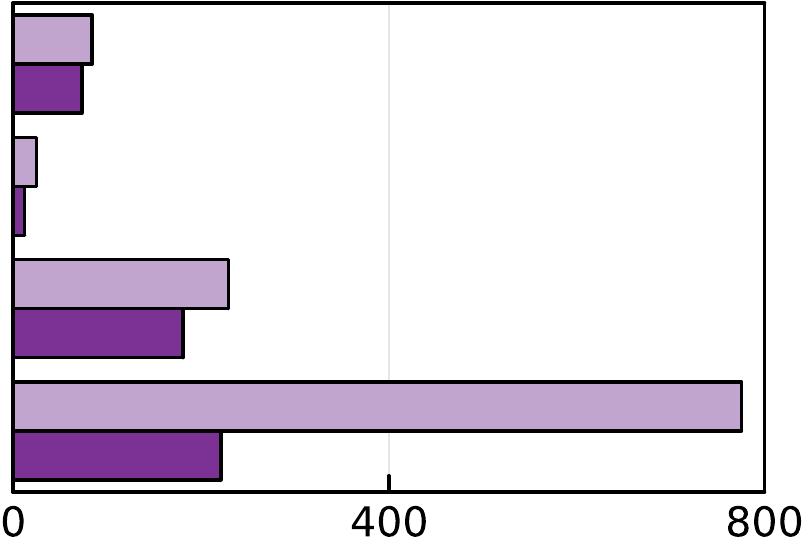}};
      \end{scope}
    \end{scope}
  \end{tikzpicture}
  \end{adjustbox}
  \vskip-6pt
  \caption{Divergences after imitation learning. \small \hskip4pt Top: $\varepsilon\!=\!0.01$. \hskip4pt Bottom: $\varepsilon\!=\!0.5$.}\label{fig:extra}
  \vskip-20pt
\end{wrapfigure}
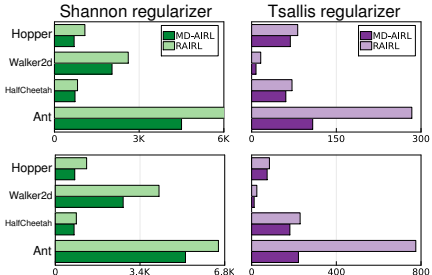
We present a detailed analysis of the noisy demonstration experiments (\fig{fig:extra}). Let the \Bg{} divergence between agent and ground-truth expert policies be the error, and we measured these errors by increasing the given noise level for the expert trajectories. \fig{fig:extra} shows a general tendency that MD-AIRL has lower errors than RAIRL. With  \tab{tab:score} and \fig{fig:extra}, we were able to find the evident correlation between average \Bg{} divergence and performance since imitation learning convergence when the divergence is zero. Thus, this is another piece of empirical evidence that verifies our theoretical claims.

\section{Conclusions and Discussion}\label{sect:concl}
In this paper, we presented MD-AIRL, a practical AIL algorithm designed to solve the imitation learning problem in the real world. We proved that the proposed method has clear advantages over previous AIL methods in terms of robustness. We verified MD-AIRL in a variety of situations, including high-dimensional spaces, limited samples, and imperfect demonstrations. The empirical evidence showed that MD-AIRL outperforms previous methods on various benchmarks. We conclude that the rich foundation of optimization theories shows a promising direction for AIL studies.

Considering RL and IRL with geometric perspectives is vital for solving real-world problems. Although our work covers various imitation learning problems with the \Bg{} divergence, this does not include some other problems when the proximity term is of other statistical divergence families, such as the f-divergence \citep{firl}. If the relationship between these classes of divergences is studied in more detail, it is expected to proceed with applications to various subfields of machine learning. The assumptions on $\Omega$ in our analyses are usually justified by enforcing a specific policy space, but some outliers might have substantial meaning for certain tasks. Therefore, extensive analyses on these parameterizations remain as future works. The ``impurity'' of the MD-AIRL reward function compared to $\PsiOmega(\Pi)$ can be regarded as a limitation. To fully resolve this problem, all data must be treated as on-policy samples, which might require a sophisticated sampling mechanism.

\textbf{Societal impacts.}\topicquad{}
The evolution of imitation learning algorithms is expected to bring a structural shift in the labor market. The negative impact could be mitigated by diversification, unification, and redefinition of routine and manual jobs. The results of our work can be abused as a tool for analyzing individual data. Therefore, we stress that certain acts should be carefully regulated, such as collecting a substantial amount of individuals' data and aggressively tracking personal identity.

\section*{Acknowledgments}
The authors would like to thank the anonymous reviewers, Woosuk Choi, Jaein Kim, and Min Whoo Lee for their helpful discussion and comments. This work was partly supported by the IITP (2022-0-00951-LBA/25\%, 2022-0-00953-PICA/25\%, 2015-0-00310-SW.StarLab/10\%, 2021-0-02068-AIHub/10\%, 2021-0-01343-GSAI/10\%, 2019-0-01371-BabyMind/10\%) grant funded by the Korean government, and the CARAI (UD190031RD/10\%) grant funded by the DAPA and ADD.

\medskip
{
\small
\bibliographystyle{unsrtnat}
\bibliography{neurips_2022_mdirl.bib}
}

\ifdefined\checklist
\section*{Checklist}

\begin{enumerate}

\item For all authors...
\begin{enumerate}
  \item Do the main claims made in the abstract and introduction accurately reflect the paper's contributions and scope?
    \answerYes{}
  \item Did you describe the limitations of your work?
    \answerYes{} See \sect{sect:concl}.
  \item Did you discuss any potential negative societal impacts of your work?
    \answerYes{} See \sect{sect:concl}.
  \item Have you read the ethics review guidelines and ensured that your paper conforms to them?
    \answerYes{}
\end{enumerate}

\item If you are including theoretical results...
\begin{enumerate}
  \item Did you state the full set of assumptions of all theoretical results?
    \answerYes{} See \sect{sect:theory}.
  \item Did you include complete proofs of all theoretical results?
     \answerYes{} See \appsect{appsect:proof}.
\end{enumerate}

\item If you ran experiments...
\begin{enumerate}
  \item Did you include the code, data, and instructions needed to reproduce the main experimental results (either in the supplemental material or as a URL)?
    \answerYes{} Our empirical studies can be reproduced by from the detailed information in Appendices~\ref{appsect:full}~and~\ref{appsect:implement}.
  \item Did you specify all the training details (e.g., data splits, hyperparameters, how they were chosen)?
     \answerYes{} See Appendix~\ref{appsect:implement}.
        \item Did you report error bars (e.g., with respect to the random seed after running experiments multiple times)?
    \answerYes{} See \sect{sect:expr}.
        \item Did you include the total amount of compute and the type of resources used (e.g., type of GPUs, internal cluster, or cloud provider)?
    \answerNA{} In experiments, each algorithm was executed in CPU (a single thread).
\end{enumerate}

\item If you are using existing assets (e.g., code, data, models) or curating/releasing new assets...
\begin{enumerate}
  \item If your work uses existing assets, did you cite the creators?
    \answerYes{}
  \item Did you mention the license of the assets?
    \answerNA{} The MuJoCo simulator used in our experiments is freely available to everyone. See the site (https://mujoco.org).
  \item Did you include any new assets either in the supplemental material or as a URL?
    \answerNo{}
  \item Did you discuss whether and how consent was obtained from people whose data you're using/curating?
    \answerNA{}
  \item Did you discuss whether the data you are using/curating contains personally identifiable information or offensive content?
    \answerNA{}
\end{enumerate}

\item If you used crowdsourcing or conducted research with human subjects...
\begin{enumerate}
  \item Did you include the full text of instructions given to participants and screenshots, if applicable?
    \answerNA{}
  \item Did you describe any potential participant risks, with links to Institutional Review Board (IRB) approvals, if applicable?
    \answerNA{}
  \item Did you include the estimated hourly wage paid to participants and the total amount spent on participant compensation?
    \answerNA{}
\end{enumerate}

\end{enumerate}
\fi

\clearpage
\vbox{
    \hsize\textwidth
    \linewidth\hsize
    \vskip 0.08in
  \hrule height 3pt
  \vskip 0.218in
  \vskip -\parskip
    \centering
    {\LARGE\bf \normalsize Appendices for\vskip-2pt{\large Robust Imitation via Mirror Descent Inverse Reinforcement Learning}\par}
  \vskip 0.26in
  \vskip -\parskip
  \hrule height 1pt
  \vskip 0.09in
}

\appendix
\section{Proofs}\label{appsect:proof}
We denote the entire set of conditional distributions as $\Delta^\cS_\cA$, which is a vector space formed by a collection of $\lvert\cS\rvert$ elements of unit $(\lvert\cA\rvert-1)$-simplexes: $\Delta_\cA=\bigl\{\,x_1e_1+\dots+x_{\lvert\cA\rvert}e_{\lvert\cA\rvert}\,\big\vert\,\sum_{i=1}^{\lvert\cA\rvert} x_i=1\hspace*{5pt}\textrm{and}\hspace*{5pt}x_i \ge0\hspace*{5pt}\textrm{for}\hspace*{5pt}i\in\cA\,\bigr\}$. A trainable policy space is a subset of the entire conditional probability denoted as $\Pi\! \coloneqq\! [\Pis\hspace*{-.5pt}]_{\hspace*{-.8pt}\sincS}\! \subset\! \Delta^\cS_\cA$. We assume that $\Pis$ is a member of a specific \Bn{} space called \Lp{} space $(\bR^{\cA}, \lVert\cdot\rVert)$, where $\lVert\cdot\rVert$ is a $\rp$-norm on $\cA$. The dual space of \Lp{} space for $1 < \rp < \infty$ is \Lq{} space $(\bR^{\cA}, \lVert\cdot\rVertast\!)$, where $\lVert\cdot\rVert_{\!\ast}$ is defined as a $\rq$-norm ($\nicefrac{1}{\rp}+\nicefrac{1}{\rq}=1$). Here, the condition $1 < \rp \le 2$ is assumed for existence and convergence properties in the dual \Lq{} space.

We begin with the following preliminary definitions: the \Lip{} continuity and martingales.
\begin{definition}[\Lip{} constants]
  Given two metric spaces $(X,d_X\!)$ and $(Y,d_Y\!)$ where $d_X$ denotes the metric on set $X$ and $d_Y$ is the metric on set $Y$, a function $f:X\to Y$ is called \Lip{} continuous if there exists a real constant $k \ge 0$ such that, for all $x_1$ and $x_2$ in $X$,
\begin{equation}
  d_Y\hspace*{-1pt}\bigl(f(x_1),f(x_2)\bigr) \le k\cdot d_X\hspace*{-.5pt}(x_1,x_2).
\end{equation}
In particular, a function $f$ is called \Lip{} continuous if there exists a constant $k\ge 0$ such that,
\begin{equation}
  \bigl\lVert f(x_1)-f(x_2) \big\rVert_{\hspace*{-.5pt}\ast} \le k  \lVert x_1-x_2\rVert,\quad\forall x_1,x_2
\end{equation}
where norms $\lVert\cdot\rVert$ and $\lVert\cdot\rVertast$ are endowed with spaces $X$ and $Y$ respectively. For the smallest $L$ that substitutes $k$, $L$ is called the \Lip{} constant and $f$ is called a $L$-\Lip{} continuous function.
\end{definition}
\begin{definition}[Discrete-time martingales]
  If a stochastic process  $\{Z_t\}_{t\ge1}$  satisfies $\bE[\lvert Z_n\rvert] < \infty$ and
\begin{equation*}
\textstyle
\textrm{\enumone{} }\,\bE[Z_{n+1}\vert X_1,\dots,X_n]\le Z_n,\
\textrm{\enumtwo{} }\,\bE[Z_{n+1}\vert X_1,\dots,X_n]=Z_n,\
\textrm{\enumthree{} }\,\bE[Z_{n+1}\vert X_1,\dots,X_n]\ge Z_n,
\end{equation*}
the stochastic process $\{Z_t\}_{t\ge1}$ is called \enumone{} a submartingale, \enumtwo{} a martingale, and \enumthree{} a supermartingale, with respect to filtration $ \{X_t\}_{t\ge1}$.
\end{definition}

The following arguments and proofs follow the results that appeared in previous literature for general aspects  \citep{md,omd_univ,md_nonlin,md_info,omd_converge,mirrorless}. Our analyses extend existing theoretical results to imitation learning and IRL; they are also highly general to cover various online methods for sequential decision problems.

\subsection{Proof of Lemma 1}\label{subsect:proof_psi_unique}
\begin{proof}[Proof of Lemma~1]
The conjugate operator of $\psi_\pi^s$ satisfies the following identity (Lemma~1 of \citep{rairl})
\begin{equation*}
\begin{aligned}
  \Omega^\ast(\psi_\pi^s) &=\max_\tildepisinDeltacA\langle\tilde{\pi}^s,\psi^s_\pi\rangle_{\!\cA}-\Omega(\tilde{\pi}^s)\\
  &=\max_\tildepisinDeltacA\langle\tilde{\pi}^s, \nabla\Omega(\pi^s)\rangle_\cA -\bigl\langle \pi^s, \nabla\Omega(\pi^s)\big\rangle_{\!\cA}+ \Omega(\pi^s)-\Omega(\tilde{\pi}^s)\\
  &=\min_\tildepisinDeltacA \Omega(\tilde{\pi}^s)-\Omega(\pi^s)-\langle\nabla\Omega(\pi^s),\tilde{\pi}^s-\pi^s\rangle_{\!\cA}\\
  &=\min_\tildepisinDeltacA D_\Omega\hspace*{-1pt}\bigl(\hspace*{1pt}\tilde{\pi}^s\big\Vert\hspace*{1pt}\pi^s\bigr),
\end{aligned}
\end{equation*}
for every state $s\in\cS$. By the property of \Bg{} divergence and the convexity of $D_\Omega(\tilde{\pi}^s\Vert\pi^s)$ with respect to $\tilde{\pi}^s$, the optimal condition is obtained by the unique maximizing argument $\tilde{\pi}(\cdot|s)=\pi(\cdot|s)$. By taking gradient to both sides with respect to $\psi_\pi^s$ we yield $\pi^s=\nabla\Omega^\ast \hspace*{-1pt}(\psi^s_\pi\hspace*{-1pt})$.

If there is another $\tilde{\pi}\in\Delta^\cS_\cA$ that makes $\psi_{\tilde{\pi}}=\psi_\pi$, this contradicts the property of unique maximizing arguments for conjugates since $\pi\in\Pi$ and $\Pi\subset\Delta^\cS_\cA$. Therefore, $\psi_\pi$ is uniquely defined for each $\pi$ and $\nabla\Omega^\ast(\psi^s)\,\textrm{\raisebox{0.5pt}{$\in$}}\,\Pi^s$ for all $s\in\cS$.
\end{proof}

\subsection{Proof of Theorem 1}\label{subsect:proofhm1}
Consider the unique fixed point of $\pi_\ast$ as the solution of $\inf_{\!\piinPi}\bE[ f(\pi,\tau_t)]$ where the expectation indicates that we consider all outputs with respect to $\tau_t$ for $t\to\infty$ i.e., $\lim_{t\to\infty}\!\bE_{\tauonet}\hspace*{-1pt}[f(\pi, \tau_t)]$. By equating derivatives to zero, we write the condition of fixed point $\pi_\ast$ as $\nabla \Omega(\pi_\ast)= \lim_{t\to\infty}\bE[\nabla \Omega(\barpiEt)]$. This assumption is useful since this paper provides some general results, the case of $\inf_{\!\piinPi}\bE[f(\pi,\tau_t)]>0$ in particular, which means that the estimates $\{\barpiEt\}_{t=1}^\infty$ do not converge to the fixed point of $\pi_\ast$, hence { $\lim_{t\to\infty}\bE_{\tauonet}[\lVert\pi_\ast-\barpiEt\rVert] \ne 0$}. As a result, MD-based imitation learning algorithms allow many challenging settings, such as scarcity of data or imperfect demonstrations.

We first introduce a fundamental relationship regarding cumulative gradients in our online MD setting.
\begin{lemma}\label{lem:gradient_equiv} Let $\{\pi_t\hspace*{-.5pt}\}_{t=1}^\infty$, $\{\barpiEt\hspace*{-1pt}\}_{t=1}^\infty$, and $\{\eta_t\hspace*{-.5pt}\}_{t=1}^\infty$ be policy, estimate, and step size sequences, respectively. The subsequent policy $\pi_{t+1}$ in \eq{eq:mdobj} is obtained by an RL algorithm  using the derivation of $\psi_{t+1}$ in \eq{eq:mdobj},  resulting to the following equation:
\begin{equation}
  \pi_{t+1}(\hspace*{1pt}\cdot\hspace*{1pt}|s)=\argmin_{\pi^s\in\Pis} \eta_t D_\Omega\hspace*{-1pt}\bigl(\pi^s\big\Vert\barpiEt^s\bigr) +(1\!-\!\eta_t) D_\Omega\hspace*{-1pt}\bigl(\pi^s\big\Vert\pi^s_t\bigr)\quad \forall s\in\cS.
\end{equation} We have for $t\in\bN$,
\begin{equation}\label{eq:gradient_equiv}
  \eta_t\Bigl(\hspace*{-1pt}\nabla\Omega\bigl(\pi^s_t\bigr)-\nabla\Omega\bigl(\barpiEst\bigr)\!\Bigr) =
  \nabla\Omega\bigl(\pi^s_t\bigr)-\nabla\Omega\bigl(\pi^s_{t+1}\bigr)\quad \forall s\in\cS.
\end{equation}
\end{lemma}
\begin{proof}[Proof of Lemma~\ref{lem:gradient_equiv}]
Since the optimization problem is convex with respect to each $\pi^s$, we equate the derivatives at $\pi_{t+1}$ to zero:
\begin{equation*}
  \eta_t\Bigl(\hspace*{-1pt}\nabla\Omega\bigl(\pi^s_{t+1}\bigr)-\nabla\Omega\bigl(\barpiEst\bigr)\!\Bigr)+(1\!-\!\eta_t)\Bigl(\nabla\Omega(\pi^s_{t+1})-\nabla\Omega\bigl(\pi^s_{t}\bigr)\!\Bigr)=0,\quad\forall\ s\in\cS.
\end{equation*}
Then, we derive \eq{eq:gradient_equiv} as
\begin{equation*}
  \begin{aligned}
  &\eta_t\Bigl(\hspace*{-1pt}\nabla\Omega\bigl(\pi^s_{t+1}\bigr)-\nabla\Omega\bigl(\barpiEst\bigr)\!\Bigr)+(1\!-\!\eta_t)\Bigl(\hspace*{-1pt}\nabla\Omega\bigl(\pi^s_{t+1}\bigr)-\nabla\Omega\bigl(\pi^s_{t}\bigr)\!\Bigr)=0\\
  &\Leftrightarrow\quad\nabla\Omega\bigl(\pi^s_{t+1}\bigr)-\eta_t \nabla\Omega\bigl(\barpiEst\bigr)-(1\!-\!\eta_t)\nabla\Omega\bigl(\pi^s_t\bigr)=0\\
  &\Leftrightarrow\quad\nabla\Omega\bigl(\pi^s_t\bigr)-\nabla\Omega\bigl(\pi^s_{t+1}\bigr)=\eta_t\Bigl(\hspace*{-1pt}\nabla\Omega\bigl(\pi^s_t\bigr)-\nabla\Omega\bigl(\barpiEst\bigr)\!\Bigr) \quad\forall\ s\in\cS.
  \end{aligned}
\end{equation*}
Therefore, the proof is complete.
\end{proof}
\lem{lem:gradient_equiv} indicates that the distances between dual maps are equivalent to $\eta_t\bigl\lVert\nabla\Omega(\barpiEst)-\nabla\Omega(\pi^s_t)\bigr\rVert_{\!\ast}$. Therefore, when the step size converges as $\lim_{t\to\infty}\eta_t=0$, the convergence in the dual space is induced as $\lim_{t\to\infty}\bigl\lVert\nabla\Omega(\pi^s_t)-\nabla\Omega(\pi^s_{t+1})\bigr\rVert_{\!\ast}=0$; thus, the convergence of associated reward functions for every state in \sect{sect:reward} is reasonable when $\Omega$ is strongly smooth.

In the following lemmas (Lemmas \ref{lem:three_identity}-\ref{lem:two_identity2}), we omit the given state for simplicity since they hold for $\forall s\in\cS$, hence one can write distributions $\pi_a = \pi^s_a$, $\pi_b = \pi^s_b$, and $\pi_c = \pi^s_c$ for a arbitrary given state $s$. First, we reintroduce the three-point identity as follows.
\begin{lemma}[Three-point identity]\label{lem:three_identity}
Let $\pi_a$, $\pi_b$, and $\pi_c$ be any policy distributions with a given state. We have the following identity:
\begin{equation*}
  \bigl\langle\nabla\Omega(\pi_a)\!-\!\nabla\Omega(\pi_b),\,\pi_c\!-\!\pi_b\bigr\rangle_{\!\!\scriptscriptstyle\cA}=D_\Omega\hspace*{-1pt}\bigl(\hspace*{1pt}\pi_c\hspace*{.5pt}\big\Vert\hspace*{1pt}\pi_b\hspace*{.5pt}\bigr)-D_\Omega\hspace*{-1pt}\bigl(\hspace*{1pt}\pi_c\hspace*{.5pt}\big\Vert\hspace*{1pt}\pi_a\hspace*{.5pt}\bigr)+D_\Omega\hspace*{-1pt}\bigl(\hspace*{1pt}\pi_b\hspace*{.5pt}\big\Vert\hspace*{1pt}\pi_a\bigr)
\end{equation*}
\end{lemma}
\begin{proof}[Proof of Lemma~\ref{lem:three_identity}]
This can be derived using the definition of divergence as follows.
\begin{equation*}
\begin{aligned}
  D_\Omega\hspace*{-.5pt}(\pi_c\Vert\pi_b)-D_\Omega\hspace*{-.5pt}(\pi_c\Vert\pi_a)+D_\Omega\hspace*{-.5pt}(\pi_b\Vert\pi_a\hspace*{-.5pt})
  &=\Omega(\pi_c)-\Omega(\pi_b)-\bigl\langle \nabla\Omega(\pi_b),\,\pi_c\!-\pi_b\bigr\rangle_{\!\!\scriptscriptstyle\cA}\\
  &\quad-\Omega(\pi_c)+\Omega(\pi_a)+\bigl\langle \nabla\Omega(\pi_a),\,\pi_c\!-\pi_a\bigr\rangle_{\!\!\scriptscriptstyle\cA}\\
  &\quad+\Omega(\pi_b)-\Omega(\pi_a)-\bigl\langle \nabla\Omega(\pi_a),\,\pi_b\!-\pi_a\bigr\rangle_{\!\!\scriptscriptstyle\cA}\\
  &=\bigl\langle\nabla\Omega(\pi_a)\!-\!\nabla\Omega(\pi_b),\,\pi_c\!-\pi_b\bigr\rangle_{\!\!\scriptscriptstyle\cA}.
\end{aligned}
\end{equation*}
Therefore, the proof is complete.
\end{proof}

Then, we introduce two identities in Lemmas~\ref{lem:two_identity1}~and~\ref{lem:two_identity2} that are later used to address the progress of mirror descent updates in terms of \Bg{} divergences.
\begin{lemma}\label{lem:two_identity1}
Let $\pi_a$, $\pi_b$, and $\pi_c$ be any policy distributions with a given state. The following identity holds.
\begin{equation}\label{eq:identity1}
  D_\Omega\hspace*{-1pt}\bigl(\hspace*{1pt}\pi_c\hspace*{.5pt}\big\Vert\hspace*{1pt}\pi_b\hspace*{.5pt}\bigr)-D_\Omega\hspace*{-1pt}\bigl(\hspace*{1pt}\pi_c\hspace*{.5pt}\big\Vert\hspace*{.5pt} \pi_a\hspace*{.5pt}\bigr)=D_\Omega\hspace*{-1pt}\bigl(\hspace*{1pt}\pi_a\hspace*{.5pt}\big\Vert\hspace*{1pt}\pi_b\hspace*{.5pt}\bigr)+\bigl\langle \nabla\Omega(\pi_a)\!-\!\nabla\Omega(\pi_b),\,\pi_c\!-\!\pi_a\bigr\rangle_{\!\!\scriptscriptstyle\cA}
\end{equation}
\end{lemma}
\begin{proof}[Proof of Lemma~\ref{lem:two_identity1}]
By Lemma~\ref{lem:three_identity}, we have
\begin{equation*}
  D_\Omega\hspace*{-1pt}\bigl(\hspace*{1pt}\pi_c\hspace*{.5pt}\big\Vert\hspace*{1pt}\pi_b\hspace*{.5pt}\bigr)-D_\Omega\hspace*{-1pt}\bigl(\hspace*{1pt}\pi_c\hspace*{.5pt}\big\Vert\hspace*{1pt}\pi_a\hspace*{.5pt}\bigr)=-D_\Omega\hspace*{-1pt}\bigl(\hspace*{1pt}\pi_b\hspace*{.5pt}\big\Vert\hspace*{1pt}\pi_a\hspace*{.5pt}\bigr)+\bigl\langle \nabla \Omega\bigl(\pi_a)\!-\!\nabla \Omega(\pi_b),\,\pi_c\!-\!\pi_b\bigr\rangle_{\!\!\scriptscriptstyle\cA}.
\end{equation*}
Utilizing an identity of two \Bg{} divergences for arbitrary $(\pi,\tilde{\pi})$:
\begin{equation}\label{eq:breg_two}
  D_\Omega\hspace*{-.5pt}(\pi\hspace*{.5pt}\Vert\tilde{\pi})+D_\Omega\hspace*{-.5pt}(\tilde{\pi}\Vert\pi)=\bigl\langle \nabla\Omega(\pi)\!-\!\nabla\Omega(\tilde{\pi}),\,\pi-\tilde{\pi}\bigr\rangle_{\!\!\scriptscriptstyle\cA},
\end{equation}
we separate $\pi_c\!-\pi_b$ into $\pi_c\!-\pi_a$ and $\pi_a\!-\pi_b$ and write the rest of the derivation as follows.
\begin{equation*}
\begin{aligned}
  &D_\Omega\hspace*{-1pt}\bigl(\hspace*{1pt}\pi_c\hspace*{.5pt}\big\Vert\hspace*{1pt}\pi_b\hspace*{.5pt}\bigr)-D_\Omega\hspace*{-1pt}\bigl(\hspace*{1pt}\pi_c\hspace*{.5pt}\big\Vert\hspace*{1pt}\pi_a\bigr)\\
  &\qquad=\underbrace{-D_\Omega\hspace*{-1pt}\bigl(\hspace*{1pt}\pi_b\hspace*{.5pt}\big\Vert\hspace*{1pt}\pi_a\bigr) + \bigl\langle \nabla \Omega(\pi_a)\!-\!\nabla\Omega(\pi_b),\,\pi_a\!-\pi_b\bigr\rangle_{\!\!\scriptscriptstyle\cA}}_\textrm{\eqbrief{eq:breg_two}}+\hspace*{3pt}\bigl\langle \nabla \Omega(\pi_a)\!-\!\nabla \Omega(\pi_b),\,\pi_c\!-\pi_a\bigr\rangle_{\!\!\scriptscriptstyle\cA}\\
  &\qquad=D_\Omega\hspace*{-1pt}\bigl(\hspace*{1pt}\pi_a\hspace*{.5pt}\big\Vert\hspace*{1pt}\pi_b\hspace*{.5pt}\bigr)+\bigl\langle \nabla\Omega(\pi_a)\!-\!\nabla \Omega(\pi_b),\, \pi_c\!-\!\pi_a\big\rangle_{\!\!\scriptscriptstyle\cA}
\end{aligned}
\end{equation*}
Therefore, we achieve the desired identity.
\end{proof}
\begin{lemma}\label{lem:two_identity2}
Let $\pi_a$, $\pi_b$, and $\pi_c$ be any policy distributions with a given state. The following identity holds.
\begin{equation}\label{eq:identity}
  D_\Omega\hspace*{-1pt}\bigl(\hspace*{1pt}\pi_b\hspace*{.5pt}\big\Vert\hspace*{1pt}\pi_a\bigr) - D_\Omega\hspace*{-1pt}\bigl(\hspace*{1pt}\pi_c\hspace*{.5pt}\big\Vert\hspace*{1pt}\pi_a\hspace*{.5pt}\bigr) =  - \bigl\langle \nabla \Omega(\pi_c)\!-\!\nabla \Omega(\pi_a),\,\pi_c\!-\!\pi_b\bigr\rangle_{\!\!\scriptscriptstyle\cA} + D_\Omega\hspace*{-1pt}\bigl(\hspace*{1pt}\pi_{b}\hspace*{.5pt}\big\Vert \hspace*{1pt}\pi_c\bigr)
\end{equation}
\end{lemma}
\begin{proof}[Proof of Lemma~\ref{lem:two_identity2}]
By Lemma~\ref{lem:three_identity}, we have
\begin{equation*}
  D_\Omega\hspace*{-1pt}\bigl(\hspace*{1pt}\pi_b\hspace*{.5pt}\big\Vert\hspace*{1pt}\pi_a\bigr) - D_\Omega\hspace*{-1pt}\bigl(\hspace*{1pt}\pi_c\hspace*{.5pt}\big\Vert\hspace*{1pt}\pi_a\hspace*{.5pt}\bigr) = - D_\Omega\hspace*{-1pt}\bigl(\hspace*{1pt}\pi_c\hspace*{.5pt}\big\Vert \hspace*{1pt}\pi_b\hspace*{.5pt}\bigr) + \bigl\langle \nabla\Omega(\pi_a)\!-\!\nabla \Omega(\pi_b),\,\pi_c\!-\!\pi_b\bigr\rangle_{\!\!\scriptscriptstyle\cA}.
\end{equation*}
We separate $\nabla\Omega(\pi_a)\!-\!\nabla\Omega(\pi_b)$ into $\nabla\Omega(\pi_a)\!-\!\nabla\Omega(\pi_c)$ and $\nabla\Omega(\pi_c)\!-\!\nabla \Omega(\pi_b)$ and write the rest of the derivation as follows.
\begin{equation*}
\begin{aligned}
  &D_\Omega\hspace*{-1pt}\bigl(\hspace*{1pt}\pi_b\hspace*{.5pt}\big\Vert\hspace*{1pt}\pi_a\hspace*{.5pt}\bigr) - D_\Omega\hspace*{-1pt}\bigl(\hspace*{1pt}\pi_c\hspace*{.5pt}\big\Vert\hspace*{1pt}\pi_a\hspace*{.5pt}\bigr)\\
  &\qquad =\underbrace{-D_\Omega\hspace*{-1pt}\bigl(\hspace*{1pt}\pi_c\hspace*{.5pt}\big\Vert\hspace*{1pt} \pi_b\hspace*{.5pt}\bigr)  + \bigl\langle \nabla \Omega(\pi_c)\!-\! \nabla \Omega(\pi_b),\,\pi_c\!-\!\pi_b\bigr\rangle_{\!\!\scriptscriptstyle\cA}}_\textrm{\eqbrief{eq:breg_two}} + \hspace*{3pt}\bigl\langle \nabla \Omega(\pi_a)\!-\!\nabla \Omega(\pi_c),\,\pi_c\!-\!\pi_b\bigr\rangle_{\!\!\scriptscriptstyle\cA}\\
  &\qquad = D_\Omega\hspace*{-1pt}\bigl(\hspace*{1pt}\pi_b\hspace*{.5pt}\big\Vert\hspace*{1pt}\pi_c\hspace*{.5pt}\bigr) + \bigl\langle \nabla\Omega(\pi_a)\!-\! \nabla \Omega(\pi_c),\,\pi_c\!-\!\pi_b\bigr\rangle_{\!\!\scriptscriptstyle\cA}
\end{aligned}
\end{equation*}
Therefore, we achieve the desired identity.
\end{proof}

Combining above lemmas, we show a key argument to prove \thm{thm:step_size} in the following lemma.
\begin{lemma}\label{lem:assume}
  Assume $\inf_{\!\piinPi}\bE[f(\pi,\tau_t)] > 0$. Assume that $\Omega$ is $\omega$-strongly convex and $\nabla\Omega$ is $L$-\Lip{} continuous for $\omega\ge0$ and $L\ge0$. If $\lim_{t\to\infty} \bE_{\tauonet}\bigl[\sum_{i=0}^\infty  \gamma^i D_\Omega\bigl(\pi_t(\,\cdot\,|s_i)\big\Vert\piE(\,\cdot\,|s_i)\bigr)\bigr]=0\ $ for $\piE\in\Pi$, then $\{\eta_t\}_{t=1}^\infty$ satisfies \eq{eq:step_cond}. Furthermore, if $\Omega$ is strongly smooth, then \thm{thm:step_size}~(a) holds with some constants $n \in \bN$ and $c > 0$.
\end{lemma}
\begin{proof}[Proof of Lemma~\ref{lem:assume}]
First, we show the condition of $\lim_{t\to\infty}\eta_t=0$. Assuming all states are decomposable\footnote{The decomposability condition; Definition~B.1 of \citet{airl}}, the condition $\lim_{t\to\infty}\bE_{\tauonet}\bigl[\sum_{i=0}^\infty \gamma^i D_\Omega\bigl(\pi_t(\,\cdot\,|s_i)\big\Vert\piE(\,\cdot\,|s_i)\bigr)\bigr]=0$ implies $\lim_{t\to\infty} \bE_{\tauonet}\bigl[\lVert \pi_t- \piE\rVert\bigr]=0$, where $\lVert\cdot\rVert$ is the \textit{matrix} norm induced by the $\rp$-norm on $\cA$. Then, our aim is to show that the gradient of the strong convex function for $\pi_t$ converges to $\nabla\Omega(\piE)$, i.e.
\begin{equation}\label{eq:claim}
  \lim_{t\to\infty}\bE_{\tauonet}\!\bigl[\bigl\lVert\nabla\Omega(\pi_t)-\nabla\Omega(\piE)\hspace*{-.5pt}\bigr\rVert_{\!\ast}\hspace*{-.5pt}\bigr]=0,
\end{equation}
where $\lVert\cdot\rVertast$ is the matrix norm induced by the $\rq$-norm and  $\nabla\Omega(\pi)$ is a shorthand notation for $[\nabla\Omega(\pi^s_t)]_{\sincS}$. To prove this argument, we use the continuity of $\nabla\Omega$ at $\piE$; this means for any $\varepsilon > 0$, there exists some $0 <\delta \le 1$ such that $\lVert   \nabla \Omega(\pi)-\nabla\Omega(\piE) \rVertast < \varepsilon$ whenever $\lVert\pi-\piE\rVert < \delta$.

When $\lVert\pi-\piE \rVert \ge \delta$, we apply the $L$-\Lip{} continuity assumption to find
\begin{equation}\label{eq:Llip}
  \bigl\lVert \nabla\Omega(\pi)-\nabla\Omega(\piE) \bigrVertast \le L \lVert \pi-\piE\hspace*{-.5pt} \rVert,
\end{equation}
where $\lVert\cdot\rVertast$ is a matrix norm induced by the $q$-norm.  Combining \eq{eq:claim} and \eq{eq:Llip}, we know that
\begin{equation}\label{eq:epslip}
  \bE_{\tauonet}\!\bigl[\,\bigl\lVert \nabla\Omega(\pi_t)-\nabla\Omega(\piE)  \bigr\rVert_{\!\ast}\bigr] \le \varepsilon+L\cdot\bE_{\tauonet}\bigl[\lVert \pi_t-\piE \rVert\bigr].
\end{equation}
Since $\lim_{t\to\infty}\bE_{\tauonet}\!\bigl[\lVert \piE-\pi_t\rVert\bigr]=0$ ensures the existence of some $n\in\bN$ such that for $t>n$, it holds that $\bE_{\tauonet}\!\bigl[\lVert \piE-\pi_t \rVert\bigr] < \nicefrac{\varepsilon}{L}$. Applying this inequality to \eq{eq:epslip}, we have $\bE_{\tauonet}\!\bigl[\,\bigl\lVert \nabla\Omega(\pi_t)-\nabla\Omega(\piE) \bigr\rVert_{\!\ast}\bigr] < 2\varepsilon$ for some $t > n$.

\medskip
For temporal estimations, let us define the infimum of the expectation throughout the time as
\begin{equation*}
  \ell=\inf_{\pi\in\Delta^\cS_\cA}\bE\bigl[
\bigl\lVert\nabla\Omega(\pi_t)-\nabla\Omega(\barpiEt)\bigrVertast\bigr] > 0.
\end{equation*}
From Lemma~\ref{lem:gradient_equiv}, we have $\eta_t \bigl(\hspace*{-.5pt}\nabla\Omega(\pi^s_t)-\nabla\Omega(\barpiEst)\bigr)=\nabla\Omega(\pi^s_t)-\nabla\Omega(\pi^s_{t+1})$ for every $s$. Taking the expectations, for every state $s$, the following inequality holds:
\begin{equation*}
  \eta_t\ell \le \eta_t\bE_{\tauonett}\!\bigl[\,\bigl\lVert \nabla\Omega(\pi^s_{t+1})-\nabla\Omega(\barpiEst)\big\rVert_{\ast}\bigr]=\bE_{\tauonett}\!\bigl[\,\bigl\lVert \nabla\Omega(\pi^s_t)-\nabla\Omega(\pi^s_{t+1}) \bigr\rVert_{\!\ast}\bigr]\quad \forall s\in\cS.
\end{equation*}
Hence the convergence of the point $[\nabla\Omega(\pi^s_{t})]_{s\in\cS}$ is confirmed by taking the limit: $\lim_{t\to\infty}\eta_t=0$.

Next, we show $\sum_{t=1}^\infty \eta_t=\infty$. By the $\omega$-strong convexity by the $L$-\Lip{} continuity of $\Omega$, we can find inequalities as
\begin{equation}\label{eq:ineqq}
  \bigl\langle \nabla\Omega(\pi^s)\!-\!\nabla\Omega(\tilde{\pi}^s),\, \pi^s\!-\!\tilde{\pi}^s\bigranglecA \le L \lVert \pi^s-\tilde{\pi}^s \rVert^2 \le \frac{2L}{\omega}D_\Omega\hspace*{-.5pt}(\pi^s\Vert\tilde{\pi}^s)\qquad\forall s\in\cS.
\end{equation}
We note that $\lVert \pi^s_{t+1}-\barpiEst\rVert \le \lVert \pi^s_{t}-\barpiEst\rVert $ so that there is a constant $\varepsilon$ that satisfies $\bE[\lVert \pi^s_{t+1}-\barpiEstt\rVert] \ge \bE[\lVert \pi^s_{t+1}-\barpiEst\rVert]+\varepsilon$. Therefore, taking expectations in \eq{eq:identity} (and setting $\pi_a=\barpiEst$, $\pi_b=\pi^s_{t+1}$, and $\pi_c=\pi^s_t$) from \lem{lem:two_identity2}, for the strongly convex $\Omega$, we can find
\begin{align}
  &\bE_{\tauonett}\!\bigl[D_\Omega(\pi^s_{t+1}\Vert\barpiEstt)\bigr]\ge\bE_{\tauonett}\!\bigl[D_\Omega(\pi^s_{t+1}\Vert\barpiEst)\bigr]+ \varepsilon^\prime\nonumber\\
  &\qquad\ge (1\!-\!a\eta_t)\,\bE_{\tauonet}\bigl[D_\Omega(\pi^s_t\Vert\barpiEst )\bigr]+\bE_{\tauonett}\bigl[D_\Omega(\pi^s_{t+1}\Vert \pi^s_t)\bigr]+\varepsilon^\prime &\textit{ \eqbrief{eq:identity}}\nonumber\\
  &\qquad\ge (1\!-\!a\eta_t)\,\bE_{\tauonet}\bigl[D_\Omega(\pi^s_t\Vert\barpiEst )\bigr] +\varepsilon^{\dprime} \qquad \forall s \in\cS, \label{eq:plug}
\end{align}
for some $t$ and $0<\varepsilon^\prime<\varepsilon^{\dprime}$ when for $\lim_{t\to\infty}\eta_t=0$. The positive constant $a\coloneqq\nicefrac{2 L}{\omega}$ is derived by the inequalities in \eq{eq:ineqq}.

Since $\lim_{t\to\infty}\eta_t=0$, we can also find a constant $n\in\bN$ such that $\eta_t \le (3a)^{-1}$ for $t\ge n$. Applying the inequality $1-x > \exp(-2x)$ for $x\in(0,1/3]$, we derive another inequality
\begin{equation}\label{eq:ineq}
  \bE_{\tauonett}\!\bigl[\hspace*{1pt}D_\Omega\hspace*{-1pt}\bigl(\pi^s_{t+1}\big\Vert \barpiEstt\bigr)\bigr] \ge \exp(-2a\eta_t)\,\bE_{\tauonet}\!\bigl[D_\Omega\hspace*{-1pt}\bigl( \pi^s_t\big\Vert\barpiEst\bigr)\bigr],\quad\forall\ t\ge n\quad \forall s \in \cS.
\end{equation}
Applying this for $t=T\!-\!1,\dots,n$ yields
\vspace*{-5pt}
\begin{equation}\label{eq:yield}
\begin{aligned}
  \bE_{\tauoneT}\!\bigl[\hspace*{1pt}D_\Omega\bigl(\hspace*{1pt}\piT^s\hspace*{.5pt}\big\Vert\hspace*{1pt}\barpiET^s\bigr)\bigr] &\ge \Biggl(\,\prod_{t=n+1}^T\!\!\exp\bigl(-2a\eta_t\bigr)\!\!\Biggr)\,\bE_{\tauonen}\!\bigl[\hspace*{1pt}D_\Omega\bigl(\pi^s_n\big\Vert\barpiEsn \bigr)\bigr]\\
  &=\exp\hspace*{-1pt}\Biggl(\hspace*{-1pt}-2a\cdot\hspace*{-6pt}\sum_{t=n+1}^T \!\eta_t\!\Biggr)\, \bE_{\tauonen}\!\bigl[D_\Omega\bigl(\pi^s_n\big\Vert \barpiEsn\bigr)\bigr].
\end{aligned}
\end{equation}
Using \eq{eq:yield}, we conclude $\bE_{\tauonen}\!\bigl[D_\Omega(\pi^s_{n}\Vert\barpiEsn)\bigr] >0$ for some $s$. Otherwise, we have that
\begin{equation*}
  \bE_{\tauonen}\!\bigl[D_\Omega(\pi^s_n\Vert\barpiEsn)\bigr] =\bE_{\tauonenn}\!\bigl[D_\Omega(\pi^s_{n+1}\Vert\barpiEsnn)\bigr]=0\quad\forall s \in\cS,
\end{equation*}
which leads to $\bE_{\tauonen}\!\bigl[\lVert\pi_n\!-\!\barpiEn \rVert^2\bigr]=\bE_{\tauonenn}\!\bigl[\lVert\pi_{n+1}\!-\!\barpiEnn\rVert^2\bigr]=0$, according to \eq{eq:ineq}. This implies $\pi_n=\barpiEn=\pi_{n+1}$ almost surely, leading to $\bE[f(\pi_t,\tau_t)]=0$. Essentially, this is a contradiction to the previous assumption $\inf_{\pi\in\Delta^\cS_\cA}\bE[f(\pi,\tau_t)] > 0$; thus, $\bE_{\tauonenn} \bigl[D_\Omega(\pi_{n+1}\Vert\barpiEnn)\bigr] >0$. Let us assume the ideal case in which the estimation process learns the exact $\piE$ in $t\to\infty$. To satisfy the limit $\lim_{T\to\infty} \bE_{\tauoneT}\bigl[D_\Omega(\piT\Vert \barpiET)\bigr]=0$ we see from \eq{eq:yield} that $\sum_{t=1}^\infty \eta_t=\infty$.

Now, we show that \thm{thm:step_size}~(a) holds. Since $\Omega$ is $\omega$-strongly convex, so basically $\Omega^\ast$ is $(\omega^{-1}\hspace*{-1pt})$-strongly smooth with respect to $\lVert\cdot\rVertast$. On the other hand, the $L$-\Lip{} continuity of $\nabla\Omega$ implies $L$-strong smoothness of $\Omega$; thus, $\Omega^\ast$ is $L$-strongly convex.

Since $\lim_{t\to\infty}\bigl\lVert\nabla\Omega(\pi^s_t)-\nabla\Omega(\pi^s_{t+1})\bigr\rVert_{\!\ast}=0$ for $\forall t\ge n$ and $\forall s \in \cS$, the condition $\eta_t \le (3a)^{-1}$ induces
\begin{equation*}
\begin{aligned}
  &\bE_{\tauonett}\!\bigl[\hspace*{1pt}D_\Omega(\pi^s_{t+1}\Vert\barpiEstt)\bigr] \\
  &\qquad\ge(1\!-\!a\eta_t) \,\bE_{\tauonet}\!\bigl[\hspace*{1pt}D(\pi^s_t\Vert\barpiEst)\bigr]+(2L)^{-1} \bE_{\tauonett}\!\Bigl[\bigl\lVert\nabla\Omega(\pi_t)-\nabla\Omega(\pi_{t+1}) \bigr\rVert_{\!\ast}^{\!2}\Bigr],\\
  &\qquad\ge(1\!-\!a\eta_t)\,\bE_{\tauonet}\!\bigl[D_\Omega(\pi^s_t\Vert\barpiEst)\bigr]+(2L)^{-1} \eta_t^2\,\bE_{\tauonett}\!\Bigl[\bigl\lVert\nabla\Omega(\pi_t)-\nabla\Omega(\barpiEt) \bigr\rVert_{\!\ast}\Bigr]. \quad\textit{Lemma~\ref{lem:gradient_equiv}}
\end{aligned}
\end{equation*}
Using the \CaSc{} inequality, we obtain a lower bound of the last term as
\begin{equation*}
  \bE_{\tauonet}\!\Bigl[\bigl\lVert\nabla\Omega(\pi_t)-\nabla\Omega(\barpiEt)\bigr\rVert_{\!\ast}^{\!2}\Bigr]\ge\Bigl\{\bE_{\tauonet}\!\Bigl[\bigl\lVert\nabla\Omega(\pi_t)-\nabla\Omega(\barpiEt)\bigr\rVert_{\!\ast}\Bigr]\Bigr\}^{\!2}\ge \ell^{\,2}.
\end{equation*}
Thus, we obtain the final inequality for all $s\in\cS$ as
\begin{equation*}
  \bE_{\tauonett}\!\bigl[\hspace*{1pt}D_\Omega(\pi^s_{t+1}\Vert\barpiEstt)\bigr] \ge (1-a\eta_t)\,\bE_{\tauonet}\!\bigl[\hspace*{1pt}D_\Omega(\pi^s_t\Vert\barpiEst)\bigr]+(2L)^{-1}(\eta_t\ell)^2,\quad\forall\ t \ge n.
\end{equation*}

Applying this inequality from $t=T\ge n+1$ to $t=n+1$, we achieve
\vspace*{-5pt}
\begin{equation*}
\begin{aligned}
  \bE_{\tauoneTT}\!\bigl[\hspace*{1pt}D_\Omega(\pi^s_{\!{\scriptscriptstyle T}+1}\Vert \barpiEsTT\!)\bigr]&\ge \bE_{\tauonen}\!\bigl[D_\Omega(\pi^s_n\Vert\barpiEsn)\bigr]\prod_{t=n+1}^T\hspace*{-3pt}(1-a\eta_t)\\&\qquad+ (2L)^{-1}\ell^{\,2}\hspace*{-2pt}\sum_{t=n+1}^{T}\hspace*{-3pt}\eta^2_t\prod_{k=t+1}^T\hspace*{-3pt}(1-a\eta_k)\\
  &\ge(2L)^{-1}\ell^{\,2}\sum_{t=n+1}^{T}\hspace*{-3pt}\eta^2_t\prod_{k=t+1}^T\hspace*{-3pt}(1-a\eta_k).
\end{aligned}
\end{equation*}
By the \CaSc{} inequality and our bound $0<1-a\eta_k\le 1$ for $k\ge n$, we have
\begin{equation*}
  \sum_{t=n+1}^T\hspace*{-3pt} \eta_t \prod_{k=t+1}^T \hspace*{-3pt}(1-a\eta_k) \le \Biggl\{\sum_{t=n+1}^T\hspace*{-3pt} \eta^2_t \prod_{k=t+1}^{T}\hspace*{-3pt} (1-a\eta_k) \Biggl\}^{\!1/2}\hspace*{-6pt}(T-n)^{1/2}.
\end{equation*}

Hence
\begin{equation*}
\begin{aligned}
  \sum_{t=n+1}^{T}\hspace*{-3pt} \eta^2_t \prod_{k=t+1}^{T}\hspace*{-3pt}(1-a\eta_k) &\ge \frac{1}{a^2(T-n)}\Biggl(\sum_{t=n+1}^{T} \hspace*{-3pt}a\eta_t \prod_{k=t+1}^{T} \hspace*{-3pt}(1-a\eta_k)\!\Biggr)^2\\
  &=\frac{1}{a^2(T-n)}\Biggl(\sum_{t=n+1}^{T} (1-(1-a\eta_t)) \prod_{k=t+1}^T\hspace*{-3pt} (1-a\eta_k)\!\Biggr)^2\\
  &=\frac{1}{a^2(T-n)}\Biggl(\sum_{t=n+1}^{T}\hspace*{-3pt}\Bigg[\prod_{k=t+1}^T \hspace*{-3pt}(1-a\eta_k)- \prod_{k=t}^T (1-a\eta_k)\Biggr]\!\Biggr)^2\\
  &\ge\frac{1}{a^2(T-n)}\Biggl(\sum_{t=n+1}^{T}\hspace*{-3pt} 1-\prod_{k=t}^{T}(1-a\eta_k)\!\Biggr)^2\\
  &\ge\frac{1}{a^2(T-n)}\bigl(1-(1\!-\!a\eta_{n+1})\bigr)^{\!2}=\frac{\eta^2_{n+1}}{T-n}\\
\end{aligned}
\end{equation*}
Therefore, we obtain the lower bound of
\begin{equation*}
  \bE_{\tauoneT}\!\bigl[\hspace*{1pt}D_\Omega(\pi_{\!{\scriptscriptstyle T}+1}\Vert\barpiETT)\bigr]\ge\frac{\eta_{n+1}^2(2L)^{-1}\ell^2}{T-n}.
\end{equation*}
Since the \Bg{} divergence is assumed to be bounded for all states, the sequence $\{\gamma^i D_\Omega(\pi^{s_i}_t \Vert \pi^{s_i}_t) \}$ will converge as $i\to\infty$. Applying the monotone convergence theorem, we can interchange expectation and summation, which yields
\begin{equation*}
\begin{aligned}
  \bE_{\tauoneT}\!\Biggl[\hspace*{2pt}\sum_{i=0}^\infty \gamma^i D_\Omega\hspace*{-1pt}\Bigl(\piT\bigl(\,\cdot\,\big|s_i\bigr)\Big\Vert\hspace*{1.5pt}\barpiET\hspace*{-.5pt}\bigl(\,\cdot\,\big|s_i\hspace*{-.5pt}\bigr)\hspace*{-1.2pt}\bigr)\Biggr] &=\sum_{i=0}^\infty\bE_{\tauoneT}\!\Bigl[\hspace*{1pt}\gamma^i D_\Omega\bigl(\piT\bigl(\,\cdot\,\big|s_i)\big\Vert\barpiET(\cdot|s_i)\bigr)\Bigr]\\
  &=\sum_{i=0}^\infty \gamma^i \bE_{\tauoneT} \Bigl[D_\Omega\bigl(\piT(\cdot| s_i)\big\Vert\barpiET(\cdot|s_i)\bigr)\Bigr]\\
  &\ge \frac{\eta_{n+1}^2(2L-2L\gamma)^{-1}\ell^2}{T-n},\quad\forall\ T\ge n.
\end{aligned}
\end{equation*}
This verifies \thm{thm:step_size}~(a) with the constant $c=\eta^2_{n+1} (2L-2L\gamma)^{-1}\ell^2$.
\end{proof}

Lastly, we show convergence to a unique fixed point of $\pi_\ast$ using the particular form of $\eta_t$ in \eq{eq:step_cond}.
\begin{lemma}\label{lem:further}
  If $\{\eta_t\}_{t=1}^\infty$ satisfies \eq{eq:step_cond}, $\lim_{t\to\infty}\bE_{\tauonet}\!\bigl[\sum_{i=0}^\infty \gamma^i D_\Omega(\pi^s_\ast \Vert \pi^s_t)\bigr]=0$. Furthermore, if the step size takes the form $\eta_t=\frac{4}{t+1}$, then $\textstyle{\bE_{\tauoneT}\!\bigl[\sum_{i=0}^\infty \gamma^{i} D_\Omega\bigl(\pi_\ast^{s_i}\big\Vert\piT^{s_i}\bigr)\bigr]=\cO\bigl(1/T)}$.
\end{lemma}
\begin{proof}[Proof of Lemma~\ref{lem:further}]
According to Lemma~\ref{lem:two_identity1} and the fundamental identity of \Bg{} divergence for the convex conjugate $\Omega^\ast$, the one-step progress regarding $\barpiEst$  can be written as
\begin{equation}\label{eq:written}
\begin{aligned}
 & D_\Omega\hspace*{-.5pt}(\pi^s_\ast \Vert \pi^s_{t+1}\hspace*{-.5pt})- D_\Omega(\pi^s_\ast\Vert \pi^s_t) =\bigl\langle \nabla \Omega(\pi^s_t)\!-\!\nabla\Omega(\pi^s_{t+1}),\,  \pi^s_\ast\!-\!\pi^s_t\bigr\rangle_{\!\!\scriptscriptstyle\cA}+D_\Omega\hspace*{-.5pt}(\pi^s_{t}\Vert \pi^s_{t+1})\\
  &\hspace*{60pt} =\eta_t \bigl\langle \nabla \Omega(\pi^s_t)\!-\!\nabla\Omega(\barpiEst),\,  \pi^s_\ast\!-\!\pi^s_t\bigr\rangle_{\!\!\scriptscriptstyle\cA}+D_{\Omega^\ast}\hspace*{-1.5pt}\bigl(\nabla\Omega(\pi^s_{t+1})\big\Vert\nabla\Omega(\pi^s_{t})\bigr),
\end{aligned}
\end{equation}
for all $s\in\cS$. As $\omega$-strong convexity of $\Omega$ implies the $(\omega^{-1})$-strong smoothness of $\Omega^\ast$, we have
\begin{equation}
  D_{\Omega^\ast}\hspace*{-1.5pt}\bigl(\nabla\Omega(\pi^s_{t+1})\big\Vert\nabla\Omega(\pi^s_{t})\bigr) \,\le\, \frac{1}{2\omega}\bigl\lVert  \nabla\Omega(\pi^s_{t+1})\!-\!\nabla\Omega(\pi^s_{t})\bigr\rVert_\ast^2\,=\,\frac{\eta^2_t}{2\omega}\bigl\lVert \nabla\Omega(\barpiEst) \!-\!\nabla\Omega(\pi^s_t)\bigr\rVert_{\ast}^2
\end{equation}

Then, we bound $\lVert\nabla\Omega(\barpiEst)\!-\!\nabla\Omega(\pi^s_t)\rVert_{\ast}^{2}$ by $2\lVert\nabla\Omega(\pi^s_t)\!-\!\nabla\Omega(\pi^s_\ast)\rVert_{\ast}^{2}+2\lVert \nabla\Omega(\pi^s_\ast)\!-\!\nabla\Omega(\barpiEst) \rVert_{\ast}^{2}$, following the work of \citet{omd_converge}. Since $\nabla\Omega$ is cocoercive with $\frac{1}{L}$ by the \Lip{} continuity of $\nabla\Omega$, we obtain
\begin{equation*}
  \bigl\lVert\nabla\Omega(\pi^s_t)\!-\!\nabla\Omega(\pi^s_\ast\hspace*{-1pt})\bigr\rVert_{\ast}^{2} \le L \bigl\langle \nabla\Omega(\pi^s_\ast)\!-\!\nabla\Omega(\pi^s_t), \pi^s_\ast-\pi^s_t \bigr\rangle
\end{equation*}
thus, using \eq{eq:written}, we get
\begin{equation}\label{eq:thus}
\begin{aligned}
  &D_\Omega(\pi^s_\ast\Vert\pi^s_{t+1})-D_\Omega(\pi^s_\ast\Vert \pi^s_t)\le \eta_t \langle \nabla\Omega(\pi^s_\ast)- \nabla\Omega(\barpiEst), \pi^s_\ast-\pi^s_t \rangle\\
  &\qquad -\biggl(  1- \frac{\eta_tL}{\omega}\biggr) \eta_t \langle \nabla\Omega(\pi^s_\ast)-\nabla\Omega(\pi^s_t), \pi^s_\ast-\pi^s_t \rangle+ \frac{\eta^2_t}{\omega}  \bigl(\lVert \nabla\Omega(\pi^s_\ast)-\nabla\Omega(\barpiEst) \rVertast^2\bigr).
\end{aligned}
\end{equation}
By taking expectation, it follows that there exists $n \in\bN$ such that $\eta_t \le \frac{\omega}{2L}$ for $t \ge n$ holds
\begin{align}
  \bE_{\tauonett}\!\bigl[\hspace*{1pt}D_\Omega\hspace*{-.5pt}(\hspace*{1pt}\pi^s_\ast \Vert \pi^s_{t+1}\hspace*{-1pt})\bigr]&\le \bE_{\tauonet}\!\biggl[ D_\Omega\hspace*{-1pt}\bigl(\hspace*{1pt}\pi^s_\ast \big\Vert \hspace*{1pt}\pi^s_t\bigr)-\frac{\eta_t}{2}D_\Omega\hspace*{-1pt}\bigl(\hspace*{1pt}\pi^s_\ast \big\Vert \hspace*{1pt} \pi^s_t\bigr)+\frac{\eta^2_t}{\omega}\bigl\lVert\nabla\Omega(\pi^s_\ast) - \nabla\Omega(\barpiEst) \bigrVertast^2\biggr],\nonumber\\
&\le \bE_{\tauonet}\!\biggl[ D_\Omega\hspace*{-1pt}\bigl(\hspace*{1pt}\pi^s_\ast \big\Vert \hspace*{1pt}\pi^s_t\bigr)-\frac{\eta_t}{2}D_\Omega\hspace*{-1pt}\bigl(\hspace*{1pt}\pi^s_\ast \big\Vert \hspace*{1pt} \pi^s_t\bigr)\biggr]+z\eta^2_t,\label{eq:imply}
\end{align}
where $z$ is the constant $z=\frac{1}{\omega}\bE[\lVert \nabla \Omega(\pi_\ast)-\Omega(\barpiEt)\rVertast^2]$. Let $\{A_t\}_{t=1}^\infty$ denote a sequence of $A_t=\sup_{s\in\cS}\bE_{\tauonet}\bigl[D_\Omega(\pi^s_\ast\Vert\pi^s_t)\bigr]$. Then we have
\begin{equation}\label{eq:combine1}
  A_{t+1} \le \biggl(1\!-\!\frac{\eta_t}{2}\biggr)A_t+z\eta^2_t,\quad\forall t \ge n.
\end{equation}

For a constant $h > 0$, we claim that $A_{t_1}<h$ for some $t_1>n^\prime$. Assume that this is not true, and we find some $t_2\ge t_1$ such that $A_t > h$, $\forall t \ge t_2$. Since $\lim_{t\to\infty}\eta_t=0$, there are some $t>t_3>t_2$ that $\eta_t \le \frac{h}{4b}$. However, \eq{eq:combine1} tells us that for $t\ge t_3$,
\begin{equation*}
  A_{t+1} \le \biggl(1\!-\!\frac{\eta_t}{2}\biggr)A_t+z\eta^2_t \le A_{t_3}-\frac{h}{4}\sum_{k=t^\prime_\gamma}^t \hspace*{-1pt}\eta_k \to -\infty\quad\textrm{(as $t\to\infty$)}.
\end{equation*}
This is a contradiction, which verifies $A_t<h$ for $t > n^\prime$. Since $\lim_{t\to\infty}\eta_t =0$, we can find some $\eta_t$ that makes $A_t$ monotonically decreasing. Then, we can conclude that the nonnegative sequence $\{A_t\}_{t=1}^\infty$ converges by iteratively applying the upper bounds.

We now prove \thm{thm:step_size}~(b) under the consideration of the condition $\eta_t=\frac{4}{t+1}$. The estimate becomes
\begin{equation*}
  A_{t+1} \le \biggl(1-\frac{2}{t+1}\biggr) A_t+\frac{16z}{(t+1)^2},\quad\forall t \ge n.
\end{equation*}
It follows the recurrence relation is
\begin{equation*}
  t(t+1) A_{t+1} \le (t-1) t A_t+16z,\quad\forall t \ge n.
\end{equation*}
Iteratively applying this relation, we obtain the general form of inequality.
\begin{equation*}
  (T-1)T A_T \le (n-1)n A_{n}+16z(T-n),\quad\forall T\ge n,
\end{equation*}
therefore we obtain the inequality as follows:
\begin{equation*}
  \bE_{\tauoneT}\!\bigl[\hspace*{1pt}D_\Omega(\pi^s_\ast \Vert \piT^s)\bigr]\le \frac{(n-1)n\bE_{\tauonen}\!\bigl[\hspace*{1pt}D_\Omega(\pi^s_\ast\Vert \pi^s_{n})\bigr]}{(T-1)T}+ \frac{16z}{T},\quad \forall T \ge n,\quad \forall s\in\cS.
\end{equation*}
By applying the monotone convergence theorem,  we can interchange expectation and summation, which yields similar result to formultion from Proof of Lemma~\ref{lem:assume}
\begin{equation*}
\bE_{\tauoneT}\!\Biggl[\hspace*{1pt}\sum_{i=1}^\infty \gamma^i D_\Omega\hspace*{-1pt}\Bigl(\pi_{\hspace*{-1pt}\ast}\hspace*{-1.2pt}\bigl(\,\cdot\,\big|\hspace*{1pt}s_i\bigr)\Big\Vert\hspace*{1pt}\piT(\,\cdot\,|\hspace*{1pt} s_i)\Bigr)\Biggr]=\cO\biggl(\frac{1}{T}\biggr).
\end{equation*}
Therefore, the proof is complete.
\end{proof}

\subsection{Proof of Theorem 2}
\textit{Necessity.\topicquad}First, we rewrite the inequality in \eq{eq:plug} as
\begin{equation}\label{eq:hence}
  \bE_{\tauonett}\hspace*{-2pt}\bigl[D_\Omega\bigl(\pi^s_{t+1}\big\Vert\barpiEstt \bigr)\hspace*{-.5pt}\bigr] \ge (1\!-\!2L\omega^{\!-1} \eta_t)\,\bE_{\tauonet}\hspace*{-2pt}\bigl[D_\Omega\bigl(\pi^s_t\big\Vert \barpiEst\bigr)\bigr],\quad\forall s\in\cS.
\end{equation}
Since we assume that $\eta_t$ converges to 0 from previous arguments, consider the step size sequence $0<\eta_t\le\frac{\omega}{(2+\kappa)L}$ for $\kappa>0$ and $t\ge n$ where $\forall n\in\bN$. Denote a constant $\tilde{a}=\frac{2+\kappa}{2}\log\frac{2+\kappa}{\kappa}$ and apply the elementary inequality
\begin{equation*}
  1-x\ge\exp\bigl(-\tilde{a} x\bigr),\quad\textrm{such that}\ 0<x\le\frac{2}{2+\kappa}
\end{equation*}
From \eq{eq:hence}, it can be obtained that
\begin{equation*}
  \bE_{\tauonett}\!\bigl[\hspace*{1pt}D_\Omega\bigl(\pi^s_{t+1}\big\Vert\barpiEstt\bigr)\hspace*{-.5pt}\bigr]\ge \exp\bigl(-2\tilde{a}L\omega^{-1}\eta_t)\bE_{\tauonet}\bigl[D_\Omega(\pi^s_t\Vert\barpiEst)\bigr].
\end{equation*}
Applying this inequality iteratively for $t=n,\dots,T-1$ gives
\begin{equation}\label{eq:gives}
\begin{aligned}
  \bE_{\tauoneT}\!\bigl[\hspace*{1pt}D_\Omega(\piT^s\Vert \barpiET^s)\hspace*{-.5pt}\bigr] &\ge \bE_{\tauonen}\!\bigl[\hspace*{1pt} D_\Omega\hspace*{-1pt}\bigl(\hspace*{1pt}\pi^s_n\big\Vert\hspace*{1pt}\barpiEsn\hspace*{-.5pt}\bigr)\bigr]\prod_{t=n}^{T-1}\exp\bigl(-2\tilde{a}L\omega^{-1}\eta_t\bigr)\\
  &=\exp\Biggl\{\!-2\tilde{a}L \omega^{-1}\hspace*{-2pt} \sum_{t=n}^{T-1}\hspace*{-1pt} \eta_t\!\Biggr\}  \bE_{\tauonen}\!\bigl[\hspace*{1pt} D_\Omega\hspace*{-1pt}\bigl(\hspace*{1pt}\pi^s_n\big\Vert\hspace*{1pt}\barpiEsn\hspace*{-.5pt}\bigr)\bigr]\quad\forall s\in\cS.
\end{aligned}
\end{equation}
From the assumption $\piE\ne \pi_n$, we have $D_\Omega\hspace*{-1pt}\bigl(\hspace*{1pt}\pi^s_n\big\Vert\hspace*{1pt}\barpiEsn\hspace*{-.5pt}\bigr)>0$ for some states. Therefore, by \eq{eq:gives}, the convergence $\lim_{t\to\infty}\bE_{\tauonet}\!\bigl[\hspace*{.5pt}D_\Omega(\pi^s_t\Vert \barpiEst )\hspace*{-.5pt}\bigr]=0$ for all states implies $\sum_{t=1}^\infty \eta_t=\infty$.

\medskip
\textit{Sufficiency.\topicquad}We use \eq{eq:combine1} in the proof of \lem{lem:further}. In the optimal case, $\lVert \nabla \Omega(\pi_\ast)-\Omega(\barpiEt)\rVertast=0$, so (\ref{eq:combine1}) takes the form (we can choose $n=1$ by \eq{eq:imply})
\begin{equation}\label{eq:derive}
  A_{t+1} \le \frac{\eta_t}{2}A_t,\quad\forall t\in\bN,
\end{equation}
where $A_t=\sup_{s\in\cS}\bE_{\tauonet}\bigl[D_\Omega(\pi^s_\ast\Vert\pi^s_t)\bigr]$ (and also $A_t = \sup_{s\in\cS}\bE_{\tauonet}\bigl[D_\Omega(\piE^s\Vert\pi^s_t)\bigr]$ for the specific parameterization of $\piE\in\Pi$).  Therefore, for any $0<h<1$, there must exist some $t_1 \in \bN$ such that $A_{t} \le h$ for $t\ge t_1$. Otherwise, $A_t > h$ for every $t \ge t_2$ with $t_2\ge t_1$, which leads to a contradiction:
\begin{equation*}
  A_{t+1} \le  A_{t_2}-\frac{h}{2} \hspace*{-.5pt}\sum_{k=t_1}^t\hspace*{-2pt} \eta_k \ \to\ -\infty\ \text{ (as $t\to\infty$)}.
\end{equation*}
\eq{eq:derive} also tells us that the sequence $\{A_t\}_{t=1}^\infty$ is monotonically decreasing. Hence $A_{t}\le h$ for every $t\ge t_1$, which proves the convergence with respect to the least upper bound of \Bg{} divergences by combining with \eq{eq:derive}
\begin{equation*}
  \lim_{t\to\infty} \sup_{s\in\cS} \bE_{\tauonet}\!\bigl[\hspace*{.5pt}D_\Omega(\pi^s_\ast\Vert\pi^s_t)\bigr]=\lim_{t\to\infty} A_t=0.
\end{equation*}
We now prove the second point in \thm{thm:conv_optimal} which is under the special condition of $\eta_t \equiv \eta_1$. It follows from \eq{eq:hence} that $A_T \ge (1-2L\omega^{\!-1} \eta_1)^{T\hspace*{-.5pt}-\hspace*{-.5pt}1} A_1$. Hence, Eq~(\ref{eq:derive}) translates to
\begin{equation*}
  A_{t+1} \le (1-\eta_1/2) A_t,
\end{equation*}
from which we find $A_{T} \le (1-\eta_1/2)^{T-1} A_1$ by iteration starting from $t=1$. Therefore, the second point is verified the theorem with $c_1=1\!-\!\frac{2L\eta_1}{\omega}$ and $c_2=1\!-\!\frac{\eta_1}{2}$.\hfill\qedsymbol

\subsection{Proof of Proposition 1}
The proof of Proposition~1 is based on the Doob's forward convergence theorem.
\begin{theorem}[Doob's forward convergence theorem]\label{thm:doob}
  Let $\{X_t\}_{t\in\bN}$ be a sequence of nonnegative random variables and let $\{\cF_t\}_{t\in\bN}$ be a filtration with $\cF_t\subset \cF_{t+1}$ for every $t\in\bN$. Assume that $\bE\bigl[X_{t+1}\vert\cF_t\bigr]\le X_t$ almost surely for every $t\in\bN$. Then the sequence $\{X_t\}$ converges to a nonnegative random variable $X_\infty$ almost surely.
\end{theorem}
We follow the proof of \lem{lem:further} and apply \eqbrief{eq:thus}. Since $\langle \pi^s_\ast-\pi^s_t, \nabla \Omega(\pi^s_\ast)-\nabla \Omega(\pi^s_t)\rangle \ge 0$ for all $s \in\cS$, \eq{eq:thus} implies: there exists $n\in\bN$ that
\begin{equation}\label{eq:implies2}
  \bE_{\tau_t}\!\bigl[\hspace*{1pt}D_\Omega(\pi^s_\ast \Vert\pi^s_{t+1})\hspace*{.5pt}\bigr] \le D_\Omega\hspace*{-1pt}\bigl(\hspace*{1pt}\pi^s_{\hspace*{-.5pt}\ast} \big\Vert \pi^s_t\bigr)+\frac{\eta_t^2}{\omega}\bE\Bigl[\bigl\lVert\hspace*{-.5pt}\nabla\Omega(\pi_{\hspace*{-.5pt}\ast})\!-\! \nabla\Omega(\barpiEt)\bigr\rVert_{\ast}^2\Bigr],\quad\forall\ t\ge n,\ \forall s \in\cS,
\end{equation}
and since the step size is scheduled as  $\lim_{t\to\infty}\eta_t = 0$, the following equation also holds:
\begin{equation}\label{eq:holds}
  \bE_{\tau_t}\!\biggl[\hspace*{1pt} \sup_{s\in\cS}D_\Omega(\pi^s_\ast \Vert\pi^s_{t+1})\hspace*{.5pt}\biggr] \le \sup_{s\in\cS}D_\Omega\hspace*{-1pt}\bigl(\hspace*{1pt}\pi^s_{\hspace*{-.5pt}\ast} \big\Vert \pi^s_t\bigr)+\frac{\eta_t^2}{\omega}\bE\Bigl[\bigl\lVert\hspace*{-.5pt}\nabla\Omega(\pi_{\hspace*{-.5pt}\ast})\!-\! \nabla\Omega(\barpiEt)\bigr\rVert_{\ast}^2\Bigr],\quad\forall\ t\ge n^\prime,
\end{equation}
for some $n^\prime\in\bN$. Then, the condition $\sum_{t=1}^\infty \eta_t^2 < \infty$ enables us to define a stochastic process $\{X_t\}$:
\begin{equation*}
  X_t\coloneqq \sup_{s\in\cS}\ D_\Omega(\pi^s_\ast \Vert \pi^s_{t+1})+\frac{1}{\omega}\bE\Bigl[\bigl\lVert \nabla\Omega(\pi^s_{\hspace*{-1pt}\ast})\!-\!\nabla\Omega(\barpiEst)\bigr\rVert_{\ast}^2\Bigr]\hspace*{-1pt}\sum_{i=t+1}^\infty\hspace*{-2pt} \eta^2_i.
\end{equation*}
Thus, by \eq{eq:holds}, it is straightforwardly derived that there exits $n\in\bN$ that $\bE_{\tau_t}[X_{t+1}]\le X_t$ for $t\ge n$. Since $X_t\ge0$, the stochastic process $\{X_t\}_{t-n+1\ge 1}$ is a submartingale (equivalently, $\{-X_t\}_{t-n+1\ge 1}$ is a supermartingale). By \thm{thm:doob}, the sequence $\{X_t\}_{t\ge 1}$ converges to a nonnegative random variable $X_\infty$ almost surely.  Therefore, $D_\Omega(\pi_\ast^s\Vert\pi_t^s )$ converges for every state.

According to Fatou's lemma, and using the convergence of $\lim_{t}\bE_{\tauonet}\bigl[\sum_{i=0}^\infty \gamma^i D_\Omega(\pi^s_\ast\Vert\pi^s_t)\bigr]=0$ proved by \lem{lem:further}, we obtain
\begin{equation*}
  \bE\Biggl[\lim_{t\to\infty}\sum_{i=0}^\infty\gamma^i D_\Omega\hspace*{-1pt}\Bigl(\pi_{\hspace*{-.5pt}\ast}\bigl(\,\cdot\,\big|s_i\hspace*{-.5pt}\bigr)\bigr)\Big\Vert\pi_t\bigl(\,\cdot\,\big|s_i\hspace*{-.5pt}\bigr)\Bigr)\Biggr] \le (1\!-\!\gamma)^{\hspace*{-.5pt}-1}\liminf_{t\to\infty} \bE_{\tauonet}\Biggl[ \sum_{i=0}^\infty \gamma^i D_\Omega(\pi^s_\ast\Vert\pi^s_t)\Biggr]=0.
\end{equation*}
Therefore, it can be concluded that the sequence of costs $\Bigl\{\sum\limits_{i=0}^\infty \gamma^i D_\Omega\hspace*{-.5pt}\bigl( \pi_\ast(\cdot|s_i)\big\Vert\pi_t(\cdot|s_i)\bigr)\!\Bigr\}_{\!t\in\bN}$ converges to $0$ almost surely.\hfill\qedsymbol

\section{\Ts{} Entropy and Associated \Bg{} Divergence Among Full Covariance Multivariate \Gs{} Distributions}\label{appsect:full}
This appendix \textbf{reintroduces} derivations of \Bg{} divergences and regularized reward functions for tractable computation when $\Omega$ is the \Ts{} entropy regularizer, which were previously proposed by \citet{renyi_tsallis} and \citet{rairl}. And then, we delineate a distinct parameterization used in this paper for modeling \Gs{} distribution policies equipped with full covariance matrices.

The standard form of the exponential family is represented as
\begin{equation}
  \exp\bigl\{\bigl\langle \theta, t(x)\bigr\rangle-F(\theta)+k(x) \bigr\}.
\end{equation}
The generalized parameterization of the multi-variate \Gs{} is defined as follows:
\begin{align*}
  \theta &=\begin{bmatrix}
 \Sigma^{-1}\mu \\
 -\frac{1}{2}\Sigma^{-1}
 \end{bmatrix}=\begin{bmatrix}
   \theta_1\\
   \theta_2
 \end{bmatrix},\\
  t(x) &=\begin{bmatrix}
    x\\
    xx^\sT
  \end{bmatrix},\\
  F(\theta) &=-\frac{1}{4}\theta^\sT_1\theta^{-1}_2\theta_1+\frac{1}{2}\ln\lvert -\pi \theta_2^{-1} \rvert=\frac{1}{2} \mu^\sT \Sigma^{-1}\mu+\frac{1}{2}\ln(2\pi)^d\lvert\Sigma\rvert,\\
  k(x) &=0,
\end{align*}
where we can analytically recover the \Gs{} distribution \citep{renyi_tsallis}
\begin{equation}
\begin{aligned}
  &\exp\bigl\{\bigl\langle \theta, t(x)\bigr\rangle-F(\theta)+k(x) \bigr\}\\
  &\qquad=\exp\biggl\{\mu^\sT\Sigma^{-1}x-\frac{1}{2}\tr\bigl(\Sigma^{-1}xx^\sT\bigr) -\frac{1}{2}\mu^\sT\Sigma^{-1}\mu+\frac{1}{2}\ln(2\pi)^d\lvert\Sigma\rvert\biggr\}\\
  &\qquad=\frac{1}{(2\pi)^{d/2}\lvert\Sigma\rvert^{1/2}}\exp\biggl\{\mu^\sT \Sigma^{-1} x-\frac{1}{2} x^\sT\Sigma^{-1}x -\frac{1}{2}\mu^\sT \Sigma^{-1} \mu\biggr\}\\
  &\qquad=\frac{1}{(2\pi)^{d/2}\lvert\Sigma\rvert^{1/2}}\exp\biggl\{\frac{1}{2}(x-\mu)^\sT \Sigma^{-1}(x-\mu)\biggr\}.
\end{aligned}
\end{equation}
For two distributions $\pi$ and $\hat{\pi}$ with $k(x)=0$, \citet{renyi_tsallis} proposed the function $I(\cdot)$:
\begin{equation*}
\begin{aligned}
  I(\pi,\hat{\pi};\alpha,\beta) =\int \pi(x)^\alpha \hat{\pi}(x)^\beta\ \rd x=\exp\Bigl\{F\bigl(\alpha\theta+\beta\hat{\theta}\bigr)-\alpha F(\theta)-\beta F(\hat{\theta})\Bigr\}
\end{aligned}
\end{equation*}
where the detailed derivation is as follows:
\begin{align*}
  &\int \pi(x)^\alpha \hat{\pi}(x)^\beta\ \rd x\\
  &\quad=\int\exp\Bigl\{\alpha\bigl\langle\theta,t(x)\bigr\rangle-\alpha F(\theta)+\beta\bigl\langle\hat{\theta},t(x)\bigr\rangle-\beta F(\hat{\theta})\Bigr\}\ \rd x\\
  &\quad=\int\exp\Bigl\{\bigl\langle\alpha\theta+\beta\hat{\theta},t(x)\bigr\rangle-F\bigl(\alpha\theta+\beta\hat{\theta}\bigr)\Bigr\}\exp\Bigl\{F\bigl(\alpha\theta+\beta\hat{\theta}\bigr)-\alpha F(\theta) -\beta F(\hat{\theta})\Bigr\}\ \rd x\\
  &\quad=\exp\Bigl\{F\bigl(\alpha\theta+\beta\hat{\theta}\bigr)-\alpha F(\theta)-\beta F(\hat{\theta})\Bigr\}\int\exp\Bigl\{\bigl\langle\alpha\theta+\beta\hat{\theta},t(x)\bigr\rangle- F\bigl(\alpha\theta+\beta\hat{\theta}\bigr)\Bigr\}\ \rd x\\
  &\quad=\exp\Bigl\{F\bigl(\alpha\theta+\beta\hat{\theta}\bigr)-\alpha F(\theta)-\beta F(\hat{\theta})\Bigr\}.
\end{align*}
\subsection{\Ts{} entropy of full covariance \Gs{} distributions}
For $\varphi(x;q)=\frac{1}{q-1} (x^{q-1}-1)$, the \Ts{} entropy can be written as
\begin{equation*}
\begin{aligned}
  \cT_q(\pi) &\coloneqq -\bE_{x\sim\pi}\varphi(x;q)=\int \pi(x) \frac{1-\pi(x)^{q-1}}{q-1}\ \rd x\\
  &=\frac{1-\int \pi(x)^q \,\rd x}{q-1}=\frac{1}{q-1}\bigl(1-I(\pi,\pi;q,0)\bigr)\\
  &=\frac{1-\exp\bigl(F(q\theta)-q F(\theta)\bigr)}{q-1}.
\end{aligned}
\end{equation*}
If $\pi$ is a multivariate \Gs{} distribution, we have
\begin{align*}
  F(q\theta) &=\frac{q}{2}\mu^\sT\Sigma^{-1}\mu+\frac{1}{2}\ln(2\pi)^d\lvert\Sigma\rvert-\frac{1}{2}\lnq^d.\\
\end{align*}

Since a covariance matrix is a symmetric positive semi-definite matrix, the LDL decomposition (a variant of \Ch{} decomposition) can be applied, which separates the covariance matrix into $\Sigma=L\diag\{\sigma^2_1,\dots,\sigma^2_d\}L^\sT$ where $L$ denotes a unit lower triangular matrix and $\diag\{\sigma^2_1,\dots,\sigma^2_d\}$ denotes a diagonal matrix with positive entries. Then we have
\begin{equation*}
\begin{aligned}
  F(q\theta) -q F(\theta) &=(1-q)\biggr\{\frac{d}{2}\ln2\pi+\frac{1}{2}\ln\lvert\Sigma\rvert-\frac{d\lnq}{2(1-q)}\biggr\}\\
  &=(1-q)\biggr\{\frac{d}{2}\ln2\pi+\frac{1}{2}\ln\prod_{i=1}^d\sigma^2_i -\frac{d\lnq}{2(1-q)}\biggr\}\\
  &=(1-q)\sum_{i=1}^d\biggl\{\frac{\ln2\pi}{2}+\ln\sigma_i-\frac{\lnq}{2(1-q)} \biggr\}.
\end{aligned}
\end{equation*}

\subsection{Tractable Form of \texorpdfstring{$\psi_\pi$}{psi\_pi}}
For separable $\Omega$, $\psi_\pi$ is written as \citep{rairl}
\begin{equation*}
  \psi_\pi(s,a)=-f^\prime(s,a)+\bE_{a\sim\pi}[f^\prime(\pi(a|s))-\varphi(a|s)]
\end{equation*}
where $\varphi(x)=\frac{k}{q-1}(1-x^{q-1})$ and accordingly $f(x)=x\varphi(x)$. For the gradient of $f(\cdot)$, we have
\begin{equation*}
\begin{aligned}
  f^\prime(x) &=\frac{k}{q-1} (1-q x^{q-1})\\
  &=\frac{k}{q-1}(q-q x^{q-1}-(q-1))\\
  &=\frac{qk}{q-1}(1-x^{q-1})-k\\
  &=q\varphi(x)-k.
\end{aligned}
\end{equation*}
Taking the expectation yields \Ts{} entropy as follows.
\begin{equation*}
  \bE_{x\sim\pi}\bigl[-f^\prime(x;\pi)+\varphi(x)\bigr]=\bE_{x\sim\pi}\bigl[k-q\varphi(x)+\varphi(x)\bigr]=(1-q)\cT^k_q(\pi)+k .
\end{equation*}
For a multivariate \Gs{} distribution $\pi$, the tractable form of $\bE_{x\sim\pi}\bigl[-f^\prime(x)+\varphi(x)\bigr]$ can be derived by using that of \Ts{} entropy $\cT^k_q(\pi)$ of $\pi$.
Thus $\psi_\pi$ can be rewritten as
\begin{equation*}
  \psi_\pi(s,a)=q\varphi(s)+(q-1)\cT^k_q(\pi)
\end{equation*}
In the special case of $q=1$ and $k=1$, we have $\psi_\pi(s,a)=\log\pi(a|s)$.

\subsection{\Bg{} Divergence with \Ts{} Entropy Regularization}
We consider the following form of the \Bg{} divergence:
\begin{equation*}
  \int\pi(x) \Bigl\{f^\prime\bigl(\hat{\pi}(x)\bigr)-\omega\bigl(\pi(x)\bigr)\Bigr\}\ \rd x-\int\hat{\pi}(x) \Bigl\{f^\prime\bigl(\hat{\pi}(x)\bigr)-\omega\bigl(\hat{\pi}(x)\bigr)\Bigr\}\ \rd x
\end{equation*}
For $\omega(x)=\frac{k}{q-1} (1-x^{q-1})$, $f^\prime(x)=\frac{k}{q-1} (1-q x^{q-1})=q\omega(x)-k$, and $k=1$, the above form is equal to
\begin{equation*}
\begin{aligned}
  &\int\pi(x) \biggl[\frac{1-q\hat{\pi}(x)^{q-1}}{q -1}\biggr] \ \rd x-\cT_q(\pi)-(q-1)\cT_q (\hat{\pi})+1\\
  &\quad=\frac{1}{q-1}-\frac{q}{q-1}\int\pi(x) \hat{\pi}(x)^{q-1}\ \rd x-\cT_q(\pi)-(q-1)\cT_q(\hat{\pi})+1\\
  &\quad=\frac{q}{q-1}-\frac{q}{q-1}\int\pi(x)\hat{\pi}(x)^{q-1}\ \rd x-\cT_q(\pi)-(q-1)\cT_q(\hat{\pi}).
\end{aligned}
\end{equation*}

Let us define two multivariate \Gs{} distributions as follows:
\begin{gather*}
  \pi(x)=\cN(x;\mu,\Sigma),\mu=[\mu_1,\cdots,\mu_d]^\sT,\Sigma=L\diag(\sigma^2_1,\cdots,\sigma^2_d)L^\sT,\\
  \hat{\pi}(x)=\cN(x;\hat{\mu},\hat{\Sigma}),\hat{\mu}=[\hat{\mu}_1,\cdots,\hat{\mu}_d]^\sT,\hat{\Sigma}=\hat{L}\diag(\hat{\sigma}^2_1,\cdots,\sigma^2_d)\hat{L}^\sT,
\end{gather*}
where $L$ and $\hat{L}$ denote unit lower triangular matrices. We have
\begin{equation*}
  \int \pi(x)\hat{\pi}(x)^{q-1}\ \rd x=I(\pi,\hat{\pi};1,q-1)=\exp\big\{F(\theta^\prime)-F(\theta)-(q-1)F(\hat{\theta})\big\},
\end{equation*}
where
\begin{align*}
  \theta &=\begin{bmatrix}\Sigma^{-1}\mu\\-\frac{1}{2}\Sigma^{-1}\end{bmatrix}\\
  \hat{\theta} &=\begin{bmatrix}\hat{\Sigma}^{-1}\mu\\-\frac{1}{2}\hat{\Sigma}^{-1}\end{bmatrix}\\
  \theta^\prime &=\theta+(q-1)\hat{\theta}=\begin{bmatrix}\Sigma^{-1}\mu+(q-1)\hat{\Sigma}^{-1}\mu\\-\frac{1}{2}(\Sigma^{-1}+(q-1)\hat{\Sigma}^{-1})\end{bmatrix}=\begin{bmatrix}\theta^\prime_1\\\theta^\prime_2\end{bmatrix}
\end{align*}
and
\begin{align*}
  F(\theta) &=\frac{1}{2}\mu^\sT\Sigma^{-1}\mu+\frac{1}{2}\ln(2\pi)^d\lvert\Sigma\rvert=\frac{1}{2}(\mu)^\sT \Sigma^{-1}\mu +\sum_{i=1}^d  \frac{\ln 2\pi}{2}+\ln\sigma_i,\\
  F(\hat{\theta}) &=\frac{1}{2}\hat{\mu}^\sT\hat{\Sigma}^{-1}\hat{\mu}+\frac{1}{2}\ln(2\pi)^d\lvert\hat{\Sigma}\rvert=\frac{1}{2}(\hat{\mu})^\sT \hat{\Sigma}^{-1}\hat{\mu} +\sum_{i=1}^d  \frac{\ln 2\pi}{2}+\ln\hat{\sigma}_i,\\
  F\bigl(\theta+(q-1)\hat{\theta}\bigr) &=-\frac{1}{4}(\theta^\prime_1)^\sT(\theta^\prime_2)^{-1}(\theta^\prime_1)+\frac{1}{2}\ln\lvert-\pi(\theta^\prime_2)^{-1}\rvert\\
\end{align*}

\subsection{Parameterization of the full covariance matrix using the LDL decomposition}\label{appsubsect:paraldl}
Computing the \Bg{} divergence for multi-variate \Gs{} distributions is challenging since the derivations involve inverses, determinants, and multiplications regarding $\Sigma$. Previous approaches did not address this issue and typically enforced $\Sigma$ to be a diagonal matrix with positive entries. Motivated by \citet{pourahmadi1999joint}, we propose to mitigate the computations regarding a covariance matrix using the LDL decomposition $\Sigma=L\diag\{\sigma^2_1,\dots,\sigma^2_d\}L^\sT$ at the parameterization level. It enables us to implement relatively simple and numerically safe computations such as
\begin{align}
 \Sigma^{-1}&=L^{\!-1}\diag(1/\sigma)(L^{\!-1})^\sT,\\
 \ln\lvert\Sigma\rvert&=2\sum_{i=1}^d\ln\sigma_i.
\end{align}
where $L$ is a unit lower triangular matrix. Finding an inverse matrix of a unit triangular matrix can be computed by $\cO(d^2)$ where the output is always a unit triangular matrix. Using the positive definiteness of $\Sigma$, the parameterization based on LDL decomposition allows a number of efficient computations for dealing with covariance matrices in practice while preserving the symmetry and the positive semi-definite matrix of $\Sigma$ on the parameterization level. We utilized these findings on implementing the full covariance \Gs{} policies and regularized reward functions.

\section{Implementation Details}\label{appsect:implement}
\subsection{Normalizing IRL rewards}
Unnormalized rewards of the IRL algorithm often mislead the agent to take  unnecessary awareness of \textit{termination} in finite-horizon MDPs \citep{dac}. For this point, IRL algorithms need to remove the difference between regarding steps depending on the MDP's time. Doob's optimal stopping theorem formally states that the expected value of a martingale at a stopping time is equal to its initial expectation. Assume a martingale makes the entire procedure a fair game on average, which means nothing can be gained by stopping the play.
\begin{theorem}[Doob's optimal stopping theorem]\label{thm:opt_stop}
Let a process $\{X_t\}^\infty_{t=1}$ be a martingale and $\tau$ be a stopping time with respect to filtration $\{\cF_t\}_{t\ge1}$. Assume that one of the conditions holds:
\vspace*{-5pt}
\begin{enumerate}[leftmargin=*,label=(\alph*),noitemsep,partopsep=0pt,topsep=0pt,parsep=0pt]
  \item $\tau$ is almost surely bounded, i.e., there exists a constant $c\in \mathbb{N}$ such that $\tau\le c$.
  \item $\tau$ has finite expectation and the conditional expectations of the absolute value of the martingale increments almost surely bounded, more precisely, $\bE[\tau] < \infty$ and there exists a constant $c$ such that $\bE\bigl[\lvert X_{t+1}-X_t\rvert\,\big\vert \cF_t\bigr] \le c$ almost surely on the event $\{\tau>t\}$ for all $t\ge0$.
  \item There exists a constant $c$ such that $\lvert X_{\min\{t,\tau\}}\rvert\le c$ almost surely for all $t\ge 0$.
Then $X_\tau$ is an almost surely well-defined random variable and  $\bE[X_\tau]=\bE[X_0]$.
\end{enumerate}
\vspace*{-5pt}
Then $X_\infty$ is integrable and $\bE[X_\infty]=\bE[X_0]$
\end{theorem}
Doob's optimal stopping theorem states one of the necessary conditions of IRL reward of normalizing the reward measures and making them a martingale even for finite-horizon benchmarks. In addition to the analyses of \citep{airl,rairl} regarding reward shaping and normalization, mean-zero rewards for training agents have the additional property of preventing the termination awareness, as stated by the optimal stopping theorem. Therefore, we suggest normalizing with the moving mean of intermediate values of regularized rewards and updating the RL algorithm with mean-subtracted rewards.

\subsection{Transformation between \texorpdfstring{$\pi_\phi$}{pi\_phi} and \texorpdfstring{$\psi_\phi$}{psi\_phi}}
In general, the regularized reward operation $\PsiOmega(\Pi)$ (as well as the \Bg{} divergence) is intractable to be computed. However, some tractable computation methods have been discovered for specific $\Omega$ if the policy is a specific parametric model (e.g., exponential families), thanks to the aforementioned studies. In \sect{sect:mdirl}, the underlying concept in \eq{eq:mdairlobj} is that the bidirectional transformation between $\pi_\phi$ and $\psi_\phi$ implicitly occurs via its shared network parameters $\phi$ without extra computation costs. For example, in our implementation, both $\pi_\phi$ and $\psi_\phi$ are analytically drawn in closed-form expressions using the shared parameter $\phi$ by the following methods, respectively.

{\textbf{Discrete policies.}\quad}Let the agent policy for a state $s\in\cS$ be defined as $\pi_\phi(a\vert s) = p(a)$ for a discrete probability distribution on the action space, typically parameterized by a softmax distribution. Since the cardinality of action space is finite in this case, we can directly compute each output according to \defn{defn:rop}, i.e., $\psi_\phi(s,a) = \Omega(p) + \nabla_p \Omega(p) - \sum_{a\in\cA}$.

{\textbf{Continuous policies with the \Sh{} regularizer.}\quad}For both discrete and continuous policies, the following equation holds: $\psi_\phi(s,a) = \log \pi_\phi(a\vert s)$ (pp. 4, \citet{rairl}). For multivariate \Gs{}s, we can analytically compute the log-likelihood. As stated in Appendix~\ref{appsubsect:paraldl}, we applied the LDL decomposition on the covariance matrix $\Sigma$, which is a variant of \Ch{} decomposition that ensures invertibility and positive-definiteness of the covariance matrix. In our experiments, this particular parameterization usually had significantly low numerical errors, thanks to the TensorFlow linear algebra libraries specialized for variants of the LU decomposition.  See the work done by \citet{pourahmadi1999joint} for more details for the parameterization.

{\textbf{Continuous policy with the \Ts{} regularizer.}\quad} Computing the operator $\PsiOmega$ for an arbitrary continuous policy is usually intractable when $\Omega$ is a \Ts{} entropic regularizer except when the policy is constrained to be specific parametric models. In this work, we assumed the \Gs{} policy, and the analytic form of $\PsiOmega(\Pi)$ was initially discovered by \citet{renyi_tsallis}. The entire portion of \appsect{appsect:full} is dedicated to derivations of $\psi_\phi$ when $\Omega$ is a Tsallis entropic regularizer.

\subsection{Network architectures}
For all networks, we used 2-layer MLP with 100 hidden units. We considered the reward model with two separate neural networks $(\psi_\phi, d_\xi)$ for the proposed reward function for $\lambda \in \bR^+$:
\begin{equation*}
  r_\phi(s,a)=\psi^\lambda_\phi(s,a)=\lambda \psi_\phi(s,a)+d_\xi(s),
\end{equation*}
Motivated by RAIRL-DBM, we considered the reward models in Fig.~\ref{fig:arch}. The model outputs reward for proximal updates trained by mirror descent and state-only discriminator network. Discriminating state visitation by $d_\xi(\cdot)$ is required because the reward function needs to consider every state (especially the state that cannot be visited by $\piE$) until $D_\KL(\rho_\pi\Vert \rho_{\piE}) \approx 0$. Fig.~\ref{fig:arch}~(a) shows logits of the softmax distribution involved when calculating rewards when the action space is discrete. For continuous control (Fig.~\ref{fig:arch}~(b)), the architecture is similar, where the mean and covariance are used to compute a reward for a particular action.
\begin{figure}
  \centering
  \begin{adjustbox}{width=0.75\textwidth, center}
  \begin{tikzpicture}[tight background, auto, thick, node distance=2cm, >=triangle 45]
  \clip (0-1,-4.976) rectangle+(12.2,5.05);
  \coordinate (sbranch) at (2,-4.5);
  \draw
  node at (2,-4.89)[text width=0.15cm, name=s] {$s$}
  node at (2,-1) [sum, name=suma1,inner sep=0pt] {\suma}
  node [block, below=0.4cm of suma1] (f) {\dervone}
  node at (0.5,-2.8) [block, name=softmax] {softmax}
  node [block, right=1cm of softmax] (lin2) {linear}
  node [block, below=0.4cm of softmax] (fc1) {FC}
  node [block, below=0.4cm of lin2] (fc2) {FC}
  node at (2,-.15)[text width=1.1cm, name=r] {\rew};
  \draw(s)--(sbranch);
  \draw[-{Latex[length=2mm]}](sbranch) -| node {}(fc1);
  \draw[-{Latex[length=2mm]}](sbranch) -| node {}(fc2);
  \draw[-{Latex[length=2mm]}](fc1) -- node {} (softmax);
  \draw[-{Latex[length=2mm]}](fc2) -- node {}(lin2);
  \draw[-{Latex[length=2mm]}](softmax) |- node
  {$\pi_{\phi} (\cdot |s)$}(f);
  \draw[-{Latex[length=2mm]}](lin2) |- node[el,right] {\bfn}(suma1);
  \draw[-{Latex[length=2mm]}](f) -- node {}(suma1);
  \draw[-{Latex[length=2mm]}](suma1) -- node {}(r);
  \coordinate (s2branch) at (8,-4.5);
  \draw
  node at (8,-4.89)[text width=0.15cm, name=s2] {$s$}
  node at (8,-1) [sum, name=suma12, inner sep=0pt] {\suma}
  node [block, below=0.4cm of suma12] (f2) {\dervtwo}
  node at (5.5, -2.8)[text width=0.2cm, name=a] {$a$}
  node at (6.6,-2.8) [circle, name=softmax2,inner sep=2pt,draw]{}
  node at (9.5, -2.8) [block, name=lin1] {linear}
  node [block, below=0.4cm of lin1] (fc22) {FC}
  node [block, left=1cm of fc22] (fc12) {FC}
  node at (8,-.15)[text width=1.1cm, name=r2] {\rewtwo};
  \draw(s2)--(s2branch);
  \draw[-{Latex[length=2mm]}](s2branch) -| node {}(fc12);
  \draw[-{Latex[length=2mm]}](s2branch) -| node {}(fc22);
  \draw[-{Latex[length=2mm]}](a) -- node {} (softmax2);
  \draw[-{Latex[length=2mm]}]([xshift=-0.12cm]fc12.north) -- node[el,left] {$\mu_{\phi}(s)$} (softmax2);
  \draw[-{Latex[length=2mm]}]
  ([xshift=0.5cm]fc12.north) |- node[el,right] {$\Sigma_{\hspace*{-.5pt}\phi}\hspace*{-.5pt}(s\hspace*{.5pt})$} (softmax2);
  \draw[-{Latex[length=2mm]}](fc22) -- node {}(lin1);
  \draw[-{Latex[length=2mm]}](softmax2) |- node
  {$\pi_{\phi}(a|s)$}(f2);
  \draw[-{Latex[length=2mm]}](lin1) |- node[el,right] {\bfn}(suma12);
  \draw[-{Latex[length=2mm]}](f2) -- node {}(suma12);
  \draw[-{Latex[length=2mm]}](suma12) -- node {}(r2);
  \end{tikzpicture}
  \end{adjustbox}
  \caption{Schematic illustrations of MD-AIRL reward architectures for discrete (left) and continuous control (right)}\label{fig:arch}
\end{figure}
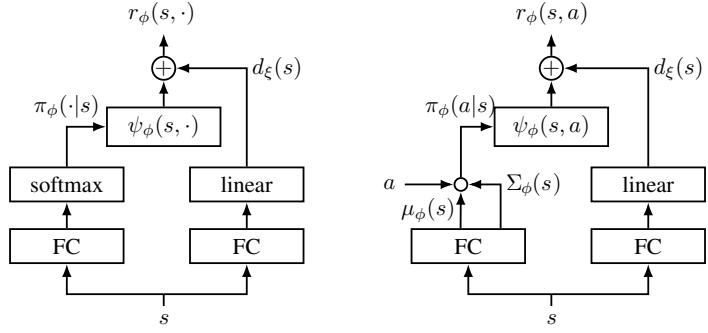

\subsection{Details on imitation learning data}
\begin{wrapfigure}{r}{0.398\textwidth}
  \vskip-25pt
  \centering
  \includegraphics[width=100pt]{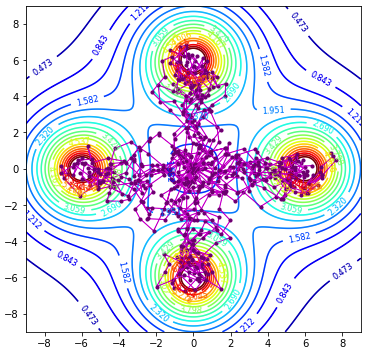}
  \vskip-2pt
  \caption{ Visualization of the multigoal environment and expert trajectories.}\label{fig:multigoal_expert}
  \vskip-14pt
\end{wrapfigure}
\textbf{The multigoal environment.\topicquad}Let the 2D coordinate denote the position of a point mass on the environment. In the multigoal environment, the agent, the point mass, is initially located according to the normal distribution $\cN(\mathbf{0},(0.1)^2 \rI)$. The four goals are located at $(6,0)$, $(-6,0)$, $(0,6)$, and $(0,-6)$, where the agent can move a maximum of 1 unit per time step for each coordinate. The ground-truth reward is given by the difference between successive values of a Gaussian mixture depicted as the contour plot in \fig{fig:multigoal_expert}.

\textbf{Collecting expert demonstrations.\topicquad}For the multigoal environment and as well as MuJoCo benchmarks, we trained an expert policy using the SAC algorithm \cite{sac2} and the demonstration data of IRL were collected from executing the trained RL expert. Only for the multigoal environment, the trajectories are post-processed to precisely capture optimal behavior (reaching each goal evenly). That is, we set the ratio of trajectories reaching each goal to exactly 25\%.

\subsection{Modeling policy with full covariance \Gs{} distributions}
We used the full covariance \Gs{} distribution in this experiment (as well as the toy experiment in \fig{fig:md_toy}). Note that the covariance matrix is positive-definite and symmetric. To achieve numerically stable computation, we applied LDL decomposition in Appendix~\ref{appsubsect:paraldl} to model covariance matrix using unit lower- and upper-triangle matrices, and a diagonal matrix. As a result, the policy network outputs a vector $[\bm{\mu}(s);\bm{\sigma}(s); \bm{l}(s)]^\sT$ for $s\in\cS$ where the additional vector $\bm{l}(s)$ denotes $\frac{d(d-1)}{2}$ entries of unit lower triangular matrix. Denote $\bm{L}(s)$ as a unit lower triangular matrix from $\bm{l}(s)$. For example, the covariance matrix can be reconstructed by
\begin{equation*}
  \bm{\Sigma}(s)=\bm{L}(s) [\diag(\bm{\sigma}(s)) ]\bm{L}(s)^\sT.
\end{equation*}
In this case, the action samples can be efficiently calculated by
\begin{equation}
  a=\bm{\mu}(s) +\bm{L}(s)(\bm{\sigma}(s)\cdot z)\quad z\sim\cN(\bm{0},\rI)
\end{equation}
Computing inverses, determinants and multiplications with unit triangular matrices and triangular and diagonal matrices can be efficiently performed by numerical libraries, where we used the accelerated linear algebra library from TensorFlow \citep{tf}. Therefore, we can efficiently model the \Bg{} divergence and reward using neural networks as provided in \appsect{appsect:full}.
We clipped the standard deviation as $\sigma_i(s) \in [\ln0.01, \ln2]$ using $\mathtt{tanh}$ for the stability. In \MJ{} experiments, instead of directly using squashed policies proposed in SAC \citep{sac}, we assumed the application of $\mathtt{tanh}$ as a part of the environment (known as \textit{hyperbolized} environments of RAIRL \citep{rairl}). Specifically, after an action $a$ is sampled from the policies, we passed $\mathtt{tanh}(a/1.01)*1.01$ to the environment. Then, we additionally clipped the hyperbolized actions to $1$, if the given environment is not tolerant to the excessive values of action.

\subsection{Hyperparameters}
\begin{center}
\begin{tikzpicture}[tight background, every node/.style={inner sep=0,outer sep=0},on grid]
  \begin{scope}[xshift=-.85cm]
    \node[scale=0.95] at (0, 0) {Table 4: The bandit environments.};
    \node[scale=0.85] at (0, -2.05) {
      \begin{tabular}{ll}
        \toprule
      Parameter&Value\\
        \midrule
      Learning rate (policy)&$1\cdot 10^{-3}$\\
      Learning rate (reward)&$1\cdot 10^{-3}$\\
      $\eta_1$&2.0\\
      $\eta_T$&0.5\\
      $\lambda$&1\\
      Discount factor ($\gamma$)&0.0\\
      Batch size&16\\
      Steps per update&50\\
      Total steps&300,000\\
        \bottomrule
      \end{tabular}};
  \end{scope}
  \begin{scope}[xshift=4.8cm]
    \node[scale=0.95] at (0, 0) {Table 5: The multigoal environment.};
    \node[scale=0.85] at (0, -2.2) {
      \begin{tabular}{ll}
        \toprule
      Parameter&Value\\
        \midrule
      Learning rate (policy)&$5\cdot 10^{-4}$\\
      Learning rate (reward)&$5\cdot 10^{-4}$\\
      Replay size &10,000\\
      $\eta_1$&1.0\\
      $\etaT$&0.1\\
  $\lambda$&1\\
  Discount factor ($\gamma$)&0.5\\
  Batch size&512\\
  Steps per update&50\\
  Total steps&300,000\\
    \bottomrule
  \end{tabular}};
  \end{scope}
  \begin{scope}[xshift=2cm,yshift=-4.8cm]
    \node[scale=0.95] at (0, 0) {Table 6: The \MJ{} environments.};
  \node[scale=.85] at (0,-2.37) {\begin{tabular}{lcccc}
    \toprule
  Parameter&Hopper-v3&Walker2d-v3&HalfCheetah-v3&Ant-v3\\
    \midrule
  Learning rate (policy)&$5\cdot 10^{-4}$&$5\cdot 10^{-4}$&$5\cdot 10^{-4}$&$5\cdot 10^{-4}$\\
  Learning rate (reward)&$5\cdot 10^{-4}$&$5\cdot 10^{-4}$&$5\cdot 10^{-4}$&$5\cdot 10^{-4}$\\
  Replay size &500,000&500,000&500,000&500,000\\
  $\eta_1$&1.0&1.0&1.0&1.0\\
  $\eta_T$&0.1&0.1&0.1&0.05\\
  $\lambda$&0.01&0.01&0.01&0.001\\
  Discount factor ($\gamma$)&0.99&0.99&0.99&0.99\\
  Batch size&256&256&256&256\\
  Steps per update&1,000&1,000&1,000&1,000\\
  Initial exploration&10,000&10,000&10,000&10,000\\
  Total steps&1,000,000&2,000,000&2,000,000&2,000,000\\
  \bottomrule
  \end{tabular}};
  \end{scope}
\end{tikzpicture}
\end{center}

\medskip
\section{Supplementary Experimental Results}\label{appsect:additional}
Guessing the optimal choice of scheduling $\eta_t$ for a short period of time is often challenging. \tab{tab:scheduling} provides extended results of the experiments depicted in \fig{fig:md_toy}. The table contains the performance of imitation learning varies by series of $\{\eta\}_{t=1}^{100}$ controlled by two hyperparameters $\alpha_1$ and $\alphaT$. These results substantially helped our hyperparameter choices of step sizes in the learning of MD-AIRL reward functions in \sect{sect:expr}.
\begin{table}[h]
\setcounter{table}{6}
\centering
\caption{Bregman divergences $D_\Omega(\pi_t\Vert\piE\hspace*{-1pt})$ after the final steps ($T=100$) with different step size scheduling ($10$ trials with different seeds).}\label{tab:scheduling}
  \begin{adjustbox}{width=0.99\textwidth, center}
\begin{tabular}{ c| c c c c }
  \toprule
  $(\eta_1, \etaT)$ & \Sh{} ($q=1$) & \Ts{} ($q=1.1$) & \Ts{} ($q=1.5$) & \Ts{} ($q=2$) \\
  \midrule
  $(2,2)$ & - & $1.19502\pm0.64091$ & $0.33996\pm0.26144$ & $0.68528\pm0.46490$\\
  $(1,1)$ & $0.08601\pm0.07951$ & $0.15432\pm0.25206$ & $0.22232\pm0.31760$ & $0.11193\pm0.16801$ \\
  $(0.5, 0.5)$ & $0.06707\pm0.06042$ & $0.07629\pm0.06931$ &$\mathbf{0.03801\pm0.04611}$& $0.06056\pm0.05854$\\
  $(0.2,0.2)$& $\mathbf{0.01051\pm 0.00920}$ & $\mathbf{0.03221\pm0.03239}$ & $1.21205\pm0.00011$&$\mathbf{0.01805\pm0.01587}$\\
  \midrule
  $(10,1)$&$0.09861\pm0.09546$&-&$1.09783\pm0.53399$&$0.65887\pm0.59846$\\
  $(1,0.1)$&$\mathbf{0.00706\pm0.00863}$&$\mathbf{0.00933\pm0.01089}$&$\mathbf{0.01660\pm0.01152}$&$\mathbf{0.02141\pm0.00899}$\\
  $(1,0.01)$&$0.01500\pm0.01510$&$0.01109\pm0.01405$&$0.02075\pm0.02099$&$0.03348\pm0.01901$\\
\bottomrule
\end{tabular}
\end{adjustbox}
\end{table}\\
The results indicate that scheduling $\eta_t$ with a harmonic progression $\eta_1=1$ and $\etaT=0.1$ shows the overall best results in this experiment. From these results and our theoretical arguments, MD is recommended to gradually lower $\eta_t$, but the rate of change has to be carefully considered, especially when $T$ is not significant. Thus, there are suitable scheduling ways of the step size $\eta_t$ in practice, depending on $\Omega$ and $T$. As a rule of thumb, we recommend setting the initial step size close to $1$ and scheduling to $\eta_{\scriptscriptstyle T} \approx 0$ when there is a reasonable amount of time for training.

\end{document}